\documentclass{article}

\usepackage[margin=1in]{geometry}
\usepackage[utf8]{inputenc}
\usepackage{amsmath}
\usepackage{amsfonts}
\usepackage{amssymb}
\usepackage{amsthm}
\usepackage{mathtools}
\usepackage{bm}
\usepackage{xcolor}
\usepackage{enumitem}
\usepackage{tikz}
\usepackage{mathrsfs}
\usepackage{graphicx}
\usepackage{caption}
\usepackage{subcaption}
\usepackage{hyperref}

\theoremstyle{plain}
\newtheorem{theorem}{Theorem}
\newtheorem{cor}{Corollary}

\newtheorem{lemma}{Lemma}

\newtheorem{prop}{Proposition}
\theoremstyle{definition}

\newtheorem{remark}{Remark}

\newcommand{\R}{\mathbb{R}}
\newcommand{\N}{\mathbb{N}}

\newcommand{\Sym}{\mathbb{S}}
\newcommand{\PSD}{\mathbb{S}_{+}}
\newcommand{\PD}{\mathbb{S}_{++}}

\newcommand{\PL}{\mathsf{PL}}
\newcommand{\Lip}{\mathsf{Lip}}

\newcommand{\bone}{\bm{1}}
\def\app#1#2{%
  \mathrel{%
    \setbox0=\hbox{$#1\sim$}%
    \setbox2=\hbox{%
      \rlap{\hbox{$#1\propto$}}%
      \lower1.1\ht0\box0%
    }%
    \raise0.25\ht2\box2%
  }%
}

\newcommand{\E}{\mathbb{E}}
\newcommand{\Var}{\mathsf{Var}}
\newcommand{\Cov}{\mathsf{Cov}}

\renewcommand{\P}{\mathbb{P}}

\newcommand{\cT}{\mathcal{T}}

\newcommand{\MMSE}{\mathsf{MMSE}}

\DeclareMathOperator{\tr}{tr}
\DeclareMathOperator*{\vectorize}{vec}
\DeclareMathOperator*{\plim}{p-lim}

\newcommand{\indep}{\mathrel{\perp\!\!\!\perp}}

\DeclareMathOperator*{\argmax}{arg\;max}
\DeclareMathOperator{\KL}{\mathnormal{D}_{KL}}
\newcommand{\op}{\mathrm{op}}
\newcommand{\ind}{\mathbf{1}}

\newcommand{\eps}{\epsilon}

\newcommand{\iid}{\overset{\mathrm{iid}}{\sim}}
\newcommand{\cQ}{\mathcal{Q}}

\newcommand{\tX}{\Tilde{\bX}}
\newcommand{\tG}{\Tilde{\bG}}
\newcommand{\tH}{\Tilde{\bH}}
\newcommand{\tY}{\Tilde{\bY}}
\newcommand{\tB}{\Tilde{\mB}}
\newcommand{\tM}{\Tilde{\bM}}
\newcommand{\tK}{\Tilde{\mK}}

\newcommand{\tZ}{\Tilde{\bZ}}
\newcommand{\tSigma}{\Tilde{\mSigma}}
\newcommand{\tLambda}{\Tilde{\mLambda}}
\newcommand{\parrow}{\xrightarrow[]{\mathrm{pr}}}

\DeclareMathOperator{\gvec}{vec}
\DeclareMathOperator{\diag}{diag}
\DeclareMathOperator{\rank}{rank}

\newcommand{\Id}{\mI}

\newcommand{\mA}{{A}}
\newcommand{\mB}{{B}}
\newcommand{\mC}{{C}}
\newcommand{\mD}{{D}}

\newcommand{\mF}{{F}}

\newcommand{\mI}{{I}}

\newcommand{\mK}{{K}}
\newcommand{\mL}{{L}}
\newcommand{\mM}{{M}}

\newcommand{\mP}{{P}}
\newcommand{\mQ}{{Q}}
\newcommand{\mR}{{R}}
\newcommand{\mS}{{S}}
\newcommand{\mT}{{T}}
\newcommand{\mU}{{U}}
\newcommand{\mV}{{V}}

\newcommand{\mX}{{X}}
\newcommand{\mY}{{Y}}
\newcommand{\mZ}{{Z}}

\newcommand{\mXi}{{\Xi}}
\newcommand{\mSigma}{{\Sigma}}
\newcommand{\mLambda}{{\Lambda}}
\newcommand{\mGamma}{{\Gamma}}
\newcommand{\mDelta}{{\Delta}}
\newcommand{\mPsi}{{\Psi}}
\newcommand{\mPhi}{{\Phi}}
\newcommand{\mOmega}{{\Omega}}

\newcommand{\mbeta}{\bm{\beta}}
\newcommand{\mgamma}{\bm{\gamma}}

\newcommand{\mmu}{\bm{\mu}}
\newcommand{\meta}{\bm{\eta}}

\newcommand{\cA}{\mathcal{A}}
\newcommand{\md}{\bm{d}}
\newcommand{\me}{\bm{e}}

\newcommand{\mq}{\bm{q}}

\newcommand{\ms}{\bm{s}}

\newcommand{\mv}{\bm{v}}

\newcommand{\mx}{\bm{x}}
\newcommand{\my}{\bm{y}}
\newcommand{\mz}{\bm{z}}

\newcommand{\bA}{\bm{A}}

\newcommand{\bG}{\bm{G}}
\newcommand{\bH}{\bm{H}}

\newcommand{\bM}{\bm{M}}

\newcommand{\bS}{\bm{S}}

\newcommand{\bW}{\bm{W}}
\newcommand{\bX}{\bm{X}}
\newcommand{\bY}{\bm{Y}}
\newcommand{\bZ}{\bm{Z}}

\newcommand{\Normal}{\mathsf{N}}
\newcommand{\GOE}{\mathsf{GOE}}

\title{Approximate Message Passing for the Matrix Tensor Product Model}
\author{Riccardo Rossetti$\textsuperscript{1}$ and Galen Reeves$\textsuperscript{1,2}$ \vspace{0.3cm} \\ 
$\textsuperscript{1}$ Department of Statistical Science, Duke University \\
$\textsuperscript{2}$ Department of Electrical and Computer Engineering, Duke University}
\date{}

\begin{document}

\maketitle

\begin{abstract}
    We propose and analyze an approximate message passing (AMP) algorithm for the matrix tensor product model, which is a generalization of the standard spiked matrix models that allows for multiple types of pairwise observations over a collection of latent variables. A key innovation for this algorithm is a method for optimally weighing and combining multiple estimates in each iteration. Building upon an AMP convergence theorem for non-separable functions, we prove a state evolution for non-separable functions that provides an asymptotically exact description of its performance in the high-dimensional limit. We leverage this state evolution result to provide necessary and sufficient conditions for recovery of the signal of interest. Such conditions depend on the singular values of a linear operator derived from an appropriate generalization of a signal-to-noise ratio for our model. Our results recover as special cases a number of recently proposed methods for contextual models (e.g., covariate assisted clustering) as well as inhomogeneous noise models.
\end{abstract}

\section{Introduction}
The problem of recovering low-dimensional structures from very large amounts of data is a central problem in modern data science applications. Among the many inference and signal processing problem that can be instantiated within this framework, the extracting a latent low-rank component from a large random matrix has received considerable attention both on the theoretical and algorithmic points of view.
From the foundational studies of the spectrum rank-one spiked Wigner (for symmetric matrices) and spiked Wishart (for asymmetric ones) models \cite{johnstone2001distribution, baik2005, Peche:2006aa}, a wide range of variations of model assumptions and alternative approaches for analysis have been developed. Some examples include studies of sparse PCA \cite{Deshpande:2014aa,Ben-Arous:2020aa}, stochastic blockmodels and its variations \cite{Abbe:2017aa,Deshpande:2018aa}. 

Among algorithmic approaches that have found considerable success in the study of low-rank spiked matrix models, \textit{approximate message passing} (AMP) \cite{Bayati:2011aa} and its variations have received considerable attention and have proved invaluable in the study of matrix factorization problems \cite{Kabashima:2016aa, Deshpande:2016aa, montanari2021estimation, Lesieur:2017aa, Fletcher:2018aa} A comprehensive review of AMP algorithms, its theory and applications is given in \cite{feng2022unifying}.

In this work, we consider the \textit{matrix tensor product} (MTP) model \cite{reeves2020information}, a generalization of the usual spiked matrix model for collections of latent signals $X_i \in \R^d$, $i \in [n]$ whose pairwise interactions are observed in additive independent Gaussian noise. We develop an AMP algorithm adapted to the MTP by leveraging a representation as a collection of low-rank spiked matrix models with arbitrary coupling structures. Furthermore, we describe the properties of the mean squared error achievable via AMP in the Bayes-optimal setting, and provide necessary and sufficient conditions for AMP to achieve weak recovery of the signal of interest in the form of a stability analysis of the fixed points of a recurrence relation (called \textit{state evolution}) that tracks the exact asymptotic estimation performance of the output of the AMP algorithm.

\subsection{The matrix tensor product model}
We begin by describing the MTP that is the object of this work. In its original formulation \cite{reeves2020information}, it is assumed there are two latent signal matrices $\bX_1 \in \R^{n_1 \times d_1}, \bX_2 \in \R^{n_2 \times d_2}$ that are observed in additive Gaussian noise, yielding the data matrix $\bY \in \R^{n_1 n_2 \times K}$ defined as
\begin{align} \label{eq:MTP_kron_def}
    \bY = \frac{1}{\sqrt{n_1}}(\bX_1 \otimes \bX_2) \mGamma + \bW,
\end{align}
where $\mGamma := (\mgamma_1, \dotsc, \mgamma_K) \in \R^{d_1 d_2 \times K}$ denotes a known coupling matrix and $\bW \in \R^{n_1n_2 \times K}$ has standard i.i.d. Gaussian entries.

As an equivalent definition that is more convenient for the purpose of deriving our AMP algorithm, we rewrite the MTP of \eqref{eq:MTP_kron_def} as a symmetric multi-view low-rank spiked matrix model observed in additive Gaussian noise. For each $k \in [K]$, we define $\mGamma_k \in \R^{d_1 \times d_2}$ to be such that $\vectorize(\mGamma_k) = \mgamma_k$, with $\vectorize(\cdot)$ denoting the vectorization operation. Then, we define the collection of matrix observations $\{\bY_k \in \R^{n_1 \times n_2}\}_{k \in [K]}$ as
\begin{align} \label{eq:MTP_mv_def}
    \bY_k = \frac{1}{\sqrt{n_1}}\bX_1 \mGamma_k \bX_2^\top + \bW_k,
\end{align}
with the entries of $\bW_k \in \R^{n_1 \times n_2}$ being i.i.d. standard Gaussian random variables, independently across $k \in [K]$. We refer to $n_1$ as the \textit{problem size}, which we assume to be growing together with $n_2$ with a limiting aspect ratio $n_2/ n_1 \rightarrow \alpha \in \R_+$. 

While the formulation with two components $\bX_1$ and $\bX_2$ is natural in the presence of asymmetric noise and is very intuitive in the case the sequences $\bX_1$ and $\bX_2$ are assumed to be independent, it leads notational difficulties in the statement of our AMP results, except for some special cases. For ease of exposition, then, we will be stating our results in terms of MTP models for which $\bX_1 = \bX_2 \in \R^{n \times d}$, and we call such signal $\bX$. This comes without loss of generality, as for $\bX_1 \in \R^{n_1 \times d_1}, \bX_2^{n_2 \times d_2} $ we can always let $n = n_1 + n_2$, $d = d_1 + d_2$ and consider the signal $\bX \coloneqq \bX_1 \oplus \bX_2$, the matrix direct product of $\bX_1, \bX_2$, and define the $d \times d$ coupling matrices as
\begin{align}
    \Tilde{\mGamma} \coloneqq \begin{bmatrix}
        0 & \sqrt{1 + \alpha} \mGamma \\ 0 & 0
    \end{bmatrix}, \quad k = 1, \dotsc, K.
\end{align}
Then, the observation models $\{\Tilde{\bY}_k \in \R^{n \times n}\}_{k \in [K]}$,
\begin{align} \label{eq:MTP_mv_same_signals_def}
    \tilde{\bY}_k = \frac{1}{\sqrt{n}} \bX \Tilde{\mGamma} \bX^\top + \Tilde{\bW}_k
\end{align}
are easily seen to be asymptotically equivalent to $\{\bY_k\}_{k \in [K]}$. Going forward, we will be referring to models of the form \eqref{eq:MTP_mv_same_signals_def} as the (asymmetric) MTP.

A variation of the model \eqref{eq:MTP_mv_same_signals_def} that is of independent interest is the \textit{symmetric case}, in which it is assumed that $\bX \in \R^{n \times d}$ and $\mGamma_k \in \Sym^d$ is a symmetric matrix for all $k \in [K]$. When the coupling matrices are symmetric, we denote them by $\mLambda_k$ to stress the presence of such structure. Under these assumptions, it is natural to have the observation model reflect the intrinsic symmetries of the signal observation in the noise as well. This leads to the symmetric observation model
\begin{align} \label{eq:MTP_mv_symmetric_def}
    \bY_k = \frac{1}{\sqrt{n}} \bX \mLambda_k \bX + \bG_k,
\end{align}
where $\bG_k \in \R^{n \times n}$ are i.i.d. draws from a Gaussian Orthogonal Ensemble (GOE):
\begin{align}
    \bG \sim \GOE \iff \bG = \frac{1}{\sqrt{2}}(\bW + \bW^\top); \quad (\bW_{ij}) \iid \Normal(0,1), \ i,j \in [n].  
\end{align}

\subsection{Definitions and notation.} 
We establish some notation that will be used throughout the paper. We use lowercase bold letters (e.g. $\mx,\my,\mz$) to denote (real-valued) vectors and uppercase letters (e.g. $X,Y,Z$) for random vectors. For matrices, we use capital letters (e.g. $\mX,\mY,\mZ$) and use the convention that random matrices are bold (e.g. $\bX, \bY, \bZ$).  Furthermore, for a matrix $\mX \in \R^{m \times d}$ we use the notation $\mX_i \in \R^d$ to denote the vector obtained from the $i$-th row of $\mX$, for $i \in [m]$, and in case the object of this operation is ambiguous we enclose the matrix whose row we are choosing with square brackets (e.g.,  $[\mX_a]_i$ or $[\mX\mY]_i$). For fixed $d, m \in \N$, we endow vectors in $\R^d$ with the standard Euclidean norm $\|\cdot\|$ and matrices in $\R^{d \times m}$ with the Frobenius norm $\|\cdot\|_F$. As we will be mainly concerned with an asymptotic analysis of approximate message passing, we will reserve the use of $n$ to represent the dimension along which the problem size is growing. Consequently, when writing, for example, $\mX \in \R^{n \times d}$, $\mX$ is to be understood as a sequence of vectors in $\R^d$, which we represent as $n \times d$ matrices. Along such sequences, we will define a $n$-dependent norm to account for the increasing problem size. For $n \in \N$, we define the norm $\| \cdot \|_{n}$ on the space $\R^{n \times d}$ as
\begin{align}
    \| \mX \|_{n} \coloneqq \frac{\| \mX \|_F}{\sqrt{n}} = \left(\frac{1}{n} \sum_{i=1}^n \| \mX_i \|^2\right)^{1/2}, \quad \mX \in \R^{ n \times d}.
\end{align}
Going forward, we will always take $d \in \N$ to be fixed and known, and we will metrize the (sequence of) spaces $\R^{n \times d}$ with the distance function induced by $\|\cdot\|_{n}$, which can be easily verified to be a proper norm for any given $n \in \N$.

Next, we introduce the notion of \textit{pseudo-Lipschitz} functions. For a pair of normed spaces $(S,\| \cdot \|), (S', \|\cdot\|')$ we say a function $\phi : S \rightarrow S'$ is pseudo-Lipschitz of order $p$ if there exists some non-negative finite constant $L$ such that, for all $x,y \in S$,
\begin{align}
    \| \phi(x) - \phi(y) \|' \leq L \| x- y\| \left( 1 + \|x\|^{p-1} + \|y\|^{p-1} \right).
\end{align}
We say the smallest $0 \leq L < \infty$ such that the above holds is the \textit{Lipschitz constant} for $\phi$, and write $\phi \in \PL_p(L)$. Clearly, taking $p = 1$ corresponds up to a scaling constant to the usual definition of a Lipschitz function, and we write $\phi \in \Lip(L) \coloneqq \PL_1(L/3)$. Similarly, we say a sequence of functions $\{ \phi_n : S_n \rightarrow S_n' \}_{n \in \N}$ is \textit{uniformly pseudo-Lipschitz} of order $p$ whenever each $\phi_n$ is pseudo-Lipschitz of order $p$ with some constant $L_n$ and furthermore $L = \sup_n L_n < \infty$. We denote such sequences by writing $\{ \phi_n \}_{n \in \N} \subset \PL_p(L)$, and $L$ is the Lipschitz constant of the sequence. 

Finally, we denote a sequences of random variables $\{ X_n \}_{n \in \N}$ that converges in probability to some limiting random variable $X$ as
\begin{align}
    X_n \parrow X; \quad \text{or} \quad \plim_{n\rightarrow\infty} X_n = X.
\end{align}
Furthermore, for some real-valued sequence $\{a_n\}_{n \in \N}$ and a sequence of random variables $\{ X_n \}_{n \in \N}$, we adopt the following \textit{in-probability} little- and big-oh notation:
\begin{alignat}{2}
    X_n = O_p(a_n) &\iff& \exists M > 0 &:  \lim_{n \rightarrow \infty} \Pr\left\{ \left|\frac{X_n}{a_n}\right| > M \right\} = 0; \\
    X_n = o_p(a_n) &\iff& \forall \epsilon > 0 &:  \lim_{n \rightarrow \infty} \Pr\left\{ \left|\frac{X_n}{a_n}\right| > \epsilon \right\} = 0.
\end{alignat}
We remark that the usual $O$-notation calculus rules still hold for these in-probability definitions. 

\section{Approximate message passing}
\subsection{AMP for the MTP} \label{sec:amp_mtp_main_result_asymmetric}
Before stating our main result, we describe all the required assumptions and we describe the construction of the \textit{state evolution} (SE) recursion that describes the asymptotic behavior of the AMP algorithm. For these purposes, it is convenient to operate with a rescaled version of the MTP model \eqref{eq:MTP_mv_def}, namely
\begin{align} \label{eq:MTP_mv_def_rescaled}
    \frac{1}{\sqrt{n}} \bY_k = \frac{1}{n} \bX \mGamma \bX + \frac{1}{\sqrt{n}}\bW_k; \quad k \in [K],
\end{align}
and for this section we will refer to these rescaled versions as the observations $\{\bY_k\}_{k \in [K]}$. We have the following assumptions.

\begin{enumerate}[label=$(\mathrm{A\arabic*})$]
    \item For $k \in [K]$, $\bW_k$ is a matrix with i.i.d. standard Gaussian entries, drawn independently of each other adn of $\bX$. Furthermore, for $p\geq 1$, $L> 0$ and any sequences of test functions $\{ \phi_{n}: \R^{n \times d} \rightarrow \R \}_{n \in \N} \subset \PL_p(L)$, it holds 
    \begin{align} \label{eq:mtp_asymmetric_signal_initialization_concentration_assumption}
        \plim_{n \rightarrow \infty} \left| \phi_{n}(\bX) - \E[ \phi_{n}(\bX)] \right| = 0. 
    \end{align}
    Finally we have that $\E \|\bX\|_{n}^2 \rightarrow C < \infty$ as $n \rightarrow \infty$. \label{as:mtp_model_amp_signal_assumptions}

    \item For each $t \in \N$, the \textit{denoiser sequence} $\{ f_t: \R^{n \times d} \rightarrow \R^{n \times d} \}_{n \in \N}$ is uniformly Lipschitz with some constant $L < \infty$. Furthermore, fix any positive semi-definite matrix $\mS \in \PSD^{2d}$ and define the $n \times 2d$ Gaussian matrix $(\bZ, \bZ') \sim \Normal(0, \mS \otimes \Id_n)$ independent of $\bX$. For any given $\mK, \mK' \in \R^{d \times d}$ denote $\bH \coloneqq \bX \mK + \bZ, \bH' \coloneqq \bX \mK' + \bZ'$. Then, for any $t,s \in \N$, $f_t$ and $f_s$ are such that
    \begin{align}
        &\frac{1}{n} \E\!\left[ \bX^\top f_t(\bH) \right]; \quad
        \frac{1}{n} \E\!\left[ f_t(\bH)^\top f_t(\bH)  \right]; \quad
        \frac{1}{n} \E\!\left[ f_t(\bH)^\top f_s(\bH')  \right] 
    \end{align}
    have well-defined and finite limits as $n \rightarrow \infty$. \label{as:mtp_model_amp_denoiser_assumption}
    \item The \textit{initialization matrix} $\bM^0 \in \R^{n \times d}$ is such that, letting $\bH = \bX \mK + \bZ$ as in \ref{as:mtp_model_amp_denoiser_assumption}, the following functions
    \begin{align}
        &\frac{1}{n} \E\!\left[ \bX^\top \bM^0 \right]; \quad
        \frac{1}{n} \E\!\left[ (\bM^0)^\top \bM^0  \right]; \quad
        \frac{1}{n} \E\!\left[ (\bM^0)^\top f_t(\bH) \right]
    \end{align}
    have well-defined and finite limits as $n \rightarrow \infty$ for all $t \in \N$. Furthermore, 
    \begin{align}
        \plim_{n\rightarrow\infty}\frac{1}{n} \bX^\top \bM^0 &= \lim_{n \rightarrow \infty} \frac{1}{n} \E\!\left[ \bX^\top \bM^0 \right]; \\
        \plim_{n \rightarrow\infty}\frac{1}{n} (\bM^0)^\top \bM^0 &= \lim_{n \rightarrow \infty} \frac{1}{n} \E\!\left[ (\bM^0)^\top \bM^0 \right]; \\
        \plim_{n  \rightarrow\infty}\frac{1}{n} (\bM^0)^\top f_t(\bH) &= \lim_{n \rightarrow \infty} \frac{1}{n} \E\!\left[ (\bM^0)^\top f_t(\bH) \right].        
    \end{align} \label{as:mtp_model_amp_initialization_assumption}
\end{enumerate}
Beside the above assumptions, we introduce a collection of \textit{reweighting matrices} $\{\mA_k^t \in \R^{d \times d}\}_{t \in \N, k \in [K]}$ with bounded entries. With this, we are ready to iteratively define the \textit{state evolution} quantities for the recursion. For $t = 1$, we set
\begin{align}
    \mK^1 &= \lim_{n \rightarrow \infty} \frac{1}{n} \sum_{k = 1}^K \left\{ \mGamma_k  \E[\bX^\top \bM^0] (\mA_k^1)^\top + \mGamma_k^\top \E[\bX^\top \bM^0] \mA_k^1 \right\}; \\
    \mSigma^1 &= \lim_{n \rightarrow \infty} \frac{1}{n} \sum_{k = 1}^K \left\{ \mA_k^1  \E[(\bM^0)^\top \bM^0] (\mA_k^1)^\top + (\mA_k^1)^\top \E[(\bM^0)^\top \bM^0] \mA_k^1 \right\}.
\end{align}
Then recursively for $s < t \in \N$, we define $\bH^t \coloneqq \bX \mK^t + \bZ^t$, for $\bZ^t \sim \Normal( 0, \mSigma^t \otimes \Id_n )$ independent of $\bX$, where 
\begin{align}
    \mK^{t+1} &= \lim_{n \rightarrow \infty} \frac{1}{n} \sum_{k = 1}^K \left\{ \mGamma_k  \E[\bX^\top f_t(\bH^t)] (\mA_k^t)^\top + \mGamma_k^\top \E[\bX^\top f_t(\bH^t)] \mA_k^t \right\}; \\
    \mSigma^{t+1} &= \lim_{n \rightarrow \infty} \frac{1}{n} \sum_{k = 1}^K \left\{ \mA_k^{t+1}  \E[f_t(\bH^t)^\top f_t(\bH^t)] (\mA_k^{t+1})^\top + (\mA_k^{t+1})^\top \E[f_t(\bH^t)^\top f_t(\bH^t)] \mA_k^{t+1} \right\}; \\
    \mSigma^{1, t+1} &= \lim_{n \rightarrow \infty} \frac{1}{n} \sum_{k = 1}^K \left\{ \mA_k^{1}  \E[(\bM^0)^\top f_t(\bH^t)] (\mA_k^{t+1})^\top + (\mA_k^{1})^\top \E[(\bM^0)^\top f_t(\bH^t)] \mA_k^{t+1} \right\}; \\
    \mSigma^{s+1, t+1} &= \lim_{n \rightarrow \infty} \frac{1}{n} \sum_{k = 1}^K \left\{ \mA_k^{s+1}  \E[f_s(\bH^t)^\top f_t(\bH^t)] (\mA_k^{t+1})^\top + (\mA_k^{s+1})^\top \E[f_t(\bH^t)^\top f_t(\bH^t)] \mA_k^{t+1} \right\}.
\end{align}
For all $t \in \N$, the Gaussian matrices $(\bZ^1, \dotsc, \bZ^t) \in \R^{n \times td}$ are jointly Gaussian with zero mean and covariance given by $\mSigma^{[t]} \otimes \Id_n, \ \mSigma^{[t]}\in \PSD^{td}$,
\begin{align}
    \mSigma^{[t]} = \begin{bmatrix}
        \mSigma^1 & \mSigma^{1,2} & \hdots & \mSigma^{1,t} \\
        (\mSigma^{1,2})^\top & \mSigma^2 & \hdots & \mSigma^{2,t} \\
        \vdots & \vdots & \ddots & \vdots \\
        (\mSigma^{1,t})^\top & (\mSigma^{2,t})^\top & \hdots & \mSigma^t
    \end{bmatrix}.
\end{align}
Finally, we introduce the \textit{(Onsager) correction terms} that will be necessary to guarantee that the AMP algorithm has the correct asymptotic distributional limits. For $t \in \N$, the correction term $\mB^t \in \R^{ d\times d}$ is given by
\begin{align} \label{eq:amp_mtp_asymmetric_correction_term}
    \mB^t = \sum_{k = 1}^K \left\{ \mA_k^{t+1} \mD^t (\mA_k^t)^\top + (\mA_k^{t+1})^\top \mD^t \mA_k^t \right\}, 
\end{align}
where $D^t \in \R^{d \times d}$ is the \textit{expected divergence matrix} defined entry-wise as
\begin{align}
    \mD_{jk}^t = \frac{1}{n} \sum_{i=1}^n \E\!\left[ \frac{\partial}{\partial H_{ik}^t} [f_t(\bH)]_{ij} \right], \quad j,k \in [d].
\end{align}

With this, we are ready to present the AMP recursion. For $t \in \N$, we have
\begin{align} \label{eq:amp_algorithm_mtp_non_symmetric}
    \bX^t &= \sum_{k=1}^K \left\{ \bY_k \bM^{t-1} (\mA_k^t)^\top + \bY_k^\top \bM^{t-1} \mA_k^t \right\} - \bM^{t-2} (\mB^{t-1})^\top; \quad \bM^t = f_t(\bX^t),
\end{align}
where we adopt the convention that for $t=1$ the term $\bM^{-1}(\mB^0)^\top$ is identically zero. Going forward, we refer to the above recursion initialized with $\bM^0$ using denoisers $f_t$ with as the asymmetric AMP recursion $\{ \bX^t \mid \bM^0, \mA_k^t, f_t \}_{t \in \N}$. 

\begin{remark}
    The algorithm described in \eqref{eq:amp_algorithm_mtp_non_symmetric} is actually a special case of a more general recursion in which the choice of denoiser involves a linear reweighting step. Such choice is related to the form of the sufficient statistic (in the Bayesian sense) of the AMP iterates for $\bX$ under their distributional limits $\bH^t$. A further advantage, as we will see, is that linear reweighting allows for a simple characterization of the AMP iterates as Gaussian noise-corrupted versions of the original signal. We will expand upon such considerations in Section~\ref{sec:bayes_optimal_reweighting}.
\end{remark}
We have the following convergence result for the AMP iterates.

\begin{theorem} \label{th:amp_mtp_state_evolution_theorem_asymmetric}
    Consider the asymmetric AMP iterations $\{ \bX^t \mid \bM^0, \mA_k^t, f_t \}_{t \in \N}$ such that \ref{as:mtp_model_amp_signal_assumptions}-\ref{as:mtp_model_amp_initialization_assumption} hold. Then, for any $t \in \N$, $p \geq 1$, $L > 0$ and sequences of functions $\{ \phi_{n} : \R^{n \times (t+1)d} \rightarrow \R \}_{n \in \N} \subset \PL_p(L)$,  it holds
    \begin{align}
        &\plim_{n \rightarrow \infty} \left| \phi_{n}(\bX, \bX^1, \dotsc, \bX^t) - \E\!\left[ \phi_{n}(\bX, \bH^1, \dotsc, \bH^t) \right] \right| = 0, 
    \end{align}
\end{theorem}
\begin{proof}
    For a proof, see Appendix~\ref{sec:amp_mtp_asymmetric_theorem_proof}.
\end{proof}

The statement of the above theorem naturally includes the often-encountered formulation of AMP convergence theorems in terms of convergence of the row-wise empirical measures of the iterates $\{ \bX^t \}_{t \in \N}$ to some limiting random vectors $\{ H^t \in \R^d \}_{t \in \N}$ when integrated with respect to some class of test functions, namely pseudo-Lipschitz functions of order 2 (with respect to the Euclidean metric) from $\R^{d}$ to $ \R$. In fact, assumption \ref{as:mtp_model_amp_signal_assumptions} ensures that the limiting row-averaged second moment of $\bX$ is finite, and furthermore from Lemma~\ref{lem:averages_pseudo_lipschitz_function_uniformly_pseudo_Lipschitz} we have that for any given function $\phi : \R^{d} \rightarrow \R \in \PL_2(L)$ the function
\begin{align}
    \bX \mapsto \frac{1}{n} \sum_{i=1}^{n} \phi(X_i), \quad \bX \in \R^{n \times d} 
\end{align}
is uniformly pseudo-Lipschitz of order 2. Assume the existence of a random variables $X_\star \in \R^{d}$ such that, for any functions $\varphi: \R^{d} \rightarrow \R \in \PL_2(L)$, $L < \infty$, we have that
\begin{align} \label{eq:mtp_signal_asymmetric_wasserstein_limits_assumption}
    \plim_{n \rightarrow \infty} \left| \frac{1}{n} \sum_{i=1}^{n} \varphi_1(X_i) - \E[\varphi(X_\star)] \right| = 0.
\end{align}
From results in \cite{feng2022unifying}, the above is equivalent to the random empirical measures induced by the rows of $\bX$ having some well-defined limit in quadratic Wasserstein distance with high probability. The following statement then holds.

\begin{cor} \label{cor:amp_mtp_asymmetric_wasserstein_convergence_result}
    Assume \ref{as:mtp_model_amp_signal_assumptions}-\ref{as:mtp_model_amp_initialization_assumption} hold and furthermore $\bX$ satisfies \eqref{eq:mtp_signal_asymmetric_wasserstein_limits_assumption} with some limit distribution $P_\star$, and let $X_\star \sim P_\star$. For $t \in \N$, define $H^t \coloneqq (\mK^t)^\top X_\star + Z^t$, for $(Z^1, \dotsc, Z^t) \sim \Normal(0, \mSigma^{[t]})$ independent of $X_\star$. Then, the asymmetric AMP iterations $\{ \bX \mid \bM^0, \mA_k^t, f_t\}_{t \in \N}$ are such that, for any test functions $\varphi: \R^{(t+1)d} \rightarrow \R \in \PL_2(L)$, $L < \infty$,
    \begin{align}
        &\plim_{n \rightarrow \infty} \left| \frac{1}{n} \sum_{i=1}^{n} \varphi( X_i, X_i^1, \dotsc, X_i^t ) - \E[\varphi(X_\star, H^1, \dotsc, H^t)] \right| = 0.
    \end{align}
\end{cor}
\begin{proof}
    This result follows from Corollary~\ref{cor:amp_mtp_symmetric_wasserstein_convergence_result} in light of the embedding in Appendix~\ref{sec:amp_mtp_asymmetric_theorem_proof}.
\end{proof}

In our formulation, we see that the correction terms $\{\mB^t\}_{t \in \N}$ are defined in terms of the expected divergence matrices $\{\mD^t\}_{t \in \N}$, which are always well-defined and finite as they have bounded entries but may be difficult to compute. In practice, the AMP iterations are usually computed using empirical estimates of $\mD^t$ defined element-wise as
\begin{align}
    \Hat{\mD}^t_{jk} = \frac{1}{n} \sum_{i=1}^n \frac{\partial }{\partial X^t_{ij}} [f_t(\bX^t)]_{ik}, \quad j,k \in [d]. \label{eq:amp_mtp_asymmetric_emprical_divergence_definition}
\end{align}
As long as such functions are uniformly pseudo-Lipschitz, Theorem~\ref{th:amp_mtp_state_evolution_theorem_asymmetric} guarantees that $\Hat{\mD}^t$ is a consistent estimators for $\mD^t$ and SE is accurate also for the algorithm being run with the empirical estimates as correction terms. In fact, it is known that this holds in general for consistent estimators of $\mD^t$. We have the following corollary.

\begin{cor} \label{cor:amp_mtp_asymmetric_empirical_correction_terms}
    Assume the asymmetric AMP algorithm $\{ \bX^t \mid \bM^0, \mA_k^t, f_t \}_{t \in \N}$ is being run with the correction terms $\{\mB^t\}_{t \in \N}$ of \eqref{eq:amp_mtp_asymmetric_correction_term} being replaced by some consistent estimators $\{ \Hat{\mB}^t\}_{t \in \N}$, with $\hat{\mB}^t \parrow \mB^t$ as $n \rightarrow \infty$. Then, the conclusions of Theorem~\ref{th:amp_mtp_state_evolution_theorem_asymmetric} hold unchanged.
\end{cor}
\begin{proof}
    This result follows from Corollary~\ref{cor:amp_mtp_symmetric_empirical_correction_terms} in light of the embedding in Appendix~\ref{sec:amp_mtp_asymmetric_theorem_proof}.
\end{proof}

\subsection{The symmetric case} \label{sec:amp_mtp_main_result_symmetric}
Another case of interest is when the observation we intend to use AMP on are assumed to have the form of the symmetric MTP of \eqref{eq:MTP_mv_symmetric_def}. Besides being a useful model for some data-genrating processes in its own right, the symmetric formulation allows to establish Theorem~\ref{th:amp_mtp_state_evolution_theorem_asymmetric} via an embedding argument using results for the symmetric case, which we present in this section.

For any given $n \in \N$, we assume we are given observations $\{\bY_k \in \R^{n \times n}\}_{k \in [K]}$ as in \eqref{eq:MTP_mv_symmetric_def} to estimate a signal $\bX \in \R^{n \times d}$. Again, for compatibility with the AMP recursion we will rescale the observations $\bY_k$ by a factor of $1/\sqrt{n}$, similarly to the asymmetric setting, that is
\begin{align}
    \bY_k = \frac{1}{n} \bX \mLambda_k \bX^\top + \frac{1}{\sqrt{n}} \bG_k.
\end{align}

We assume we have the following quantities.
\begin{itemize}
    \item an initialization matrix $\bM^0 \in \R^{n \times d}$;
    \item for each $t \in \N$, a sequence of uniformly Lipschitz denoisers $\{ f_t : \R^{n \times d} \rightarrow \R^{n \times d} \}_{n \in \N} \subset \Lip(L)$, $L < \infty$;
    \item for each $t \in \N$, a collection of \textit{reweighting} matrices $\{\mA_k^t \in \R^{d \times d}\}_{k \in [K]}$.
\end{itemize}

For this symmetric model, we make assumptions that are totally analogous to the ones for the asymmetric case. 
\begin{enumerate}[label=$(\mathrm{S\arabic*})$]
    \item For $k \in [K]$, $\bG_k \iid \GOE(n)$ independent of $\bX$. Furthermore, for $p\geq 1, L> 0$ and any sequences of test functions $\{ \phi_n : \R^{n \times d}\rightarrow \R \}_{n \in \N} \subset \PL_p(L)$, the signal $\bX$ is such that 
    \begin{align}
        \plim_{n \rightarrow \infty} \left| \phi_n(\bX) - \E[ \phi_n(\bX)] \right| = 0. 
    \end{align}
    Additionally, we have $\E \|\bX\|^2 \rightarrow C < \infty$ as $n \rightarrow \infty$. \label{as:mtp_model_amp_signal_assumptions_symmetric}
    \item For each $t \in \N$, the \textit{denoiser sequence} $\{ f_t: \R^{n \times d} \rightarrow \R^{n \times d} \}_{n \in \N}$ is uniformly Lipschitz in with some constants $L < \infty$. Furthermore, fix any positive semi-definite matrices $\mS \in \PSD^{2d}$ and define a pair Gaussian matrices $(\bZ, \bZ')$ such that $(\bZ, \bZ') \sim \Normal(0, \mS \otimes \Id_{n})$, independently of $\bX$. For any given $\mK\in \R^{d \times d}$, denote $\bH \coloneqq \bX \mK + \bZ, \bH' \coloneqq \bX \mK' + \bZ'$. Then, for any $t,s \in \N$, $f_t$ is such that
    \begin{align}
        &\frac{1}{n} \E\!\left[ \bX^\top f_t(\bH) \right]; \quad
        \frac{1}{n} \E\!\left[ f_t(\bH)^\top f_t(\bH)  \right]; \quad
        \frac{1}{n} \E\!\left[ f_t(\bH)^\top f_s(\bH')  \right] 
    \end{align}
    have well-defined and finite limits as $n \rightarrow \infty$. \label{as:mtp_model_amp_denoiser_assumption_symmetric}
    \item The \textit{initialization matrix} $\bM^0 \in \R^{n \times d}$ is such that, letting $\bH$ as in \ref{as:mtp_model_amp_denoiser_assumption_symmetric}, the following functions
    \begin{align}
        &\frac{1}{n} \E\!\left[ \bX^\top \bM^0 \right]; \quad
        \frac{1}{n} \E\!\left[ (\bM^0)^\top \bM^0  \right]; \quad
        \frac{1}{n} \E\!\left[ (\bM^0)^\top f_t(\bH) \right]
    \end{align}
    have well-defined and finite limits as $n_1, n_2 \rightarrow \infty$ for all $t \in \N$. Furthermore,
    \begin{align}
        \plim_{n\rightarrow\infty}\frac{1}{n} \bX^\top \bM^0 &= \lim_{n \rightarrow \infty} \frac{1}{n} \E\!\left[ \bX^\top \bM^0 \right]; \\
        \plim_{n\rightarrow\infty}\frac{1}{n} (\bM^0)^\top \bM^0 &= \lim_{n \rightarrow \infty} \frac{1}{n} \E\!\left[ (\bM^0)^\top \bM^0 \right]; \\
        \plim_{n\rightarrow\infty}\frac{1}{n} (\bM^0)^\top f_t(\bH) &= \lim_{n \rightarrow \infty} \frac{1}{n} \E\!\left[ (\bM^0)^\top f_t(\bH) \right].
    \end{align}
    \label{as:mtp_model_amp_initialization_assumption_symmetric}
\end{enumerate}

With this, the AMP algorithm in the symmetric case is defined as the recursion below. For any $n \in \N$ and $t \in \N$, we write:
    \begin{align} 
        \bX^{t} &= \sum_{k=1}^K \bY_k \bM^{t-1} (\mA_k^{t})^\top - \bM^{t-2} (\mB^{t-1})^\top; \quad \bM^{t} = f_{t}(\bX^{t}). \label{eq:amp_algorithm_mtp_symmetric} 
    \end{align}
Again, by convention we take ${\mM^{-1}}$ and $\mB^0$ to be zero matrices. Analogously to the asymmetric case, $\{ \bX^t \in \R^{n \times d} \}_{t \in \N}$ are the \textit{iterates} of the symmetric iterations. The terms $\{ \mB^t \in \R^{d \times d} \}_{t \in \N}$ are \textit{correction} terms ensuring that the distributional properties of each iterate $\bX^t$ can be characterized exactly in the large $n$ limit and are constructed similarly to the asymmetric case. Their exact form is given in \eqref{eq:amp_mtp_correction_term_symmetric_definition}. We will refer to the symmetric iterations obtained via \eqref{eq:amp_algorithm_mtp_symmetric} as $\{ \bX^t \mid \bM^0, \mA_k^t, f_t \}_{t \in \N}$.

The recursively defined SE for the symmetric AMP iterations $\{ \bX^t \mid \bM^0, \mA_k^t, f_t \}_{t \in \N}$ is as follows. At the first iteration, we write
    \begin{align}
        \mK^1 &\coloneqq \sum_{k=1}^K \left\{ \mLambda_k \!\left(\lim_{n \rightarrow \infty} \frac{1}{n} \E[\bX^\top \bM^0]\right) (\mA_k^1)^\top \right\}; \\
        \mSigma^1 &\coloneqq \sum_{k=1}^K \left\{\mA_k^1 \!\left(\lim_{n \rightarrow \infty} \frac{1}{n} \E[(\bM^0)^\top \bM^0]\right) (\mA_k^1)^\top\right\}.
    \end{align}
Then, for $s < t \in \N \cup \{0\}$, we let $\bH^t \coloneqq \bX \mK^t + \bZ^t$, $\bZ^t \sim \Normal(0, \mSigma^t \otimes \Id_n)$ and have
\begin{align}
    \mK^{t+1} &\coloneqq \sum_{k=1}^K \left\{\mLambda_k \! \left(\lim_{n \rightarrow \infty} \frac{1}{n} \E\!\left[\mX^\top f_t(\bH^t)\right]\right) (\mA_k^t)^\top \right\}; \\
    \mSigma^{t+1} &\coloneqq \sum_{k=1}^K \left\{ \mA_k^t \!\left( \lim_{n \rightarrow \infty} \frac{1}{n} \E\!\left[(f_t(\bH^t)^\top f_t(\bH^t)\right] \right) (\mA_k^t)^\top \right\}; \\
    \mSigma^{s+1,t+1} &\coloneqq  \sum_{k=1}^K \left\{ \mA_k^s \!\left( \lim_{n \rightarrow \infty} \frac{1}{n} \E\!\left[(f_s(\bH^s)^\top f_t(\bH^t)\right]\right) (\mA_k^t)^\top \right\},
\end{align}
where we adopt the convention $f_0 \equiv \bM^0$. The terms $\mSigma^{s,t}$ represents the covariance structure among the rows of the limiting Gaussian noise matrices $\bZ^s$ and $\bZ^t$.

The correction terms $\{\mB^t\}_{t \in \N}$ of the AMP recursion are defined as in the asymmetric case as function of the expected divergence matrices $\{ \mD^t \in \R^{d \times d} \}_{t \in \N}$, whose $jk$-th entry is given by
\begin{align} \label{eq:amp_symmetric_expected_mean_divergence_definition}
    [\mD^t]_{jk} = \frac{1}{n} \sum_{i=1}^n \E\!\left[ \frac{\partial}{\partial H_{ik}^t} [f_t(\bH^t)]_{ij} \right].
\end{align}
With this, we define each correction term $\mB^t$ as
\begin{align}
    \mB^t \coloneqq \sum_{k=1}^K \left\{ \mA_k^{t+1} \mD^t \mA_k^t \right\}. \label{eq:amp_mtp_correction_term_symmetric_definition}
\end{align}

This concludes the introduction of all quantities necessary to present the AMP convergence theorem for the symmetric MTP model \eqref{eq:MTP_mv_symmetric_def}.

\begin{theorem} \label{th:amp_mtp_symmetric_main_result}
    Under assumptions \ref{as:mtp_model_amp_signal_assumptions_symmetric}-\ref{as:mtp_model_amp_initialization_assumption_symmetric}, the symmetric AMP iterations $\{ \bX^t \mid \bM^0, \mA_k^t, f_t \}_{t \in \N}$ are such that, for any $t \in \N, p \geq 1, L < \infty$ and sequences of test functions $\{ \phi_n: \R^{n \times (t+1)d} \rightarrow \R \}_{n \in \N} \subset \PL_p(L)$, one has
    \begin{align}
        \plim_{n \rightarrow \infty} \left| \phi_n(\bX, \bX^1, \dotsc, \bX^t) - \E\!\left[ \phi_n(\bX, \bH^1, \dotsc, \bH^t) \right] \right| = 0.
    \end{align}
\end{theorem}
\begin{proof}
    Appendix~\ref{sec:mtp_amp_symmetric_proof} is dedicated to the proof of the result.
\end{proof}

For the symmetric MTP, we have direct analogues of Corollaries~\ref{cor:amp_mtp_asymmetric_wasserstein_convergence_result} and \ref{cor:amp_mtp_asymmetric_empirical_correction_terms}. For the former, we need an empirical measure convergence assumption akin to \eqref{eq:mtp_signal_asymmetric_wasserstein_limits_assumption}. Assume there exist a random variable $X_\star \in \R^d$ with bounded second moment such that, for all functions $\varphi : \R^d \rightarrow \R^d \in \PL_2(L)$, $L<\infty$, the signal sequence $\bX_\star\in\R^{n\times d}$ satisfies
\begin{align} \label{eq:mtp_signal_symmetric_wasserstein_limits_assumption}
    \plim_{n \rightarrow \infty} \left| \frac{1}{n} \sum_{i=1}^n \varphi(X_i) - \E[\varphi(X_\star)] \right| = 0.
\end{align}

\begin{cor} \label{cor:amp_mtp_symmetric_wasserstein_convergence_result}
    Assume \ref{as:mtp_model_amp_signal_assumptions_symmetric}-\ref{as:mtp_model_amp_initialization_assumption_symmetric} hold and furthermore $\bX$ satisfies \eqref{eq:mtp_signal_symmetric_wasserstein_limits_assumption}. Define $H^t \coloneqq (\mK^t)^\top X + Z^t$, for $Z^t \sim \Normal(0, \mSigma^t) \indep X$. Then, the symmetric AMP iterations $\{ \bX^t \mid \bM^0, \mA_{k}^t, f_t\}_{t \in \N}$ are such that, for any test function $\varphi : \R^{(t+1)d} \rightarrow \R \in \PL_2(L)$,
    \begin{align}
        &\plim_{n \rightarrow \infty} \left| \frac{1}{n} \sum_{i=1}^{n} \varphi(X_i, X_i^1, \dotsc, X_i^t) - \E[\varphi(X, H^1, \dotsc, H^t)] \right| = 0.
    \end{align}
\end{cor}
\begin{proof}
    This result is proved in Appendix~\ref{sec:symmetric_mtp_corollaries_proof}.
\end{proof}

\begin{cor} \label{cor:amp_mtp_symmetric_empirical_correction_terms}
    Assume the symmetric AMP algorithm $\{ \bX^t \mid \bM^0, \mA_{k}^t, f_t \}_{t \in \N}$ is being run with the correction terms $\mB^t$ of \eqref{eq:amp_mtp_correction_term_symmetric_definition} being replaced by some consistent estimators $\Hat{\mB}^t$. Then, the conclusions of Theorem~\ref{th:amp_mtp_symmetric_main_result} hold as well.
\end{cor}
\begin{proof}
    See Appendix~\ref{sec:symmetric_mtp_corollaries_proof}.
\end{proof}

\section{Bayes-optimal inference with AMP}
All AMP iterations presented in the previous section are quite flexible, in the sense that one has large freedom in choosing the denoising functions and reweighting matrices that drive the algorithm. Here, we will focus on a particular setup for the AMP iterations, namely the so-called \textit{Bayes-optimal} setting. In particular, we will instantiate our analysis with the symmetric model \eqref{eq:MTP_mv_symmetric_def}, since as was shown is Section~\ref{sec:mtp_amp_symmetric_proof} it is sufficient to also represent the asymmetric version \eqref{eq:MTP_mv_def}, at the cost of having to account for block-constraints in the signal and coupling matrices.

In practice, for $t \in \N$, the Bayes-optimal setting amounts to choosing as denoiser $\{f_t : \R^{n \times d} \rightarrow \R^{n \times d} \}_{n \in \N}$ the conditional mean estimator for $\bX$ under a linear Gaussian observation model given by $\bH^t = \bX \mK^t + \bZ^t$, for $\bZ \sim \Normal(0, \mSigma^t \otimes \Id_n) \indep \bX$. The parameters $\mK^t$ and $\mSigma^t$ are given by SE, and $\bH^t$ is the random sequence associated with iterate $\bX^t$, as in Theorem~\ref{th:amp_mtp_symmetric_main_result}, that approximates the behavior of $\bX^t$ for sufficiently large problem sizes.  For each $n \in \N$, the Bayes rule at iteration $t$ (which we will denote with $\eta_t : \R^{n \times d} \rightarrow \R^{n \times d}$) is then given by
\begin{align}
    \eta_t(y) = \E\!\left[ \bX \mid \bH^t = y \right] = \frac{ \int_{\R^{n \times d}} x \exp\!\left\{ - \frac{1}{2} \langle (y - x\mK^t)(\mSigma^t)^{-1/2},  (y - x\mK^t) \rangle \right\} \ P_{\bX}(d x) }{ \int_{\R^{n \times d}} \exp\!\left\{ - \frac{1}{2} \langle (y - x\mK^t)(\mSigma^t)^{-1/2},  (y - x\mK^t) \rangle \right\} \ P_{\bX}(d x) } .
\end{align}
Notice that $\eta_t$ is implicitly parametrized by the state evolution $\mK^t, \mSigma^t$. We will assume, going forward, that the prior sequence $P_{\bX}$ on $\bX$ is such that $\eta_t$ satisfies the uniformly Lipschitz assumption for all pairs of SE parameters $\mK^t, \mSigma^t$. The idea behind such choice is that of being able to characterize the high-dimensional limit of the squared error loss incurred in estimating $\bX$ by $\eta_t(\bX^t)$ as the minimum mean-squared error for a Gaussian linear observation $\bH^t$ of $\bX$. By Lemma~\ref{lem:inner_products_uniformly_lipschitz_functions}, in fact, function $\phi_n: \R^{n \times 2d} \rightarrow \R^{d \times d}$ defined as 
\begin{align}
    \phi_n(\bX, \bX^t) \coloneqq \frac{1}{n} (\bX - \eta_t(\bX^t))^\top (\bX - \eta_t(\bX^t))
\end{align}
is uniformly pseudo-Lipschitz of order 2, and as $n$ grows large it will be approximated with high probability the matrix MMSE
\begin{align}
    M_n( \mK^t, \mSigma^t) = \frac{1}{n} \E\!\left[ (\bX - \eta_t(\bH^t))^\top (\bX - \eta_t(\bH^t)) \right],
\end{align}
whose trace is the usual average squared error loss. As observed in \cite{Reeves:2018aa}, observations $\bH^t$ are statistically equivalent to linear observations in white Gaussian noise $\bZ^t \sim \Normal(0, \Id_d \otimes \Id_n)$ with a positive semi-definite SNR given by $\left(\mK^t(\mSigma^t)^{-1}(\mK^t)^\top\right)^{1/2}$, for $\mSigma^t \in \PD^d$. When $\mSigma^t$ is singular, however, this equivalence is in general not well-defined as some directions of no variance may correspond to exact observations that would otherwise be lost by the symmetrization. 

When the matrices $K^t$ and $\mSigma^t$ stem from state evolution, however, it is possible to simply replace the inverse $(\mSigma^t)^{-1}$ by its symmetric pseudo-inverse $(\mSigma^t)^\dagger$ as the column space of $\mK^t$ ends up aligning to the row space of $\mSigma$. Because of this, from now on we will always be parametrizing $\bH^t$ by a single \textit{effective SNR} matrix $\mS^t \in \PSD^d$ constructed as $\mS^t = \mK^t(\mSigma^t)^\dagger(\mK^t)^\top$, and denote the corresponding MMSE matrix $M_n( \mS^t)$.

\subsection{Optimal reweighting} \label{sec:bayes_optimal_reweighting}
In our specification of the AMP algorithm $\{ \bX^t \mid \bM^0, \mA_k^t, f_t \}_{t \in \N}$ for the MTP, the role of the reweighting matrices $\mA_k^t$ is that of aggregating the block-diagonal elements in the lifted recursion that appears in Appendix~\ref{sec:mtp_amp_symmetric_proof}, in order to return an iterate whose dimension is conformable with the signal of interest $\bX$. Actually, the lifted iterates $\{ \tX^t | \tM^0, F_t \}_{t \in \N}$ could in principle contain more information about the signal $\bX$ than what is contained in $\bX^t$. 

Here, we leverage the fact that under Theorem~\ref{th:amp_mtp_symmetric_main_result} the iterates $\bX^t$ are approximated by the Gaussian channel observations $\bH^t$ to argue that a carefully chosen set of reweighting matrices is sufficient to achieve the same estimation performance in the limit as the one attainable via the lifted recursion $\{ \tX^t | \tM^0, F_t \}_{t \in \N}$. Informally, we can condition on having obtained an estimate $\bM^t$ for $\bX$ from both arguments. Then, we can compare the approximating random sequence $\tH^{t+1}$ associated with the lifted algorithm to the one that is occurs after the reweighting step by $T_{t+1}$, i.e. $\bH^{t+1} = T_{t+1}(\tH^{t+1})$, for the linear operator $T_{t+1}$ constructed from $\{\mA_k^{t+1}\}_{k \in [K]}$ as in Appendix~\ref{sec:mtp_amp_symmetric_proof}. $\tH^{t+1}$ is statistically equivalent to collection $\{\bH_k^{t+1}\}_{k \in [K]}$ of observations of $\bX$ in independent additive Gaussian noise,
\begin{align}
    &\bH_k^{t+1} = \bX \mLambda_k \mF^{t+1} + \bZ_k^{t+1}, \quad \bZ_k \iid \Normal(0, \mQ^{t+1} \otimes \Id_n); \\
    &\mF^{t+1} \coloneqq \plim_{n \rightarrow \infty} \frac{1}{n} \bX^\top \bM^t ; \quad \mQ^{t+1} \coloneqq \plim_{n \rightarrow \infty} \frac{1}{n} (\bM^t)^\top \bM^t, \label{eq:def_state_evolution_no_coupling_matrices}
\end{align}
compared to the single observation 
\begin{align}
    &\bH^{t+1} = \sum_{k=1}^K \bH_k^{t+1} (\mA_k)^{t+1} = \bX \mK^{t+1} + \bZ^{t+1} , \quad \bZ^{t+1} \sim \Normal(0, \mSigma^{t+1} \otimes \Id_n); \\
    &\mK^{t+1} = \sum_{k=1}^K \mLambda_k \mF^{t+1} (\mA_k^{t+1})^\top ; \quad \mSigma^{t+1} = \sum_{k=1}^K \mA_k^{t+1} \mQ^{t+1} (\mA_k^{t+1})^\top.
\end{align}

Clearly, since $\bH^{t+1} = T_{t+1}(\tH^{t+1})$, Bayes-optimal posterior inference based on $\tH^{t+1}$ will yield an MMSE no larger than the MMSE attainable by using $\bH^{t+1}$, with equality achievable if $T_{t+1}$ produces a Bayesian sufficient statistic for $\bX$. We formalize this intuition in the following result.

\begin{prop} \label{prop:amp_mtp_optimal_reweighting_choice}
    The Bayes-optimal AMP recursion $\{ \bX^t \mid \bM^0, \mA_k^t, \eta_t \}_{t \in \N}$ is such that, for each $t\in \N$, the optimal reweighting choice (in terms of minimizing the mean-squared error) is $\mA_k^t = \mLambda_k \mK^t (\mQ^t)^\dagger$, where $\mK^t, \mSigma^t$ are defined as in \eqref{eq:def_state_evolution_no_coupling_matrices}. Furthermore, such choice achieves the same asymptotic MMSE as the the Bayes-optimal lifted recursion $\{ \tX^t \mid \tM^0, F_t \}_{t \in \N}$ under matching initialization $\Id_k \otimes \bM^k$.
\end{prop}

\begin{proof}
    The result can be shown by iteratively matching the conditional mean estimators under the limiting distributions. At the first step, we have that the approximating random variables for $\bX^1$ and $\tX^1$ (denoted $\bH^1$ and $\tH^1$, respectively) imply the same conditional densities for $\bX$. We have
    \begin{align}
        p(x \mid \tH^1 = y) = P_{\bX}(d x) \exp\!\left\{ - \frac{1}{2} \sum_{k=1}^K \langle x \mLambda_k \mF^1 (\mQ^1)^\dagger , x \mLambda_k \mF^1 \rangle + \langle S_1(y), x \rangle  - K(S_1(y) ) \right\}
    \end{align}
    where $S_1(y) \coloneqq \sum_{k \in [K]} y_{(k)} (\mLambda_k \mF^1 (\mQ^1)^\dagger)^\top$, $y_{(k)}$ is the $n \times d$ $k$-th diagonal block of $y \in \R^{nK \times dK}$ and
    \begin{align}
        K(S_1(y)) = \log \!\left( \int \exp\!\left\{ - \frac{1}{2} \sum_{k=1}^K \langle x \mLambda_k \mF^1 (\mQ^1)^\dagger , \mX \mLambda_k \mF^1 \rangle + \langle S_1(y), \mX \rangle \right\} P_{\bX}(dx) \right)
    \end{align}
    is the cumulant generating function for the density. In other words, $S_1(\tH)$ is a sufficient statistic (in the Bayesian sense) for $\bX$, and we notice that choosing $\mA_k^1 = \mLambda_k \mF^1 (\mQ^1)^\dagger$ one gets $\bH^1 = S_1(\tH^1)$. Thus, the limiting MMSE approximated by the lifted recursion (after appropriate normalization) must equal the one approximated by the Bayes-optimal AMP as the conditional mean estimators must be the same. Since the conditional mean estimators are matched, $\tM^1 = \Id_K \otimes \bM^1$ and it is a simple induction to show that the claim holds for all $t \in \N$. 
\end{proof}

\begin{remark}
    Under the Bayes-optimal choice of denoisers, after the first iteration (for which the parameters $\mF^1$ and $\mQ^1$ depend on the structure of the initialization $\bM^0$), the law of iterated expectation guarantees that SE simplifies and $\mF^t = \mQ^t$ for all $t \geq 2$. The optimal reweighting thus simply becomes the choice $\mA_k^t = \mLambda_k$ for all $t$. To simplify the presentation, we will assume that also $\mF^1 = \mQ^1$. This, then, gives a reduced state evolution $\{\mS^t \in \PSD^d \}_{t \in \N}$ where $\mS^t = \sum_{k \in [K]} \mLambda_k \mQ^t \mLambda_k$. Going forward, we refer to the sequences $\{\mQ^t, \mS^t\}_{t \in \N}$ as the \textit{overlaps} and the \textit{effective SNR} (in short SNR) for the AMP recursion, and we relate them via a linear mapping $\cT: \R^{d \times d} \rightarrow \R^{d \times d}$ given by $\cT(\mQ) = \sum_{k \in [K]} \mLambda_k \mQ \mLambda_k$. We will furthermore assume that in the Bayes-optimal case one is reweighting optimally as well.
\end{remark}

\subsection{Properties of the MMSE matrix}
In order to characterize the high-dimensional estimation properties of the Bayes-optimal AMP algorithm, we need to study the large-$n$ properties of the matrix MMSE function, as it directly relates to both the orbits of the SE recursion and the asymptotic estimation performance of the algorithm itself. An aspect of particular interest is understanding the continuity and smoothness properties of the limiting MMSE matrix seen as a function of the SNR level $\mS^t \in \PSD^d$
\begin{align}
    M(\mS^t) = \lim_{n \rightarrow \infty} M_n( \mS^t) = \lim_{n \rightarrow \infty} \frac{1}{n}\E\!\left[\bX^\top \bX\right] - \lim_{n \rightarrow \infty} \frac{1}{n} \E\!\left[ \E[\bX \mid \bH^t]^\top \E[\bX \mid \bH^t] \right].
\end{align}
For our purposes, establishing the first-order differential properties of $M(\mS^t)$ will be of fundamental importance. In this section, we provide sufficient conditions for $M$ to be differentiable on the relative interior of $\PSD^d$ and admit directional derivatives on the boundary of the positive semi-definite cone, corresponding to the notion of the effective observation $\bH^t$ carrying signal over a subspace of the support of $\bX$. This eventuality is particularly relevant for our setup, as the coupling matrices $\mLambda_k$ in the MTP are assumed to be unrestricted and may in principle operate on a subspace. Indeed, the embedding structure use to prove Theorem~\ref{th:amp_mtp_state_evolution_theorem_asymmetric} incorporates by construction some subspace constraints.

We begin by characterizing the fixed-$n$ MMSE matrix $M_n( \mS)$ as a partial trace over the MMSE matrix associated with the vector linear channel with signal $\gvec(\bX) \in \R^{nd}$. Letting $\bZ \sim \Normal(0, \Id_{d} \otimes \Id_n)$ independent of $\bX$ and $\bH_{\mS} = \bX\mS^{1/2} + \bZ$, the MMSE matrix $\E\!\left[ \Cov(\gvec(\bX) \mid \bH_{\mS}) \right] \in \PSD^{nd}$ associated with this vector channel can be represented as the quadruply-indexed collection of expected conditional covariances $\E\!\left[ \Cov(X_{ik}, X_{jl} \mid \bH_{\mS}) \right]$, for $i,j \in [n], k,l \in [d]$. 

It is immediate to see that, by definition the $kl$-th entry of $M_n( \mS)$, for $k, l \in [d]$, is given by
\begin{align}
    [M_n( \mS)]_{kl} = \E\!\left[ \frac{1}{n} \sum_{i=1}^n \Cov(X_{ik}, X_{il} \mid \bH_{\mS}) \right].
\end{align}
Letting $\tr_n$ denote the partial trace over $\R^{n \times n}$ for $\R^{d \times d} \otimes \R^{n \times n}$, the above relation can be represented in matrix form as $M_n( \mS) = \E\!\left[ \tr_n\{ \Cov(\gvec(\bX) \mid H_{\mS}) \}/n \right]$. A first observation to be made is that, under our assumptions \ref{as:mtp_model_amp_signal_assumptions_symmetric}-\ref{as:mtp_model_amp_initialization_assumption_symmetric} needed for Theorem~\ref{th:amp_mtp_symmetric_main_result} to hold, the $n$-limit of $M_n( \mS)$ is indeed well-defined for all $\mS \in \PSD^d$ and can be approximated directly as the concentration limit of the argument of the expectation, since by Lemma~\ref{lem:inner_products_uniformly_lipschitz_functions} all maps $(\bX, \bH_{\mS}) \mapsto n^{-1}\sum_{i \in [n]} \Cov(X_{ik}, X_{il} \mid \bH_{\mS})$ are uniformly pseudo-Lipschitz of order 2. 

A further known property \cite{Payaro:2009aa} of the vectorized MMSE matrix is the form of its gradient with respect to a general input SNR $\Tilde{\mS} \in \PD^{nd}$, which is given by the Kronecker square of the conditional covariance matrix:
\begin{align}
    \nabla_{\Tilde{\mS}} \E\!\left[\Cov(X \mid H_{\Tilde{\mS}})\right] = - \E\!\left[ \Cov(X \mid H_{\Tilde{\mS}}) \otimes \Cov(X \mid H_{\Tilde{\mS}}) \right].
\end{align}
Here with a slight abuse of notation we denoted $H_{\Tilde{\mS}} = \Tilde{\mS} \gvec(\bX) + Z$, $Z \sim \Normal(0, \Id_{nd}) \indep \bX$. Unsurprisingly, the gradient $\nabla M_n( \mS)$ has a similar structure. First, we establish a notation for the partitioning of $\Cov(\gvec(\bX) \mid \bH_{\mS})$ into $d^2$ $n \times n$ sub-matrices. For $j, k \in [d]$ and $i,j \in [n]$,
\begin{align}
    \Cov(\gvec(\bX) \mid \bH_{\mS}) = \begin{bmatrix}
        \mC_{11} & \mC_{12} & \hdots & \mC_{1d} \\
        \mC_{21} & \mC_{22} & \hdots & \mC_{2d} \\
        \vdots & \vdots & \ddots & \vdots \\
        \mC_{d1} & \mC_{d2} & \hdots & \mC_{dd}
    \end{bmatrix}, \quad [\mC_{kl}]_{ij} = \Cov(X_{ik}, X_{jl} \mid \bH_{\mS}), \ \mC_{kj} = \mC_{jk}^\top.
\end{align}
With this, we build the $d^2 \times d^2$ positive semi-definite matrix $\mPsi_n(\bH_{\mS}) \in \PSD^{d^2}$ defined element-wise as
\begin{align}
    [\mPsi_n(\bH_{\mS})]_{((i-1)d+j)((k-1)d+l)} = \frac{1}{n} \langle \mC_{ik} , \mC_{jl} \rangle, \quad {i,j,k,l} \in [d].
\end{align}
The mapping $\Cov(\gvec(\bX) \mid \bH_{\mS}) \mapsto \mPsi_n(\bH_{\mS})$ can be seen as a generalized Kronecker product acting on matrices of $n \times n$ matrices, with the usual scalar product replaced by the trace inner product. 
\begin{prop} \label{prop:matrix_mmse_gradient_identity}
    For any $\mS \in \PD^d$ and $n \in \N$, the gradient of the matrix MMSE $M_n( \mS)$ is given by
    \begin{align}
        \nabla M_n( \mS) = -\E\!\left[ \mPsi(\bH_{\mS}) \right] \label{eq:matrix_mmse_gradient_identity}
    \end{align}
\end{prop}
\begin{proof}
    The calculation is showed in Appendix~\ref{sec:mmse_matrix_gradient_identity_proof}.
\end{proof}
The gradient \eqref{eq:matrix_mmse_gradient_identity} is to be understood as a map acting on $\nabla M_n( \mS): \R^{d^2} \rightarrow \R^{d^2}$ representing the Jacobian of the vectorized MMSE map $\gvec(\mS) \mapsto \gvec(M_n(\mS))$, for $\mS \in \PD^d$. Alternatively, $\nabla M_n( \mS)$ can be expressed directly as a linear operator, parametrized by $\mS$, $ \nabla M_n(\mS): \PD^d \rightarrow \PD^d$ given by
\begin{align} \label{eq:mmse_gradient_linear_operator_representation}
    -\big( \nabla M_n(\mS) \big)(\mT)& = \frac{1}{n} \tr_n\!\left\{ \E\!\left[ \Cov(\gvec(\bX) \mid H_{\mS}) (T \otimes \Id_n) \Cov(\gvec(\bX) \mid H_{\mS})\right] \right\}
\end{align}
With these gradients, we can verify that for each $n \in \N$ the map $M(\mS)$ is Lipschitz under a bounded fourth moment assumption.
\begin{prop} \label{prop:mmse_function_fixed_n_lipschitz}
    For any $n \in \N$, the MMSE function $M_n(\mS)$ is Lipschitz continuous on $\PSD^d$ and has well-defined directional gradients on the boundary of $\PSD^d$ as long as $\E\| \bX \|_{\op}^4 < \infty$. 
\end{prop}
\begin{proof}
    From the defintion of operator norm, we have that, for all $\mS \in \PD^d$,
    \begin{align}
        \| \nabla M_n(\mS) \|_{\op} &= \max_{\{\mV \in \PSD^d: \|\mV\|_F \leq 1\}} \| \E [\mPsi(\bH_{\mS})\gvec(\mV) ] \|. 
    \end{align}
    Noticing that $\langle \mC_{ik} , \mC_{jl} \rangle \leq \E [ [\bX^\top \bX]_{ik} [\bX^\top \bX]_{jk} \mid \bH_{\bS} ] $, we can write
    \begin{align}
        \max_{\{\mV \in \PSD^d: \|\mV\|_F \leq 1\}} \| \E [\mPsi(\bH_{\mS})\gvec(\mV) ] \| &\leq \frac{1}{n} \max_{\{\mV \in \PSD^d: \|\mV\|_F \leq 1\}} \| \E[\bX^\top\bX \mid \bH_{\mS}] \mV \E[\bX^\top\bX \mid \bH_{\mS}] \|_F \\ 
        & = \frac{1}{n} \max_{\{\mv \in \R^d: \|\mv\| \leq 1\}} \| \E[\bX^\top\bX \mid \bH_{\mS}] \mv \mv^\top \E[\bX^\top\bX \mid \bH_{\mS}] \|_F \\
        & = \frac{1}{n} \max_{\{\mv \in \R^d: \|\mv\| \leq 1\}} \| \E[\bX^\top\bX \mid \bH_{\mS}]^{\otimes 2} \mv^{\otimes 2} \| \\
        & \leq \frac{1}{n} \E \| \bX^\top \bX\|_{\op}^2 \\
        & = \frac{1}{n} \E \| \bX\|_{\op}^4.
    \end{align}
    From this both claims follow immediately.
\end{proof}

While the bound in Proposition~\ref{prop:mmse_function_fixed_n_lipschitz} provides a sufficient condition to ensure that, for fixed $n$, $M_n$ is Lipschitz and differentiable everywhere (including at the boundary), it is not sufficiently tight to guarantee that such properties of $M_n$ carry over to the limit. Indeed, even for a random signal $\bX$ of i.i.d. draws from sub-Gaussian distribution, classical results in random matrix theory \cite{Bai:1999aa} imply that $\E\|\bX^\top\bX\|_{\op} = \Theta(n)$, and thus the above upper bound on the Lipschitz constant is diverging as $n \rightarrow \infty$. 

To pass to the limit and conduct our intended analysis of state evolution, then, we need some form of ``weak row-correlation" regularity assumption, in the form of a uniform boundedness assumption for the entries of the matrix $\E [\mPsi(\bH_{\mS})]$. Going forward, we will be strengthening the Lipschitz assumption on the Bayes-denoisers as follows.
\begin{enumerate}[label=$(\mathrm{WC})$] 
    \item The prior sequence $\bX$ is such that, for all $n \in \N$ and $\mS \in \PSD^d$, the Bayes-optimal denoisers $\eta(\; \cdot \;; \mS) : y \mapsto \E[\bX \mid \bH_{\mS} = y]$ are uniformly Lipschitz as a sequence in $n$. Furthermore there exists a constant $B < \infty$ such that $\| \E[\mPsi(\bH_{\mS})] \|_{\op} \leq B$. \label{as:signal_weak_correlation_assumption}
\end{enumerate}

From now on, we will be considering prior sequences such that \ref{as:signal_weak_correlation_assumption} holds, so that the limiting MMSE matrix function $M : \PSD^d \rightarrow \PSD^d$ that characterizes the high-dimensional performance of the AMP iterates is Lipschitz continuous. It is also convenient to establish notation for the limiting \textit{overlap function} $\psi: \PSD^d \rightarrow \PSD^t$, defined as
\begin{align}
    \psi(\mS) \coloneqq \lim_{n \rightarrow \infty} \E\!\left[ \E[\bX \mid \bH_{\mS}]^\top \E[\bX \mid \bH_{\mS}] \right] = \lim_{n \rightarrow \infty} \frac{1}{n} \E[\bX^\top \bX] - M(\mS).
\end{align}
With this, limiting SE can be expressed as a non-linear recurrence relation that is a function of $\psi$ and the linear transformation $\cT$. For an initial overlap level $\mQ^1$ and $t \in \N$, (Bayes-optimal) state evolution is given by
\begin{align} \label{eq:bayes_optimal_state_evolution_recusion_overlap_argument}
    \mQ^{t+1} = \psi( \cT(\mQ^t) )
\end{align}
or, by letting $\mS^t \coloneqq \cT(\mQ^t)$ denote the effective SNR for the AMP iteration at time $t$, SE can be rephrased in terms of the SNR sequence as
\begin{align}
    \mS^{t+1} = \cT( \psi(\mS^t) ).
\end{align}

We conclude this section with another useful property of the gradient of the overlaps $\nabla \psi(\mS)$ (or, equivalently, of $\nabla M(\mS)$), namely its self-adjointness under assumption \ref{as:signal_weak_correlation_assumption}.
\begin{prop} \label{prop:overlap_gradient_is_self_adjoint}
    Under assumption \ref{as:signal_weak_correlation_assumption}, the linear mapping $\PSD^d \rightarrow \PSD^d$ given by $\nabla \psi(\mS)$ is self-adjoint for all $\mS \in \PSD$.
\end{prop}
\begin{proof}
    For any $n \in \N$ and $\mS \in \PSD^d$, positive-definiteness of the conditional covariance matrix $\Cov(\gvec(\bX) \mid \bH_{\mS})$ implies that the map $\PSD^{nd} \rightarrow \PSD^{nd}$ given by $\mX \mapsto \E\!\left[ \Cov(\gvec(\bX) \mid \bH_{\mS}) \mX \Cov(\gvec(\bX) \mid \bH_{\mS}) \right] / n$ is self-adjoint. Since the map $\PSD^{d} \rightarrow \PSD^{d}$, $\mS \mapsto \mS \otimes \Id_n$ is adjoint to the partial trace operation $\tr_n$, it follows that the composition of the three must be self-adjoint linear operator. Furthermore, from \ref{as:signal_weak_correlation_assumption}, the limiting linear map $\nabla\psi(\mS)$ is the pointwise limit of the sequence of the operators given by \eqref{eq:mmse_gradient_linear_operator_representation}, which are all self-adjoint. We conclude that $\nabla\psi(\mS)$ is itself self-adjoint for all $\mS \in \PSD^d$.
\end{proof}

\section{Computational Limits for the MTP} \label{sec:computational_limits}
When running an AMP algorithm with some initialization $\bM^0$ and generic denoiser sequence $f_t$, the expectation is that the sequence of iterates $\bX^t$ achieve progressively higher effective SNR $\mS^t$, until they stabilize to some fixed-point $\mS^\star$. In the Bayes-optimal case, the orbits of $\mS^t$ can be equivalently formulated in terms of the overlap matrices $\mQ^t$, which have a direct and intuitive connection to the MSE that is achieved by AMP. 

In the Bayes-optimal AMP case the fixed points of SE were observed in many cases to correspond to the replica-symmetric predictions for the MMSE, suggesting that optimally tuned AMP is information-theoretically optimal in a wide variety of settings, with explicit algorithmic constructions possible in some settings \cite{Deshpande:2016aa, Deshpande:2014aa}. The study of SE fixed points is therefore a fundamental tool to understand the computational limits for a very wide class of estimation techniques that are either based on AMP recursions or can be approximated via AMP, such as various spectral methods \cite{Mondelli:2021aa, Mondelli:2021ab}. 

\subsection{SE fixed points and fundamental limits}
In this section, we highlight the connection between the fixed points of the Bayes-optimal AMP we describe and the critical points of the approximation formula of the mutual information for the MTP derived in \cite{reeves2020information}. With our multi-view spiked matrix model representation of the MTP, the approximation of the mutual information is given by the variational formula
\begin{align}
    \hat{\mathcal{I}}_n(\cT) = \min_{\mQ \in \cQ_n} \sup_{\mR \in \PSD^d} \!\left\{ I_n(R) + \frac{1}{4} \left\langle \cT\!\left( \frac{1}{n} \E[\bX^\top \!\bX] - \mQ \right), \left( \frac{1}{n} \E[\bX^\top \!\bX] - \mQ \right) \right\rangle - \frac{1}{2} \left\langle \mR , \left( \frac{1}{n} \E[\bX^\top \!\bX] - \mQ \right) \right \rangle \right\},
\end{align}
where 
\begin{align}
    \cQ_n \coloneqq \left\{ \mQ \in \PSD^d : \mQ \preceq \frac{1}{n} \E[\bX^\top \bX] \right\}, \quad I_n(R) = \frac{1}{n}I(\bX; \bH_{\mR}),
\end{align}
for $\bH_{\mR} \coloneqq \bX \mR^{1/2} + \bZ$, $\bZ \sim \Normal(0, \Id_d \otimes \Id_n) \indep \bX$ and $I(\cdot \; ; \; \cdot)$ denoting the mutual information. We note that the second term in the function to be optimized differs by a factor of $1/2$ due to the symmetry we are assuming in the MTP. In the fixed-$d$ regime, $\hat{\mathcal{I}}_n(\cT)$ was shown to converge pointwise to the limiting value of the mutual information for the MTP. 

An alternative formulation of the same variational formula expressed in terms of the ($1/n$-rescaled) relative entropy between the distributions of $\bH_{\mR}$ and $\bZ$, which we denote by $D_n(\mR)$. We have the formula
\begin{align}
    \hat{\mathcal{D}}_n(\cT) = \max_{\mQ \in \cQ_n} \inf_{\mR \in \PSD^d} \left\{ D_n(\mR) + \frac{1}{4} \left\langle \cT\!\left(\mQ \right), \mQ \right\rangle - \frac{1}{2} \left\langle \mR , \mQ \right \rangle \right\}; \quad D_n(\mR) = \frac{1}{n} \KL(P_{\bH_{\mR}} \Vert P_{\bZ}).
\end{align}
With bounded second moments, the multivariate I-MMSE relation \cite{Palomar:2006aa} implies that the gradients of $D_n(\mR)$ are defined everywhere for $\mR \in \PD^d$ and can be extended unambiguously to the boundary of $\PSD^d$ by taking limits. Furthermore,
\begin{align}
    \nabla_{\mR} D_n(\mR) = \frac{1}{2n} \E\!\left[ \E[\bX \mid \bH_{\mR}]^\top \E[\bX \mid \bH_{\mR}] \right] \eqqcolon \psi_n(\mR).
\end{align}

These gradients are related to the SE recursion in the MTP by noticing that the non-linear component of the state evolution self-map $\psi$ can be understood as the pointwise limit of the sequence of functions $\psi_n$. Thus, one can define a sequence of approximate SE maps that approximate SE. For $n \in \N$, we write 
\begin{align}
    \mQ_n^{t+1} = \psi_n(\cT(\mQ_n^t)),
\end{align}
for the $n$-th approximation to SE. Letting $\mQ_n^\star \in \PSD^d$ be a fixed point for the above recurrence relation, we have the property that the set of fixed points of the recursion is a superset of the maximizers in $\mQ$ of $\hat{\mathcal{D}}_n$.
\begin{prop} \label{prop:state_evolution_fixed_points_extremizers_variational_formula_finite_sample}
    For any $n \in \N$, it holds
    \begin{align}
        \argmax_{\mQ \in \cQ_n} \inf_{\mR \in \PSD^d} \left\{ D_n(\mR) + \frac{1}{4} \left\langle \cT\!\left(\mQ \right), \mQ \right\rangle - \frac{1}{2} \left\langle \mR , \mQ \right \rangle \right\} \subset \{ \mQ_n^\star \in \PSD^d : \mQ_n^\star = \psi_n(\cT(\mQ_n^\star)) \},
    \end{align}
    whenever the map $\mX \mapsto \langle \cT(\mX), \mX \rangle$ is convex in $\mX \in \Sym^d$.
\end{prop}
\begin{proof}
    Under the convexity condition in the assumptions, the approximation formula $\hat{\mathcal{D}}_n$ reduces \cite{reeves2020information} to the solution to the single-variable maximization problem 
    \begin{align}
        \hat{\mathcal{D}}_n(\cT) = \max_{\mQ \in \cQ_n} \left\{ D_n( \cT(\mQ) ) - \frac{1}{4} \langle \cT(\mQ), \mQ \rangle \right\}
    \end{align}
    which is the same as the one in \cite{Mayya:2019aa,Barbier:2020ab} up to a scaling factor due to model symmetry. Simply tanking gradients using the I-MMSE relationship provides that all maximizers of the above formula must satisfy
    \begin{align}
        \frac{1}{2} \cT(\psi_n(\cT(\mQ))) = \frac{1}{2} \cT(\mQ),
    \end{align}
    which directly provides the desired inclusion.
\end{proof}

While Proposition~\ref{prop:state_evolution_fixed_points_extremizers_variational_formula_finite_sample} is suggestive of a direct connection between the fixed points of SE and the fundamental limits of inference, there are several caveats that need to be addressed. For one, the inclusion relation guarantees that the fixed points of Bayes-optimal SE are critical points of for the approximation formula, but they may correspond to local minima, maxima or inflection points in $\max_{\mR \in \PSD^d} F_n(\mQ, \mR ; \cT)$. When SE gets stuck at a local minimum, statistical-to-computational gaps will be observed. We present an instance of this phenomenon in Section~\ref{sec:heteroskedastic_rank_one_model_application}.

Furthermore, state evolution only holds exactly in the asymptotic limit, and the appropriate object of comparison for the recurrence relation $\mQ^{t+1} = \psi(\cT(\mQ^t))$ is the limiting function $F(\mQ, \mR; \cT)$, defined as $F(\mQ, \mR; \cT) = D(\mR) + \langle \cT(\mQ) , \mQ \rangle / 4 - \langle \mR, \mQ \rangle / 2$, where $D(\mR)$ is the pointwise limit of $D_n(\mR)$, which is convex and continuously differentiable when assumption \ref{as:signal_weak_correlation_assumption} holds true. 

Even when $\mX \mapsto \langle \cT(\mX), \mX \rangle$ is convex, it is not in general true that the sequence of maximizers in $\mQ$ for each $F_n$ converges to a set of cluster points that are given by the maximizers of $F$. If that is the case, the arguments of Proposition~\ref{prop:state_evolution_fixed_points_extremizers_variational_formula_finite_sample} can be extended to the limit to provide an exact characterization of the relationship between the fixed points of state evolution and the information-theoretic limits of inference for the MTP. The conditions under which this passing to the limit is feasible and their relationship with the weak-correlation assumption \ref{as:signal_weak_correlation_assumption} are interesting questions of theoretical and practical relevance that we leave as research directions for future work.

We conclude with presenting a special case in which the limit of the approximation formulas $\hat{\mathcal{D}}_n(\cT)$ can itself be characterized in terms of the solution to a finite-dimensional variational problem, leading to the aforementioned correspondence between the fixed points of SE and the fundamental limits. 

In case the prior sequence $\bX \in \R^{n \times d}$ is a sample of i.i.d. draws from some underlying probability distribution $P_{X_\star}$ on $\R^d$ with bounded fourth moments, such that for $S \in \PSD^d$, $H_{\mS} \coloneqq \mS^{1/2} X_{\star} + Z$, the conditional expectation function $\E[ X_\star \mid H_{\mS} = y]$ is Lipschitz in $y \in \R^d$, it was shown \cite{reeves2020information} that 
\begin{align} \label{eq:iid_priors_fundamental_limits_approx_formula}
    &\lim_{n \rightarrow \infty} \hat{\mathcal{D}}_n(\cT) = \hat{\mathcal{D}}_\star(\cT) = \max_{ \mQ \in \cQ } \inf_{\mR \in \PSD^d} \left\{ D_\star(\mR) - \frac{1}{4} \left\langle \cT(\mQ), \mQ \right\rangle + \frac{1}{2} \left\langle \mR, \mQ \right\rangle \right\}; \\
    & D_\star(\mR) = \KL(P_{H_{\mS}} \Vert P_Z), \quad \cQ = \{ \mQ \in \PSD^d : \mQ \preceq \E[X_{\star} X_\star^\top] \}.
\end{align}
In this case, we can straightforwardly extend the results in Proposition~\ref{prop:state_evolution_fixed_points_extremizers_variational_formula_finite_sample} to the large-$n$ limit.
\begin{theorem} \label{th:state_evolution_fixed_points_extremizers_iid_priors}
    Assume the signal sequence $\bX \in \R^{n \times d}$ has i.i.d. rows $X_i \iid P_\star$ such that assumption \ref{as:signal_weak_correlation_assumption} holds. Furthermore, assume $\cT$ is such that the map $\mY \mapsto \langle \cT(\mY), \mY \rangle$ is convex for $\mY \in \Sym^d$. Then,
    \begin{align}
        \argmax_{\mQ \in \cQ} \inf_{\mR \in \PSD^d} \left\{ D_\star(\mR) - \frac{1}{4} \left\langle \cT(\mQ), \mQ \right\rangle + \frac{1}{2} \left\langle \mR, \mQ \right\rangle \right\} \subset \left\{ \mQ^\star \in \PSD^d : \mQ^\star = \psi(\cT(\mQ^\star)) \right\}.
    \end{align}
\end{theorem}
\begin{proof}
    The proof follows the same steps of Proposition~\ref{prop:state_evolution_fixed_points_extremizers_variational_formula_finite_sample} using the approximation formula \eqref{eq:iid_priors_fundamental_limits_approx_formula}.
\end{proof}

\subsection{Stability of SE and weak recovery}
Besides their location, we will pay particular attention to the stability of SE fixed points, as it relates to the possibility of achieving weak recovery to the signal of interest. Unstable fixed points usually arise due to some inherent symmetries of the model under study, and are in practice irrelevant for computational purposes, as for any finite sample size random fluctuations around those fixed points (e.g. introduced by a random initialization) are sufficient to escape their basin of attraction. A well-known example of the correspondence between stability and weak recovery occurs is the standard spiked Wigner model with zero-mean signal. In this case, the SNR level at which the zero fixed point for SE becomes is exactly the same as the one at which the BBP phase transition \cite{baik2005} occurs.

Going forward, we will assume without loss of generality that the signal of interest $\bX$ has been normalized so as to ensure that
\begin{align} \label{eq:signal_limiting_row_covariance_assumption}
    \lim_{n \rightarrow \infty} \frac{1}{n} (\E[\bX^\top\bX] - \E[\bX]^\top \E[\bX]) = \begin{bmatrix}
        \Id_p & 0 \\ 0 & 0
    \end{bmatrix} \eqqcolon \mP, \quad p \leq d.
\end{align}
Since it will play a key role in the analysis of weak recovery for AMP, we adopt the notation $\mOmega \coloneqq \lim_{n \rightarrow \infty} \E[\bX]^\top \E[\bX] / n$, and refer to $\mOmega$ as the \textit{mean overlap} or the \textit{baseline overlap}.
Intuitively, this represents the limiting value of the overlap that is achievable by simply guessing the prior mean of the signal, that is $\psi(0)$. 

A convenient elementary property of the overlap function $\psi$ is a translation rule that allows to express the overlap from a signal $\bX$ with arbitrary $\mOmega \in \PSD^d$ as an affine transformation of the overlap for the translated signal $\bX - \E[\bX]$. Letting $\psi(\; \cdot \; ; \mOmega)$ denote the overlap function for a non-centered signal and $\psi(\cdot)$ the overlap function for the centered version, translation-invariance of the MMSE matrix directly implies that, for all $\mS \in \PSD^d$,
\begin{align} \label{eq:overlap_translation_rule}
    \psi(\mS; \mOmega) = \psi(\mS) + \mOmega.
\end{align}

\begin{remark}
     When the signal is such that $\E[\bX] = 0$, or more precisely when $\mOmega = 0$, the MTP model has an inherent symmetry (and therefore sign-ambiguity) that makes the asymptotic performance of AMP trivially zero in the absence of side information. Indeed, any independent initialization $\bM^0$ for the AMP recursion will necessarily yield $\mQ^1 = 0$, meaning that, in the large-$n$ limit, the AMP iterates asymptotically behave as pure Gaussian noise and are therefore completely uninformative about $\bX$. In other words, state evolution always has a trivial fixed point at zero, but it may be unstable. If this is the case, for any fixed $n \in \N$, random fluctuations around zero overlap at initialization will cause SE to diverge from the fixed points, and the SE orbits may be bounded away from zero. 
\end{remark}

On the other hand, let us consider the case in which we have some side information about $\bX$, say for simplicity a Gaussian side-channel observation $\bH_\epsilon = \sqrt{\epsilon} \bX + \bW$, for $\epsilon > 0$ and $\bW \sim\Normal(0, \Id_d \otimes \Id_n)$ independent of $\bX$ and $\{\bY_k\}_{k \in [K]}$. With this, even in case $\mOmega = 0$ it is possible to resolve the sign-ambiguity and construct an initialization $\bM^0(\bH_\eps)$ that yields non-zero overlap in the first stage. If this is the case, though, at all iterations $t$ the Bayes-optimal denoisers $\eta_t$ should account for knowledge of this side information as well, and should technically be tracking the performance of a different AMP recursion, namely the one in which the signal follows the law the posterior distribution $P_{\bX \mid \bH_\epsilon}$ that is no longer zero-mean.

As we will see in Theorem~\ref{th:state_evolution_mean_information_achieves_nonzero_overlap}, the presence of mean information $\mOmega$ is equivalent to side information in the sense that the overlap function behaves exactly as a conditional overlap, with conditioning random variable $\bH_{\mOmega} \coloneqq \bX (\cT(\mOmega))^{1/2} + \bW$ \cite{guo2011}. This does not preclude, however, the possibility that the limiting behavior of the AMP iterates is that of a Gaussian noise-corrupted version of a lower-dimensional projection of $\mX$, and estimation is therefore limited to a low-rank projection of the signal. The structure of this subspace is related to the interaction between the mean information $\mOmega$ and the coupling structure of the MTP represented by $\cT$. A question of interest is thus the possibility of obtaining estimates of larger rank with an arbitrarily small amount of side information. 

In both the zero-mean overlap and the nonzero-mean overlap cases, the ability to improve the performance of AMP via an arbitrarily small injection of side information relates to the idea of state evolution being stuck at an unstable fixed point. If this is the case, adding a minimal amount of side information in the estimation problem may lead to improvements in the estimation performance even without said side information being available in further AMP iterations. Due to the relationship with known phase transitions in recovery regimes for many known problems (e.g. the BBP phase transition in the spiked Wigner model), we refer to the existence of such instabilities around SE fixed points as the possibility of weak recovery.

\subsection{MMSE and overlap for a Gaussian sequence signal}
We begin by providing an explicit derivation for the SE recursion when the input signal sequence $\bX$ is a Gaussian sequence with mean and (possibly singular) covariance constraints. For each $n \in \N$, let us assume that the vectorized signal $\gvec(\bX) \in \R^{nd}$ has a known mean and covariance, that we denote by $\mmu_n \in \R^{nd}$ and $\mPhi_n \in \PSD^{nd}$, respectively. From the reduction in \eqref{eq:overlap_translation_rule}, we will be focusing on the case $\mmu = 0$.

A well-known result is that, for Gaussian linear channels with prior constraints on mean and variance, a Gaussian prior is least favorable in terms of the MMSE, and the Bayes-optimal estimator is a linear function of the observations. Therefore, it is possible to establish lower bounds on the fixed-point overlaps achievable via Bayes-optimal AMP by studying the Gaussian prior case. Explicit evaluation of the SE trajectory is particularly feasible in case the prior covariances $\mPhi_n$ are taken to have some latent low-rank structure, in which case the formulas below become efficient to evaluate, allowing to compute explicitly some lower bounds on the mean-square error achieved by AMP. We note that, however, our general necessary and sufficient conditions for stability in Section~\ref{sec:necessary_sufficient_conditions_stability} hold irrespective of the distribution on the prior signal.

Letting $H_{\mS} \coloneqq (\mS^{1/2} \otimes \Id_n) X + Z$, for $\mS \in \PSD^d$ and $Z \sim \Normal(0, \Id_d \otimes \Id_n) \indep X\sim \Normal(0, \mPhi_n)$, we let $\mV \in \R^{nd \times q_n}$ be a full rank factorization  of $\mPhi_n = \mV \mV^\top$. We only consider the zero-mean case here as the MMSE matrix is translation invariant. The Gaussian optimal estimator $g(y; \mS)$ is given by
\begin{align}
    g(y; \mS) = \mV (\Id_{q_n} + \mV^\top (\mS \otimes \Id_{q_n})^{-1} \mV)^{-1} \mV^\top (\mS^{1/2} \otimes \Id_n) y,
\end{align}
leading to the MMSE function
\begin{align}
    M_n(\mS) = \frac{1}{n} \tr_n \{ \mV (\Id_{q_n} + \mV^\top (\mS \otimes \Id_n)\mV )^{-1} \mV^\top \}.
\end{align}
Under \ref{as:signal_weak_correlation_assumption}, the limiting MMSE $M(\mS)$ is well defined and, furthermore, the overlap function associated with a prior sequence $X \sim \Normal(0, \mPhi_n)$ can be expressed as
\begin{align}
    \psi(\mS) = \lim_{n \rightarrow \infty} \left( \frac{1}{n} \tr_n \{ \mV (\Id_{q_n} + \mV^\top (\mS \otimes \Id_n)\mV )^{-1} \mV^\top (\mS \otimes \Id_n)\mV \mV^\top \} \right).
\end{align}

A last property of the Gaussian MMSE is that, in neighborhoods of zero SNR, the MMSE function under any prior sequence with matched means and covariances satisfying \ref{as:signal_weak_correlation_assumption} has approximately the same first-order behavior. We have the following
\begin{prop} \label{prop:mmse_expansion_zero_is_tight}
    Let $X \sim P \in \R^{n\times d}$ be distributed as a random variable with mean $\mmu_n$ and covariance $\mPhi_n$ satisfying \ref{as:signal_weak_correlation_assumption}, and let $X' \sim \Normal(\mmu_n, \mPhi_n)$. Let $\psi(\mS)$ be the limiting overlap associated with $H_{\mS} \coloneqq (\mS^{1/2} \otimes \Id_n) X + Z$ and $\psi'(\mS)$ the Gaussian overlap associated with the observations $H'_{\mS} \coloneqq (\mS^{1/2} \otimes \Id_n) X' + Z$. Then, for $\epsilon > 0$ and $\mT \in \PSD^d$ such that $\|\mT\|_F=1$, 
    \begin{align}
        \| \psi(\eps \mT) - \psi'(\eps \mT) \|_{F} = o(\eps).
    \end{align}
\end{prop}
\begin{proof}
    From assumption \ref{as:signal_weak_correlation_assumption}, both limiting overlaps $\psi$ and $\psi'$ are Lipschitz and therefore admit a first-order expansion around zero. The result follows directly from the definition of the directional gradients $\nabla \psi$ by a first-order expansion at zero using the form of the gradients given in Proposition~\ref{prop:matrix_mmse_gradient_identity}.
\end{proof}

A direct consequence of Proposition~\ref{prop:mmse_expansion_zero_is_tight} is that for the zero mean overlap case the stability analysis for SE in the least-favorable Gaussian case is tight and therefore completely characterizes weak recovery for any arbitrary prior sequence such that the Bayes-optimal algorithm convergence and the limiting MMSE map is Lipschitz.

\subsection{Conditions for weak recovery} \label{sec:necessary_sufficient_conditions_stability}
We begin by observing that via a simple change of variable one can always express the state evolution in terms of the overlap function for a zero-mean version of the signal, denoted by $\psi$. By the translation rule \eqref{eq:overlap_translation_rule}, for a signal with arbitrary baseline overlap $\mOmega$ we can write state evolution as
\begin{align}
    \mQ^{t+1} &= \psi( \cT(\mQ^{t}) ; \mOmega) = \psi( \cT(\mQ^{t})) + \mOmega .
\end{align}
A first property that we verify is that state evolution is positive in the sense that it is order-preserving in the Loewner order, that is $\mX \preceq \mY$ implies that $\psi( \cT(\mX) + \cT(\mOmega) ) \preceq \psi( \cT(\mY) + \cT(\mOmega) )$. The fact that in the limit $\psi$ is order preserving follows from the fact that under \ref{as:signal_weak_correlation_assumption} the limit of $\E[\mPsi(\mS)]$ is a positive semi-definite matrix. This occurs since the covariance matrices are positive semi-definite, and so are their Kronecker squares, and the partial trace operation is a positive operator over the space of bounded operators \cite{Bhatia:2003aa}. 

Therefore, the function $\psi( \cT(\mQ^t) )$ will be positive if and only if $\cT$ is. $\cT$ is a linear operators on the space of square matrices and admits a representation as operators of the type  $T(\mX) = \sum_{i \in I} \mL_i \mX \mL_i^\top$ for some index set $I$. Objects of this form are a special case of what are known as \textit{Kraus representations} of completely positive operators. They have been extensively studied in the quantum information theory literature, and many of their special properties are known \cite{Watrous:2018aa}. We collect some useful elementary properties of such operators in the lemma below.

\begin{lemma} \label{lem:completely_positive_operator_properties}
    For for $K \leq d^2$, consider a linear operator $T: \R^{d \times d} \rightarrow \R^{d \times d}$ of the form 
    \begin{align}
        T(\mX) = \sum_{k =1}^K \mL_k \mX \mL_k^\top, \quad \{\mB_k \in \R^{d \times d}\}_{k \in [K]}.
    \end{align}
    The following are true.
    \begin{enumerate}[label=(\alph*)]
        \item $T$ is positive, in the sense that it maps elements of $\PSD^d$ into elements of $\PSD^d$.
        \item $T$ is order-preserving with respect to Loewner order, $\mX \preceq \mY \implies T(\mX) \preceq T(\mY)$.
        \item There exist orthonormal matrices $\{ \mV_i \in \R^{d \times d} \}_{i \in [d^2]}$ and non-negative scalars $\{ \theta_i \in [0, \infty) \}_{i \in [d^2]}$ such that we can write
        \begin{align} \label{eq:state_evolution_choi_kraus_decomosition}
            T(\mX) = \sum_{i=1}^{d^2} \theta_i \mV_i \mX \mV_i^\top,
        \end{align}
        and furthermore there are exactly $R \leq K$ non-zero values of $\theta_i$, where 
        \begin{align}
            R = \rank \left( \sum_{k \in [K]} \gvec(\mB_k)\gvec(\mB_k)^\top \right).
        \end{align}
        \item If the operator $T$ is self-adjoint, then it also admits an eigendecomposition
        \begin{align} \label{eq:state_evolution_expansion_eigendecomposition}
            \mT(\mX) = \sum_{i=1}^{d^2} \lambda_i \langle \mU_i, \mX \rangle \mU_i
        \end{align}
        for some $\{\lambda_i \in \R\}_{i \in [d^2]}$ and orthonormal matrices $\{ \mU_i \in \R^{d\times d}\}_{i \in d^2}$. Furthermore, the eigenbases can be chosen such that $d(d+1)/2$ elements are symmetric and $d(d-1)/2$ elements are skew-symmetric.
    \end{enumerate}
\end{lemma}

\begin{proof} We prove each claim separately.
    \begin{enumerate}[label=(\alph*)]
        \item For any $\mX \in \PSD^d$ consider a square root factorization $\mX = \mV\mV^\top$. Then, for each $\mv \in \R^d$, evaluating the inner product $\langle \mv, T(\mX) \mv \rangle$ yields
        \begin{align}
            \langle \mv, T(\mX) \mv \rangle = \sum_{k =1}^K \mv^\top (\mL_k \mV)(\mL_k \mV)^\top \mv = \sum_{k=1}^K \| (\mL_k \mV)^\top \mv \|^2 \geq 0,  
        \end{align}
        and therefore $T(\mX) \in \PSD^d$.
        \item This is a direct consequence of (a), since for $\mX \preceq \mY$ one has $T(\mY) - T(\mX) = T(\mY - \mX) \in \PSD^d$ and therefore $\mX \preceq \mY \implies T(\mX) \preceq T(\mY)$. 
        \item For any pair $\mX, \mY \in \R^{d \times d}$, we have
        \begin{align}
            \langle T(\mX), \mY \rangle &= \sum_{k=1}^K \langle \mL_k \mX \mL_k^\top , \mY \rangle \\
            &= \sum_{k=1}^K \langle \mL_k, \mY^\top \mL_k \mX \rangle \\
            &= \sum_{k=1}^K \langle \gvec(\mL_k), (\mX \otimes \mY) \gvec(\mL_k) \rangle \\
            &= \left\langle \sum_{k=1}^{k} \gvec(\mL_k) \gvec(\mL_k)^\top , \mX \otimes \mY \right\rangle.
        \end{align}
        The matrix $\sum_{k \in [K]} \gvec(\mL_k) \gvec(\mL_k)^\top$ is positive semi-definite, ad hence admits an eigendecomposition $\sum_{i \in [d^2]} \theta_i \mv_i\mv_i^\top$, for an orthonormal basis $\{\mv_i \}_{i \in [d^2]}$ of $\R^{d^2}$. From positive semi-definiteness all $\theta_i$ coefficients are non-negative and in particular there are $\rank( \sum_{k \in [K]} \gvec(\mL_k) \gvec(\mL_k)^\top )$ strictly positive eigenvalues. Substituting back and letting $\gvec(\mV_i) = \mv_i$ yields the desired decomposition for $T$.
        \item Existence of an eigendecomposition follows immediately from the spectral theorem. Furthermore, from the definition of $T$, we notice that $T$ commutes with matrix transposition, i.e. $T(\mX^\top) = T(\mX)^\top$ for all $\mX \in \R^{d \times d}$. Thus, if $\mU_i$ is an eigenvector of $T$, we have $T(\mU_i^\top) = T(\mU_i)^\top = \lambda_i \mU_i^\top$, so $\mU_i^\top$ is also an eigenvector for $T$ with eigenvalue $\lambda_i$. We can therefore modify the eigenbasis as follows. For each $i\leq j \in [d^2]$, if $\mU_i = \mU_i^\top$ leave $\mU_i$ unchanged and else, for all pairs $i,j$ such that $\mU_i^\top = \mU_j$, construct two new orthogonal matrices $\mU_i' = (\mU_i + \mU_j)/2\sqrt{2}$ and $ \mU_j' = (\mU_i - \mU_j)/ \sqrt{2} $. This constitutes an alternative eigenbasis with $(n+1)n/2$ symmetric and $n(n-1)/2$ skew-symmetric elements. 
    \end{enumerate}
\end{proof}

Lemma~\ref{lem:completely_positive_operator_properties} confirms that Bayes-optimal state evolution is an order-preserving recurrence relation. An immediate consequence of this is that in the presence of mean information the AMP recursion converges to a non-zero overlap fixed point that is larger (in the Loewner ordering) than the prior mean estimate as long as the prior sequence is non-degenerate.
\begin{theorem} \label{th:state_evolution_mean_information_achieves_nonzero_overlap}
    Whenever the signal sequence is such that $0 \prec \mOmega$, the Bayes-optimal AMP algorithm is associated with a fixed-point centered overlap $\mQ^\star \succeq \mOmega$, with the relation holding strictly if and only if the prior sequence is non-degenerate.
\end{theorem}
\begin{proof}
    Let us initialize the AMP algorithm with the prior mean estimator $\bM^0=\E[\bX]$, which is the Bayes-optimal estimator for a linear channel with zero SNR. This corresponds to the setting $\mQ^1 = \mOmega$. This corresponds to $\psi(\cT(\mOmega)) + \mOmega = \mQ^2 \succeq \mQ^1 = \mOmega  $ , with equality possible only when the prior sequence is deterministic. From the order-preserving property of SE, it follows that $\mQ^2 \preceq \mQ^3$ and so on. This implies that the iterates $\{\mQ^t\}_{t \in \N}$ are monotonically increasing in the Loewner order, and since SE is bounded and continuous it must be that $\{\mQ^t\}_{t \in \N}$ must converge to a fixed point $\mQ^\star \succeq \mOmega$.
\end{proof}

We now move on to studying the stability properties of SE around its fixed points $\mQ^\star$. In particular, we are interested in a directional notion of stability, corresponding to the idea of adding some arbitrarily small side information to the model to ``nudge" the overlap above its fixed point. If the overlap fixed point remains unchanged under such perturbation, we say the recursion is stable, and is unstable otherwise. Formally, for a SE fixed point $\mQ^\star$, we say SE is \textit{(locally) stable} whenever there exists some $\delta > 0$ such that, for all elements in the set 
\begin{align}
    B_\delta(\mQ^\star) = \{ \mX \in \PSD^d : \mQ^\star \preceq \mX \preceq (\mQ^{\star} + \delta \mY), \ \mY\in \PSD^d, \|\mY\|_F=1  \},
\end{align}
the SE map $\mQ^{t+1} = \psi( \cT(\mQ^t )) + \mOmega$ is such that 
\begin{align}
    \sup_{\mX \in B_\delta(\mQ^\star)} \limsup_{t \rightarrow \infty}  \| \mQ^{\star} - \mQ^t(\mX)\|_F = 0,
\end{align}
where $\mQ^t(\mX)$ denotes the $t$-th epoch of SE initialized at $\mX$.

This essentially reduces to understanding the behavior of the gradient given by the linear map $\nabla \psi(\cT(\mQ^\star))$ near the fixed point itself, for some arbitrarily small perturbation $\mY \in \PSD^d$. 

\begin{theorem} \label{th:state_evolution_stability_theorem}
    Consider the Bayes-optimal SE recursion \eqref{eq:bayes_optimal_state_evolution_recusion_overlap_argument} and one of its fixed points $\mQ^\star$. Then $\mQ^\star$ is a stable fixed point of state evolution if 
    \begin{align} \label{eq:state_evolution_fixed_point_stability_condition}
        \max_{\{ \mY \in \PSD^d: \|\mY\|_F \leq 1\}} \| \nabla \psi(\mQ^\star) (\cT(\mY)) \|_F < 1, 
    \end{align}
    and conversely $\mQ^\star$ is unstable whenever 
    \begin{align}
        \max_{\{ \mY \in \PSD^d: \|\mY\|_F \leq 1\}} \| \nabla \psi(\mQ^\star) (\cT(\mY)) \|_F > 1. 
    \end{align}
\end{theorem}
\begin{proof}
    To simplify notation, we will denote the linear map $\nabla \psi(\mQ^\star) (\cT(\;\cdot\;))$ as $\cA_{\mQ^\star}(\;\cdot\;)$, for some SE fixed point $\mQ^\star$. We begin with the \textit{if} direction. For $\delta>0$ and $\mX \in B_\delta(\mQ^\star)$, a first order expansion of SE at $\mY \in A_\delta(\mQ^\star)$ a first-order expansion of $\mQ^1(\mY)$ about $\mQ^\star$ gives that
    \begin{align}
        \| \mQ^\star - \mQ^1(\mY) \|_F = \| \cA_{\mQ^\star}(\eps \mX) + o(\eps) \|_F,
    \end{align}
    for $0 < \eps < \delta$ and $\mX \in \PSD^d$ such that $\|\mX\|_F$ and $\mY = \mQ^\star + \eps \mX$. Then, if \eqref{eq:state_evolution_fixed_point_stability_condition} holds, there must exist some $\delta > 0$ such that for all $0<\eps<\delta$ and $\mY \in B_\delta(\mQ^\star)$ one has
    \begin{align}
        \| \mQ^\star - \mQ^1(\mY) \|_F < \| \mQ^\star - \mY \|_F,
    \end{align}
    hence SE is locally contractive and $\lim_{t \rightarrow \infty} \mQ^t(\mY) = \mQ^\star$, and the fixed point is stable.
    For the converse direction, let us assume that there exists some unit vector $\mY \in \PSD^d$ such that $\| \cA_{\mQ^\star}(\mY) \| > 1$. From application of the chain rule, we have that
    \begin{align}
        \nabla_{\mQ^\star}\mQ^t(\mQ^\star) = \cA_{\mQ^\star}^t,
    \end{align}
    where $\cA_{\mQ^\star}^t$ denotes the $t$-fold composition of $\cA_{\mQ^\star}$ with itself. Let us consider $\cA_{\mQ^\star}^2$. From the self-adjointness established in Proposition~\ref{prop:overlap_gradient_is_self_adjoint}, we have
    \begin{align}
        \| \cA_{\mQ^\star}^2 (\mY) \|_F &= \sup_{\mV \in \PSD^d \setminus \{0\}} \frac{\langle \cA_{\mQ^\star}^2 (\mY), \mV \rangle }{ \|\mV\|_F } \\
        &\geq \frac{\langle \cA_{\mQ^\star}^2 (\mY), \mY \rangle }{ \|\mY\|_F } \\
        &= \frac{\langle \cA_{\mQ^\star} (\mY), \cA_{\mQ^\star} (\mY) \rangle }{ \|\mY\|_F } \\
        &> \| \cA_{\mQ^\star} (\mY) \|_F.
    \end{align}
    Reasoning inductively yields that, for all $t \in \N$, $\| \cA_{\mQ^\star}^t (t\mY) \|_F > \| \cA_{\mQ^\star}^{t-1} (t\mY) \|_F > \| t\mY \|$ for all $t > 0$. Then, there must exist some $\delta > 0$ such that, for all $0 < \eps < \delta$, we have that
    \begin{align}
        \| \eps \mY \|_F < \limsup_{t \rightarrow \infty} \| \mQ^\star - \mQ^t(\mQ^\star + \eps \mY) \|_F,
    \end{align}
    and the state evolution orbit is bounded away from the fixed point $\mQ^\star$ with an arbitrarily small perturbation $\eps\mY$, showing that the fixed point is unstable.
\end{proof}

An instance in which the stability principle illustrated in Theorem~\ref{th:state_evolution_stability_theorem} becomes particularly tractable, in the sense that it does not require to evaluate high-dimensional conditional covariance matrices to obtain the structure of the maps $\cA_{\mQ^\star}$ is that of no mean-overlap information, i.e. $\mOmega=0$. In such case, it is immediate that $\mQ^\star = 0$ is the no-side-information fixed point of state evolution. When this is the case, under \ref{as:signal_weak_correlation_assumption} the form of the map $\cA_0( \;\cdot\; )$ can easily be seen to be given by
\begin{align}
    \cA_0(\mX) = \mP \cT(\mX) \mP; \quad \mX \in \PSD^d.
\end{align}
With our normalization of $\mP$, this simply involves studying the restricted singular values of a reduced version of $\cT$ (which we call $\Tilde{(\cT)}$) that tracks exclusively the top-left block of the coupling matrices:
\begin{align} \label{eq:zero_mean_reduced_state_evolution_overlap_gradient}
    \Tilde{\cT}(\mX) = \sum_{k=1}^K \mL_k \mX \mL_k; \quad \mX \in \PSD^d, \ \mL_k \coloneqq \begin{bmatrix}
        \Id_p & 0
    \end{bmatrix} \mLambda_k \begin{bmatrix}
        \Id_p \\ 0
    \end{bmatrix}.
\end{align}
Furthermore, from Proposition~\ref{prop:mmse_expansion_zero_is_tight}, this weak recovery condition is universal among the class of prior sequences satisfying \ref{as:signal_weak_correlation_assumption}. We summarize this observation in the following corollary.
\begin{cor} \label{cor:state_evolution_stability_zero_mean}
    Consider a sequence of signal matrices $\mX$ satisfying \ref{as:signal_weak_correlation_assumption} such that $\mOmega=0$. Then, the zero matrix is a fixed point of state evolution and it is stable whenever
    \begin{align}
         \max_{\{ \mY \in \PSD^d: \|\mY\|_F \leq 1\}} \| \Tilde{\cT}(\mY) \|_F < 1,
    \end{align}
    and it is otherwise unstable when 
    \begin{align}
         \max_{\{ \mY \in \PSD^d: \|\mY\|_F \leq 1\}} \| \Tilde{\cT}(\mY) \|_F > 1,
    \end{align}
    where $\tilde{\cT}$ is defined as is \eqref{eq:zero_mean_reduced_state_evolution_overlap_gradient}.
\end{cor}
\begin{proof}
    This is a follows immediately from Theorem~\ref{th:state_evolution_stability_theorem} and Proposition~\ref{prop:mmse_expansion_zero_is_tight}.
\end{proof}

By Lemma~\ref{lem:completely_positive_operator_properties}, all linear operators $\cA_{\mQ^\star}$ (which can be verified to be completely positive) can be expressed in terms of their spectral decomposition with symmetric-skew-symmetric eigenvectors \eqref{eq:state_evolution_expansion_eigendecomposition}. Then, all symmetric matrices can be represented exclusively in terms of the symmetric eigenvectors, and a further simple sufficient condition for establishing the global stability of fixed points in SE involves checking whether the magnitude of the leading eigenvalue in the eigendecomposition of $\tilde{\cT}$ is smaller than one. 

\begin{prop} \label{prop:fixed_point_stability_eigenvalue_sufficient_condition}
    For a SE recursion satisfying \ref{as:signal_weak_correlation_assumption}, all fixed points $\mQ^\star$ are stable if, for a symmetric-skew-symmetric eigendecomposition \eqref{eq:state_evolution_expansion_eigendecomposition} of $\Tilde{\cT}$,
    \begin{align}
        \Tilde{\cT}(\mX) = \sum_{i=1}^{p^2} \lambda_i \langle \mU_i, \mX \rangle \mU_i,
    \end{align}
    one has 
    \begin{align}
        \max_{\{i : \mU_i \in \Sym^p\}} |\lambda_i| < 1.
    \end{align}
\end{prop}
\begin{proof}
    Since the expected conditional covariance of a random variable is upper bounded by its unconditional covariance, it follows that for all $\mQ^\star, \mX \in \PSD^d$ we have that $\cA_{\mQ^\star}(\mX) \preceq \cA_{0}(\mX)$. Thus, it is sufficient to consider $\cA_0$ to provide an upper bound for the restricted singular values of $\cA_{\mQ^\star}$ for any fixed point $\mQ^\star$.
    From the arguments above, $\cA_0 = \Tilde{\cT}$, and we have
    \begin{align}
        \max_{\{ \mY \in \PSD^p: \|\mY\|_F \leq 1\}} \| \tilde{\cT}(\mY) \|_F &= \max_{\{ \mY \in \PSD^p: \|\mY\|_F \leq 1\}} \left\| \sum_{i=1}^{p^2} \lambda_i \langle \mU_i, \mY \rangle \mU_i \right\|_F \\
        &= \max_{\{ \mY \in \PSD^p: \|\mY\|_F \leq 1\}} \left( \sum_{i: \mU_i \in \Sym^p} \lambda_i^2 \langle \mU_i, \mY \rangle^2 \right)^{1/2} \\
        & \leq \max_{\{ \mY \in \Sym^p: \|\mY\|_F \leq 1\}}  \left( \sum_{i: \mU_i \in \Sym^p} \lambda_i^2 \langle \mU_i, \mY \rangle^2 \right)^{1/2} \\
        &= \max_{\{i : \mU_i \in \Sym^p\}} |\lambda_i| .
    \end{align}
    Thus, by Theorem~\ref{th:state_evolution_stability_theorem}, the condition $\max_{\{i : \mU_i \in \Sym^p\}} |\lambda_i| < 1$ is sufficient to rule out instability for any fixed point $\mQ^\star$. 
\end{proof}

\section{Rank-one heteroskedastic spiked matrix model} \label{sec:heteroskedastic_rank_one_model_application}
We demonstrate an application of our AMP algorithm and our SE stability analysis by considering a special case of the MTP model, namely a rank-one matrix factorization problem observed in heteroskedastic Gaussian noise. This model constitutes a highly flexible generalization of the widely studied spiked Wigner model in Gaussian noise and is sometimes known as the multi-species model.  Work on the fundamental limits of inference for such models has appeared in the last years using a different approaches and varying levels of generality \cite{behne2022fundamental, guionnet2022low}, and an AMP algorithm specifically tailored for this model class has recently appeared \cite{pak2023optimal} for separable denoisers. We will use the model specification of \cite{behne2022fundamental}, and demonstrate an application of our Bayes-optimal AMP algorithm to study the SE vis-à-vis the fundamental limits for the model.

\subsection{Model Setup and Bayes-optimal AMP}
 For $n \in \N$ and $d \in \N$, we specify a partition $\{J_1, \dotsc, J_d\}$ of $[n]$ with cardinalities $|J_j| = n_j$, $j \in [D]$ satisfying $n / n_j \rightarrow \beta_j \in [0,1]$ as $n \rightarrow \infty$. For $j \leq \ell \in [K]$, we let $\lambda_{j\ell} \in \R$, we let $\bone_{j\ell} \in \R^{n_j \times n_\ell}$ denote a $n_j \times n_\ell$ matrix of ones and we define the SNR profile matrix
\begin{align}
    \mDelta \coloneqq \begin{bmatrix}
        \bone_{11} \lambda_{11} & \bone_{12}\lambda_{12} & \hdots & \bone_{1d} \lambda_{1d} \\
        \bone_{21} \lambda_{12} & \bone_{22}\lambda_{22} & \hdots & \bone_{2d} \lambda_{2d} \\
        \vdots & \vdots & \ddots & \vdots \\
        \bone_{d1} \lambda_{1d} & \bone_{d2}\lambda_{2d} & \hdots & \bone_{dd} \lambda_{dd}
    \end{bmatrix} \quad \in \R^{n \times n}.
\end{align}
We take a signal vector $X \in \R^n$ and define the observation model as
\begin{align} \label{eq:hetero_rank_one_spiked_model}
    \bY = \frac{1}{n} (XX^\top) \circ \mDelta + \frac{1}{\sqrt{n}}\bG \quad \in \R^{n \times n},
\end{align}
where $\bG \sim \GOE(n)$ and $\circ$ denotes the element-wise (Hadamard) product of matrices. This model thus corresponds to a version of the widely studied symmetric rank-one matrix factorization problem with a known profiled noise variance, with the exception of allowing some blocks to be unobserved when $\lambda_{j\ell} = 0$. Partitioning the spike $X$ into blocks $\{X_j \in \R^{n_j}\}_{j \in [d]}$ corresponding to the block structure in the SNR profile, \eqref{eq:hetero_rank_one_spiked_model} can be seen as a special instance of the MTP with $K=1$ and a rank-$d$ signal $\bX \in \R^{n \times d}$ with a block diagonal structure
\begin{align} \label{eq:hetero_rankone_signal_struct}
    \bX \coloneqq \diag(X_1, \dotsc, X_d) = \begin{bmatrix}
        X_1 & 0 & \hdots & 0 \\
        0 & X_2 & \hdots & 0 \\
        \vdots & \vdots & \ddots & \vdots \\
        0 & 0 & \hdots & X_d
    \end{bmatrix} \quad \in \R^{n \times d},
\end{align}
with symmetric coupling matrix $\mLambda = (\lambda_{j\ell})_{j, \ell \in [d]}$, $\lambda_{j\ell} = \lambda_{\ell j}$, yielding
\begin{align} \label{eq:hetero_rankone_diag_representation}
    \bY = \frac{1}{n} \bX \mLambda \bX^\top + \frac{1}{\sqrt{n}}\bG.
\end{align}
Special configurations of the coupling matrix $\mLambda$ recover a variety of models that have been analyzed in their own right, such as the widely known symmetric and asymmetric rank-one spiked matrix models, generalized spiked covariance models \cite{Bai:2012aa}, spatially coupled models \cite{Barbier:2018aa} and Gaussian approximations to the stochastic blockmodel \cite{Abbe:2017aa} and its contextual version \cite{DSMM2018contextual}, among others.

Henceforth, we will be assuming that the matrix $\mLambda$ as well as the limiting group sizes $\mbeta \coloneqq (\beta_1, \dotsc, \beta_d)$ are known, and we have access to prior information about the structure of the signal within each block $X_j$. Coherently with the standard assumptions on these classes of symmetric models, we will be assuming that each signal block is constituted of i.i.d. draws from some probability distribution $P_j$, $j \in [d]$, with zero mean and bounded fourth moments. Additionally, we normalize each $P_j$ to have unit second moment and assume that each $P_j$ is such that the posterior expectation of $X_j \sim P_j$ given observations $\mS^{1/2}X_j + Z$, $Z \sim \Normal(0,1) \indep X_j$ is Lipschitz.

In terms of the implications of these assumptions, we note that the i.i.d. draw assumption implies that the weak correlation condition \ref{as:signal_weak_correlation_assumption} necessary to have a Lipschitz state evolution recursion is satisfied, and furthermore the model is zero-symmetric, and hence the ensuing state evolution has a trivial fixed point at zero in absence of an informative initialization to break symmetry.

For our AMP algorithm, we will assume an initialization $\bM^0$ for which it holds
\begin{align}
    \plim_{n \rightarrow \infty} \frac{1}{n} \bX^\top \bM^0 = \plim_{n \rightarrow \infty} \frac{1}{n} (\bM^0)^\top \bM^0 = \mQ^1 \succ 0.
\end{align}
This essentially amounts being provided as side-information Bayes-optimal estimate for some independent side information, which is subsequently discarded during the AMP iterations. The choice for this type of initialization is motivated by the goal of performing a stability analysis for the Bayes-optimal state evolution given a non-trivial initialization.

From the i.i.d. assumption, it is clear that the Bayes-optimal denoiser choice for AMP is the separable conditional expectation adapted to the block structure of $\bX$. We show this for the first step of the AMP recursion, and for $t \geq 2$ a similar argument follows by using Theorem~\ref{th:amp_mtp_symmetric_main_result}. At the first AMP step, the limiting distribution for $\bX^1$ is given (before reweighting by $\mLambda$) by 
\begin{align}
    \bH^1 = \mX \mLambda \mQ^1 + \bZ^1; \quad \bZ^1 \sim \Normal(0, \mQ^1 \otimes \Id_n) \indep \bX.
\end{align}
Using the same arguments as the ones in Section~\ref{sec:bayes_optimal_reweighting}, it follows that the Bayes-optimal estimator $\eta_1 : \R^{n \times d} \rightarrow \R^{n \times d}$ given $\bH^1$ is obtained entry-wise as
\begin{align}
    [\eta_1(y)]_{ij} = \begin{cases} \E\!\left[ X^\star_j \mid \langle H_i,  \mLambda \me_j \rangle = \langle y \me_i, \mLambda \me_j \rangle \right] & \text{if } j: i \in J_j, \ \ X^\star_j \sim P_j \\ 
    0 & \text{otherwise}\end{cases}.
\end{align}
where $\me_j$ denotes the $j$-th basis vector of $\R^d$. Equivalently, we can choose to reweight the iterate $\mX^1$ by $\mLambda$ and act conditionally on the reweighted limiting random variable (which by an abuse of notation we still label by $\bH^1$)
\begin{align}
    \bH^1 = \mX \cT(\mQ^1) + \bZ^1; \quad \bZ^1 \sim \Normal(0, \cT(\mQ^1) \otimes \Id_n) \indep \bX,
\end{align}
and the Bayes-optimal denoiser can be seen as acting separably on the block-diagonal elements of the AMP iterate, i.e. $\eta_1(y)$ is defined element-wise as
\begin{align}
    [\eta_1(y)]_{ij} = \begin{cases}
        \E[ X^\star_j \mid H_{ij}^1 = y_{ij} ] & \text{if } i \in J_j, \ X^\star \sim P_j \\
        0 & \text{otherwise}
    \end{cases}.
\end{align}
Applying the same reasoning recursively, we can describe the Bayes optimal AMP recursion $\{ \bX^t \mid \bM^0, \eta_t \}$.

\subsection{State evolution analysis}
With the Bayes-optimal algorithm for the heteroskedastic rank-one model specified, we conduct a stability analysis of SE in the spirit of Section~\ref{sec:computational_limits}. Since many of the quantities simplify quite considerably in this special case, we opt for a direct derivation of the overlap maps and related gradients instead of the abstract operator formulation we used in the general case. From Corollary~\ref{cor:amp_mtp_symmetric_wasserstein_convergence_result}, the SE recursion associated with the Bayes-optimal algorithm $\{ \bX^t \mid \bM^0, \eta_t \}$ can be expressed in terms of the overlap matrices $\mQ^t$. For some overlap level $\mQ^t$, we define the SNR vector $\ms^t \coloneqq (s_1^t, \dotsc, s_d^t) \in \R_+^{d}$ as $s_j^t = [\cT(\mQ^t)]_{jj}$. With Bayes-optimal denoisers, then, the overlap $\mQ^{t+1}$, $t \in \N$, can be obtained as
\begin{align}
    \mQ^{t+1} = \begin{bmatrix}
        \beta_1 \E[ \E[X_1^\star \mid \sqrt{s_1^t} X_1^\star + Z]^2  ] & 0 & \hdots & 0 \\
        0 & \beta_2 \E[ \E[X_2^\star \mid \sqrt{s_2^t} X_2^\star + Z]^2  ] & \hdots & 0 \\
        \vdots & \vdots & \ddots & \vdots \\
        0 & 0 & \hdots & \beta_d \E[ \E[X_d^\star \mid \sqrt{s_d^t} X_d^\star + Z]^2  ]
    \end{bmatrix}.
\end{align}
We see that state evolution for the heteroskedatic rank-one model is diagonal (in absence of dependencies between the blocks), and we can reduce SE to a vector-to-vector map. We define \textit{overlap vectors} $\mq^t \coloneqq (q_1^t, \dotsc q_d^t)$ as the vector of diagonal elements in $\mQ^t$. Since $\ms^t$ is composed of the diagonal elements of $\mLambda \diag(\mq^t) \mLambda$, it holds  $ \ms^t = \mLambda^{\circ 2} \mq^t$, and we can write state evolution as 
\begin{align}
    \mq^{t+1} = \psi( \mLambda^{\circ 2} \mq^t ); \quad \psi(\ms) \coloneqq \begin{pmatrix}
        \beta_1 \E[ \E[X_1^\star \mid \sqrt{s_1} X_1^\star + Z]^2  ] \\
        \vdots \\
        \beta_d \E[ \E[X_d^\star \mid \sqrt{s_d} X_1^\star + Z]^2  ].
    \end{pmatrix}
\end{align}

From the zero-mean assumption on the priors $P_j$, we know that $\mq^\star = 0$ is a fixed point for state evolution. We use the result in Corollary~\ref{cor:state_evolution_stability_zero_mean}, then, to provide necessary and sufficient condition for recovery in the heteroskedastic model. In a right-neighborhood of zero, the gradient of the overlap function $\nabla_{\ms}\psi(0)$ is given by the diagonal matrix $\diag(\mbeta)$. Thus, letting $\mT \coloneqq \diag(\mbeta) \mLambda^{\circ 2}$ be the gradient of SE map evaluated at zero, state evolution is unstable and weak recovery is achievable whenever
\begin{align}
    \max_{\{ \mx \in \R_+^d : \|\mx\| \leq 1 \}} \| \mT \mx \| > 1,
\end{align}
and is otherwise impossible when
\begin{align}
    \max_{\{ \mx \in \R_+^d : \|\mx\| \leq 1 \}} \| \mT \mx \| < 1.
\end{align}
In this setting, a sufficient condition to ignore the positivity constraint in the optimization set is that the matrix $\mT$ has a positive leading eigenvalue. An easy-to-verify condition for this to occur is that the matrix $\mT$ be irreducible. By the Perron-Frobenius theorem \cite[Chapter 13]{Gantmakher:2000aa}, then, the leading right eigenvector of $\mT$ is positive, and we can simply study the magnitude of $\| \mT \|_{\op}$. Furthermore, it can be shown that under such conditions SE is monotonically increasing to a fixed point $\mq^\star \succ 0$. In this estimation context, irreducibility is a very natural assumption, as it corresponds to the notion that no subset of the $d$ groups in the signal is isolated from the other, as if that were the case the two models would simply decouple and could be studied individually, as long as there is no statistical dependence between the distribution the signals are drawn from.

\subsection{SE fixed points and fundamental limits}
To conclude the presentation of the properties of Bayes-optimal SE, we compare the fixed-point overlap to the critical points that appear in the characterization of the fundamental limits for the heteroskedastic rank-one model in \cite{behne2022fundamental}. Specializing the approximation formula provided in \cite{reeves2020information}, they provide an asymptotic lower bound for the block-level MMSE 
\begin{align}
    \MMSE_j(\mLambda, \mbeta) = \frac{1}{n_j} \E\!\left[ \| X_j - \E[X_j \mid \bY]\|_F^2 \right]
\end{align}
in terms of the extremizers of a variational formula, namely
\begin{align} \label{eq:heteroskedastic_rank_one_fundamental_limits_lower_bound}
    \liminf_{n \rightarrow \infty} \MMSE_j(\mLambda, \mbeta) \geq 1 - \frac{q_j^\star}{\beta_j},
\end{align}
where $q_j^\star$ is the $j$-th coordinate of the global maximizer $\mq^\star$ of
\begin{align} \label{eq:heteroskedastic_rank_one_model_mmse_variational_formula}
    \max_{\mq \in [0, \mbeta]} \inf_{\ms \succeq 0} \left\{ \langle \mbeta, D(\ms) \rangle + \frac{1}{4} \langle \mq, \mLambda^{\circ 2}\mq \rangle - \frac{1}{2} \langle \ms, \mq \rangle \right\}; \qquad D(\ms) \coloneqq \begin{pmatrix} \KL( \sqrt{s_1}X_1^\star + Z \; \| \; Z ) \\ \vdots \\  \KL( \sqrt{s_d}X_d^\star + Z \; \| \; Z )\end{pmatrix},
\end{align}
for $X_j^\star \sim P_j \indep Z \sim \Normal(0,1)$. The critical points of the variational formula and SE are related via the I-MMSE relationship \cite{Guo:2005aa}, which gives that
\begin{align}
    \frac{\partial}{\partial s_j} D(\ms) = \frac{1}{2} \beta_j^{-1} [\psi(\ms)]_j.
\end{align}
With this, we can verify that all critical points of \eqref{eq:heteroskedastic_rank_one_model_mmse_variational_formula} are exactly (up to renormalization by $\mbeta$) the fixed points of the SE recursion, providing the inclusion relation
\begin{align}
    \arg \max_{\mq \in [0,\mbeta]} \inf_{\ms \succeq 0} \left\{ \langle \mbeta, D(\ms) \rangle + \frac{1}{4} \langle \mq, \mLambda^{\circ2}\mq \rangle - \frac{1}{2} \langle \ms, \mq \rangle \right\} \subset \left\{ \mv : \mbeta \circ \mv = \psi(\mT (\mbeta \circ \mv)) \right\}.
\end{align}
Another remarkable connection is the nature of the lower bound \eqref{eq:heteroskedastic_rank_one_fundamental_limits_lower_bound} and the instability in the zero fixed point in SE. In fact, we only have a lower bound in the fundamental limits due to the presence of a phase transition at exactly zero side information. From symmetry, the MMSE of the model is trivial when side information is exactly zero, but as soon as an infinitesimal bias is introduced in the model (in the form of side information) the limiting MMSE jumps to the lower bound, just like for an arbitrarily small perturbation about zero state evolution may move to a strictly non-zero value. 

Despite these similarities, however, there exists a statistical-to-computational gap between the theoretical MMSE lower that can be achieved and the SNR regimes in which state evolution attains orbits that are bounded away from zero. Thus, while there is a strong qualitative analogy between the results of SE analysis and the fundamental limits results, there exist some prior signals for which the SNR regimes in which AMP algorithms are suboptimal in terms of mean-square error, and their performance is dominated by some other, possibly infeasible, estimation procedure. We will provide a numerical example of this type of statistical-to-computational gaps in the next section.

\subsection{Numerical Experiments}
We conclude this application with a concrete example of the above analysis, in which we study a model in which $d = 2$, $\mbeta = (0.6, 0.4)$ and the SNR (or inverse-variance) profile is parametrized by a single scalar value $s>0$ that multiplies a fixed irreducible structure $\mXi$,
\begin{align}
    \mLambda^{\circ 2} = c \cdot \mXi; \qquad \mXi = \begin{bmatrix}
        0.7 & 0.3 \\ 0.3 & 0.7
    \end{bmatrix}.
\end{align}
Entries in $X_1$ are i.i.d. samples from a $\mathsf{Unif}(\{\pm 1\})$ distribution, while entries of $X_2$ are i.i.d. samples from a Bernoulli-Gaussian distribution with sparsity parameter $\epsilon \in (0,1]$, $\mathsf{BG}(\epsilon)$ i.e.
\begin{align}
    X \sim \mathsf{BG}(\epsilon) \iff X=BN; \qquad B \sim \mathsf{Be}(\epsilon) \indep N \sim \Normal(0, \epsilon^{-1}),
\end{align}
and the case $\epsilon = 1$ recovers the i.i.d. standard Gaussian case. 

As we will see, this simple setup is sufficient to illustrate the full complexity of behaviors that we described in the previous section. We will be comparing the information-theoretical lower bound \eqref{eq:heteroskedastic_rank_one_fundamental_limits_lower_bound} on $\MMSE_l$ with the MSE implied by the perturbed SE fixed point. We will do this for $l=1,2$ as a function of the choice of scaling $c$ and for varying sparsity levels $\epsilon \in \{0.05, 0.1, 0.5, 1\}$. In particular, we track the block-MSEs as a function of $\|\mT_c \|_{\mathrm{op}} \coloneqq c \cdot \|\diag(\mbeta) \mXi\|_{\mathrm{op}}$, which we can interpret as the implied SNR for the model. The results are plotted in Figure~\ref{fig:hetero_2group_fundlimits}.

\begin{figure}
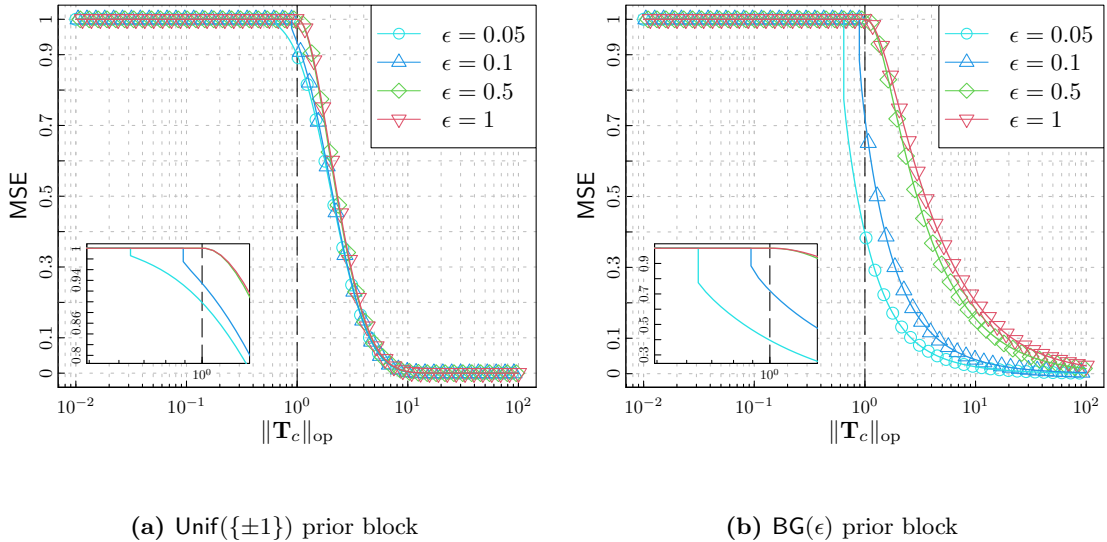

    \centering
    \begin{subfigure}{.45\textwidth}
    \centering
    \include{2group_fundlims_rad}
    \caption{$\mathsf{Unif}(\{\pm 1\})$ prior block}
    \end{subfigure}
    \begin{subfigure}{.45\textwidth}
    \centering
    \include{2group_fundlims_bg}
    \caption{$\mathsf{BG}(\epsilon)$ prior block}
    \end{subfigure}
    \caption{Mean square errors achievable with Bayes-optimal AMP (points) contrasted to the theoretical minimum mean squared errors (lines) implied by the fundamental limits. Different colors indicate the sparsity levels $\epsilon$ chosen. The figure insets zoom the lines describing the MMSE lower bounds in the region where phase transitions occur.} \label{fig:hetero_2group_fundlimits}
\end{figure}

At low sparsity levels ($\epsilon \in \{0.5, 1\}$), the MMSE lower bound curves show a smooth transition from the unestimable regime to the weak recovery regime exactly at $\|\mT_c\|_{\mathrm{op}} = 1$, and the MSE achieved by AMP matches the theoretical optimum at all values of $\|\mT_c\|_{\mathrm{op}}$. On the contrary, as sparsrity increases, the MMSE curves exhibit jump discontinuities in the region $\|\mT_c\|_{\mathrm{op}} < 1$, and the MMSE falls below 1 before the AMP phase transition occurs, indicating the presence of statistical-to-computational gaps. We also observe that the sparsity in $X_2$ also influences the MMSE transitions for $X_1$, whose prior distribution is unchanged across the various settings. Finally, we note that once $\|\mT_c\|_{\mathrm{op}} > 1$ the AMP MSE exhibits a jump discontinuity and aligns exactly with the MMSE, indicating that past the weak recovery threshold AMP achieves Bayes-optimal performance across all sparsity level considered.

\section{Conclusion and discussion}
To summarize the contributions of this paper, we have introduced an AMP algorithm adapted to the structure of the MTP model, in which the information coming from the different views of the signal is combined linearly, preserving the asymptotic characterization of the AMP iterates as additive Gaussian noise-corrupted observations of the signal. Extending this approach to cases in which the noise is non-Gaussian is surely an interesting prospect both from the theoretical and practical view, and is in line with the recent interest in extending the AMP theory to larger classes of random matrix ensembles.

We furthermore showed that, in a Bayesian estimation setting, the linear reweighting step is optimal with respect to mean-squared error when the appropriate weights are chosen, leading to a particularly simple characterization of the associated SE recursion in terms of a nonlinear overlap map $\psi$ and a linear operator $\cT$. 

We highlighted the connections between the Bayes-optimal state evolution recurrence relation and the fundamental limits of inference for the MTP by highlighting the relationship between a finite-sample approximation of SE and the critical points of an asymptotically approximation formula for the mutual information of the model. Establishing rigorously the general conditions under which this correspondence carries over to the asymptotic setting is still an open question that we deem of great interest, and will be interested in exploring in future work. 

Together with the analysis of the location of SE fixed points, we presented necessary and sufficient conditions for the fixed-point overlap achievable by AMP to be stable with respect to arbitrarily small perturbation, which has very direct implication for the possibility of weak recovery of the signal when the MTP signal is asymptotically zero-mean. While our conditions are rather general, some further refinements are in principle possible. An improvement, for example, would be to characterize the conditions under which instability of a SE fixed point implies convergence to a different fixed point given a perturbation, which we do not prove in case the fixed point is unstable.

Finally, we demonstrated the generality of our approach with an application of our results to the heteroskedastic rank-one spiked matrix model, which is recovered as a special case of the MTP model.

\bibliographystyle{plain}
\bibliography{refs}

\begin{thebibliography}{10}

\bibitem{Abbe:2017aa}
Emmanuel Abbe.
\newblock Community detection and stochastic block models: recent developments.
\newblock {\em The Journal of Machine Learning Research}, 18(1):6446--6531,
  2017.

\bibitem{Bai:1988aa}
Z.~D. Bai and Y.~Q. Yin.
\newblock Necessary and sufficient conditions for almost sure convergence of
  the largest eigenvalue of a wigner matrix.
\newblock {\em The Annals of Probability}, 16(4):1729--1741, 1988.

\bibitem{Bai:1999aa}
ZD~Bai.
\newblock Methodologies in spectral analysis of large dimensional random
  matrices, a review.
\newblock {\em Statistica Sinica}, 9:611--677, 1999.

\bibitem{Bai:2012aa}
Zhidong Bai and Jianfeng Yao.
\newblock On sample eigenvalues in a generalized spiked population model.
\newblock {\em Journal of Multivariate Analysis}, 106:167--177, 2012.

\bibitem{baik2005}
Jinho Baik, G{\'e}rard {Ben Arous}, and Sandrine P{\'e}ch{\'e}.
\newblock {Phase transition of the largest eigenvalue for nonnull complex
  sample covariance matrices}.
\newblock {\em The Annals of Probability}, 33(5):1643 -- 1697, 2005.

\bibitem{Barbier:2018aa}
Jean Barbier, Mohamad Dia, Nicolas Macris, Florent Krzakala, and Lenka
  Zdeborov{\'a}.
\newblock Rank-one matrix estimation: analysis of algorithmic and information
  theoretic limits by the spatial coupling method.
\newblock {\em arXiv preprint arXiv:1812.02537}, 2018.

\bibitem{Barbier:2020ab}
Jean Barbier and Galen Reeves.
\newblock Information-theoretic limits of a multiview low-rank symmetric spiked
  matrix model.
\newblock In {\em 2020 IEEE International Symposium on Information Theory
  (ISIT)}, pages 2771--2776, 2020.

\bibitem{Bayati:2011aa}
Mohsen Bayati and Andrea Montanari.
\newblock The dynamics of message passing on dense graphs, with applications to
  compressed sensing.
\newblock {\em IEEE Transactions on Information Theory}, 57(2):764--785, 2011.

\bibitem{behne2022fundamental}
Joshua~K Behne and Galen Reeves.
\newblock Fundamental limits for rank-one matrix estimation with groupwise
  heteroskedasticity.
\newblock In {\em International Conference on Artificial Intelligence and
  Statistics}, pages 8650--8672. PMLR, 2022.

\bibitem{Ben-Arous:2020aa}
G\'erard Ben~Arous, Alexander~S. Wein, and Ilias Zadik.
\newblock Free energy wells and overlap gap property in sparse pca.
\newblock In Jacob Abernethy and Shivani Agarwal, editors, {\em Proceedings of
  Thirty Third Conference on Learning Theory}, volume 125 of {\em Proceedings
  of Machine Learning Research}, pages 479--482. PMLR, 09--12 Jul 2020.

\bibitem{berthier2020state}
Raphael Berthier, Andrea Montanari, and Phan-Minh Nguyen.
\newblock State evolution for approximate message passing with non-separable
  functions.
\newblock {\em Information and Inference: A Journal of the IMA}, 9(1):33--79,
  2020.

\bibitem{Bhatia:2003aa}
Rajendra Bhatia.
\newblock Partial traces and entropy inequalities.
\newblock {\em Linear algebra and its applications}, 370:125--132, 2003.

\bibitem{boucheron2013concentration}
St{\'e}phane Boucheron, G{\'a}bor Lugosi, and Pascal Massart.
\newblock {\em Concentration inequalities: A nonasymptotic theory of
  independence}.
\newblock Oxford university press, 2013.

\bibitem{Deshpande:2016aa}
Yash Deshpande, Emmanuel Abbe, and Andrea Montanari.
\newblock {Asymptotic mutual information for the balanced binary stochastic
  block model}.
\newblock {\em Information and Inference: A Journal of the IMA}, 6(2):125--170,
  12 2016.

\bibitem{Deshpande:2014aa}
Yash Deshpande and Andrea Montanari.
\newblock Information-theoretically optimal sparse pca.
\newblock In {\em 2014 IEEE International Symposium on Information Theory},
  pages 2197--2201, 2014.

\bibitem{Deshpande:2018aa}
Yash Deshpande, Subhabrata Sen, Andrea Montanari, and Elchanan Mossel.
\newblock Contextual stochastic block models.
\newblock In S.~Bengio, H.~Wallach, H.~Larochelle, K.~Grauman, N.~Cesa-Bianchi,
  and R.~Garnett, editors, {\em Advances in Neural Information Processing
  Systems}, volume~31. Curran Associates, Inc., 2018.

\bibitem{DSMM2018contextual}
Yash Deshpande, Subhabrata Sen, Andrea Montanari, and Elchanan Mossel.
\newblock Contextual stochastic block models.
\newblock In S.~Bengio, H.~Wallach, H.~Larochelle, K.~Grauman, N.~Cesa-Bianchi,
  and R.~Garnett, editors, {\em Advances in Neural Information Processing
  Systems}, volume~31. Curran Associates, Inc., 2018.

\bibitem{feng2022unifying}
Oliver~Y Feng, Ramji Venkataramanan, Cynthia Rush, Richard~J Samworth, et~al.
\newblock A unifying tutorial on approximate message passing.
\newblock {\em Foundations and Trends{\textregistered} in Machine Learning},
  15(4):335--536, 2022.

\bibitem{Fletcher:2018aa}
Alyson~K Fletcher and Sundeep Rangan.
\newblock {Iterative reconstruction of rank-one matrices in noise}.
\newblock {\em Information and Inference: A Journal of the IMA}, 7(3):531--562,
  01 2018.

\bibitem{Gantmakher:2000aa}
Feliks~Ruvimovich Gantmakher.
\newblock {\em The theory of matrices}, volume 131.
\newblock American Mathematical Soc., 2000.

\bibitem{gerbelot2021graph}
C{\'e}dric Gerbelot and Rapha{\"e}l Berthier.
\newblock Graph-based approximate message passing iterations.
\newblock {\em arXiv preprint arXiv:2109.11905}, 2021.

\bibitem{guionnet2022low}
Alice Guionnet, Justin Ko, Florent Krzakala, and Lenka Zdeborov{\'a}.
\newblock Low-rank matrix estimation with inhomogeneous noise.
\newblock {\em arXiv preprint arXiv:2208.05918}, 2022.

\bibitem{Guo:2005aa}
Dongning Guo, S.~Shamai, and S.~Verdu.
\newblock Mutual information and minimum mean-square error in gaussian
  channels.
\newblock {\em IEEE Transactions on Information Theory}, 51(4):1261--1282,
  2005.

\bibitem{guo2011}
Dongning Guo, Yihong Wu, Shlomo~S. Shitz, and Sergio Verdú.
\newblock Estimation in gaussian noise: Properties of the minimum mean-square
  error.
\newblock {\em IEEE Transactions on Information Theory}, 57(4):2371--2385,
  2011.

\bibitem{johnstone2001distribution}
Iain~M Johnstone.
\newblock On the distribution of the largest eigenvalue in principal components
  analysis.
\newblock {\em The Annals of statistics}, 29(2):295--327, 2001.

\bibitem{Kabashima:2016aa}
Yoshiyuki Kabashima, Florent Krzakala, Marc M{\'e}zard, Ayaka Sakata, and Lenka
  Zdeborov{\'a}.
\newblock Phase transitions and sample complexity in bayes-optimal matrix
  factorization.
\newblock {\em IEEE Transactions on Information Theory}, 62(7):4228--4265,
  2016.

\bibitem{Lesieur:2017aa}
Thibault Lesieur, Florent Krzakala, and Lenka Zdeborov{\'a}.
\newblock Constrained low-rank matrix estimation: phase transitions,
  approximate message passing and applications.
\newblock {\em Journal of Statistical Mechanics: Theory and Experiment},
  2017(7):073403, jul 2017.

\bibitem{Mayya:2019aa}
Vaishakhi Mayya and Galen Reeves.
\newblock Mutual information in community detection with covariate information
  and correlated networks.
\newblock In {\em 2019 57th Annual Allerton Conference on Communication,
  Control, and Computing (Allerton)}, pages 602--607, 2019.

\bibitem{Mondelli:2021aa}
Marco Mondelli and Ramji Venkataramanan.
\newblock Approximate message passing with spectral initialization for
  generalized linear models.
\newblock In Arindam Banerjee and Kenji Fukumizu, editors, {\em Proceedings of
  The 24th International Conference on Artificial Intelligence and Statistics},
  volume 130 of {\em Proceedings of Machine Learning Research}, pages 397--405.
  PMLR, 13--15 Apr 2021.

\bibitem{Mondelli:2021ab}
Marco Mondelli and Ramji Venkataramanan.
\newblock Pca initialization for approximate message passing in rotationally
  invariant models.
\newblock In M.~Ranzato, A.~Beygelzimer, Y.~Dauphin, P.S. Liang, and J.~Wortman
  Vaughan, editors, {\em Advances in Neural Information Processing Systems},
  volume~34, pages 29616--29629. Curran Associates, Inc., 2021.

\bibitem{montanari2021estimation}
Andrea Montanari and Ramji Venkataramanan.
\newblock Estimation of low-rank matrices via approximate message passing.
\newblock {\em The Annals of Statistics}, 49(1):321--345, 2021.

\bibitem{pak2023optimal}
Aleksandr Pak, Justin Ko, and Florent Krzakala.
\newblock Optimal algorithms for the inhomogeneous spiked wigner model.
\newblock {\em arXiv preprint arXiv:2302.06665}, 2023.

\bibitem{Palomar:2006aa}
D.P. Palomar and S.~Verdu.
\newblock Gradient of mutual information in linear vector gaussian channels.
\newblock {\em IEEE Transactions on Information Theory}, 52(1):141--154, 2006.

\bibitem{Payaro:2009aa}
Miquel Payaro and Daniel~P. Palomar.
\newblock Hessian and concavity of mutual information, differential entropy,
  and entropy power in linear vector gaussian channels.
\newblock {\em IEEE Transactions on Information Theory}, 55(8):3613--3628,
  2009.

\bibitem{Peche:2006aa}
S.~P{\'e}ch{\'e}.
\newblock The largest eigenvalue of small rank perturbations of hermitian
  random matrices.
\newblock {\em Probability Theory and Related Fields}, 134(1):127--173, 2006.

\bibitem{reeves2020information}
Galen Reeves.
\newblock Information-theoretic limits for the matrix tensor product.
\newblock {\em IEEE Journal on Selected Areas in Information Theory},
  1(3):777--798, 2020.

\bibitem{Reeves:2018aa}
Galen Reeves, Henry~D. Pfister, and Alex Dytso.
\newblock Mutual information as a function of matrix snr for linear gaussian
  channels.
\newblock In {\em 2018 IEEE International Symposium on Information Theory
  (ISIT)}, pages 1754--1758, 2018.

\bibitem{Watrous:2018aa}
John Watrous.
\newblock {\em The Theory of Quantum Information}.
\newblock Cambridge University Press, 2018.

\end{thebibliography}

\appendix

\section{Proofs for Section~\ref{sec:amp_mtp_main_result_symmetric}} 
\subsection{SE for AMP with non-separable denoisers}
The main ingredient for our results about AMP iterations is a matrix-valued version of the AMP convergence theorem for non-separable denoisers introduced in \cite{berthier2020state}. For deterministic signals, such a matrix version was proved in \cite{gerbelot2021graph} by extending the Long AMP theory of \cite{berthier2020state} to handle matrix-valued recursions. Here, we provide a statement of the result with our notation, and refer to said sources for a proof. 

\begin{remark}
    Unlike \cite{berthier2020state, gerbelot2021graph}, our Theorem~\ref{th:amp_mtp_state_evolution_theorem_asymmetric} is stated in terms of \textit{random} signals $\bX$ and initializations $\bM^0$. Such a modification can be made as long as the assumption on the signal structure \ref{as:mtp_model_amp_signal_assumptions} allows for concentration of the type in \eqref{eq:mtp_asymmetric_signal_initialization_concentration_assumption} and the relevant expectation limits are assumed to be finite. In practice, the proof of the results are essentially unchanged, as they can be carried out conditionally on the sigma-algebra generated by $\bX$ and the conditioning can be subsequently removed via a simple union probability argument by dominated convergence.
\end{remark}

We assume a (sequence of) random matrices $\bG \sim \GOE(n) \in \R^{n \times n}$ independent of a random signal sequence $\bX \in \R^{n \times d}$ satisfying assumption \ref{as:mtp_model_amp_signal_assumptions_symmetric}. Furthermore, we have a sequence of denoisers $\{ f_t : \R^{n \times d} \rightarrow \R^{n \times d} \}_{t \in \N} \subset \Lip(L)$ satisfying \ref{as:mtp_model_amp_denoiser_assumption_symmetric} and an initialization $\bM^0 \in \R^{n \times d}$ such that \ref{as:mtp_model_amp_initialization_assumption_symmetric} holds. Note that parts of the cited assumptions are redundant as the pertain to the spiked setting, but are surely sufficient for an AMP SE theorem when the iterations are acting of pure noise. Throughout this section, for $s < t \in \N \cup \{0\}$ we recursively define the SE matrices as
\begin{align}
    \mSigma^1 &\coloneqq \lim_{n \rightarrow \infty} \frac{1}{n} \E[(\bM^0)^\top \bM^0]; \\
    \mSigma^{t+1} & \coloneqq \lim_{n \rightarrow \infty} \frac{1}{n} \E[f_t(\bZ^t)^\top f_t(\bZ^t)]; \\
    \mSigma^{s+1, t+1} &\coloneqq \lim_{n \rightarrow \infty} \frac{1}{n} \E[f_s(\bZ^s)^\top f_t(\bZ^t)],
\end{align}
where $\bZ^t \sim \Normal(0, \mSigma^t \otimes \Id_n)$ and we identified $f_0 \equiv \bM^0$. The AMP recursion for $t \in \N$ is defined as 
\begin{align} \label{eq:amp_matrix_valued_symmetric_gaussian_noise_recursion}
    \bX^t = \frac{1}{\sqrt{n}}\bG \bM^{t-1} - \bM^{t-2} (\mB^{t-1})^\top ; \quad \bM^t = f_t(\bX^t),
\end{align}
with the correction terms defined entry-wise as
\begin{align} 
    [\mB^t]_{jk} = \frac{1}{n} \sum_{i=1}^n \E\!\left[ \frac{\partial}{\partial Z_{ik}^t} [f_t(\bZ^t)]_{ij} \right].
\end{align}
We have the following theorem.
\begin{theorem}[\cite{gerbelot2021graph}, Theorem 2] \label{th:amp_matrix_valued_symmetric_gaussian_noise}
    Consider the AMP recursion \eqref{eq:amp_matrix_valued_symmetric_gaussian_noise_recursion} with all the above definitions and assumptions \ref{as:mtp_model_amp_signal_assumptions_symmetric}-\ref{as:mtp_model_amp_initialization_assumption_symmetric}. Then, for $t \in \N, p \geq 1, L < \infty$ and a sequence of test functions $\{ \phi_n: \R^{n \times (t+1)d} \rightarrow \R^{n \times (t+1)d} \}_{n \in \N} \subset \PL_p(L)$, it holds
    \begin{align}
        \plim_{n \rightarrow \infty} \left| \phi_n(\bX, \bX^1, \dotsc, \bX^t) - \E[\phi(\bX, \bZ^1, \dotsc, \bZ^t)] \right| = 0.
    \end{align}
\end{theorem}

\subsection{AMP in the Spiked Matrix Setting} \label{sec:spiked_matrix_amp_proof}
Thus far, we have stated our AMP convergence result for algorithm \eqref{eq:amp_matrix_valued_symmetric_gaussian_noise_recursion}, which acts directly on a GOE matrix. In the MTP, however, the views $\{\bY_k\}_{k \in [K]}$ are seen as low-rank spiked matrix models, in which the GOE component is the additive Gaussian noise. In this section, we show via a standard shifting argument that under assumptions \ref{as:mtp_model_amp_signal_assumptions_symmetric}-\ref{as:mtp_model_amp_initialization_assumption_symmetric} the AMP iterates obtained from a spiked matrix model in GOE noise behave as Gaussian-noise corrupted versions of the original signal. Let us assume without loss of generality (as will be shown later) that $K=1$, so that the observation model is described by
\begin{align} \label{eq:spiked_matrix_symmetric_definition}
    \bY = \frac{1}{n} \bX \mLambda \bX + \frac{1}{\sqrt{n}}\bG,
\end{align}
where $\bG \sim \GOE(n)$ and $\mLambda \in \Sym^d$. For an initialization $\bM^0$ and a sequence of uniformly Lipschitz denoisers $\{f_t : \R^{n \times d} \rightarrow \R^{n \times d}\} \subset \PL(L)$, $L < \infty$, we write the AMP algorithm as
\begin{align} \label{eq:amp_spiked_matrix_symmetric_recursion}
    \bX^t &= \bY \bM^{t-1} - \bM^{t-2}(\mB^{t-1})^\top; \quad \bM^t = f_t(\bX^t).
\end{align}
Under the assumption that \ref{as:mtp_model_amp_signal_assumptions_symmetric}-\ref{as:mtp_model_amp_initialization_assumption_symmetric} hold, we define the SE parameters recursively as
\begin{align}
    \mK^1 &= \lim_{n \rightarrow \infty} \frac{1}{n} \mLambda\E[\bX^\top\bM^0]; \quad 
    \mSigma^1 = \lim_{n \rightarrow \infty} \frac{1}{n} \E[(\bM^0)^\top \bM^0] 
\end{align}
and then for $t \in \N \cup \{0\}$ we define $\bH^t \coloneqq \bX \mK^t + \bZ^t$, $\bZ^t \sim \Normal(0, \mSigma \otimes \Id_n)$ and have
\begin{align}
    \mK^{t+1} &= \lim_{n \rightarrow \infty} \frac{1}{n} \E\!\left[ \mLambda\bX^\top f_t( \bH^t) \right]; \quad 
    \mSigma^{t+1} = \lim_{n \rightarrow \infty} \frac{1}{n} \E\!\left[ f_t(\bH^t)^\top f_t(\bH^t) \right].
\end{align}
where $\mZ^{t-1} \sim \Normal\!\left( 0, \mSigma^{t-1} \otimes \Id_n \right)$. Finally, letting $f_0 \equiv \bM^0$, we have that for $s < t \in \N \cup \{0\}$ the matrices $\bZ^{s+1}, \bZ^{t+1}$ are jointly Gaussian and
\begin{align}
\frac{1}{n} \E\!\left[ (\bZ^{s+1})^\top \bZ^{t+1} \right] = \lim_{n \rightarrow \infty} \frac{1}{n} \E\!\left[ f_s(\bH^s)^\top f_t(\bH^t) \right] \eqqcolon \mSigma^{s+1, t+1}.
\end{align}
The correction terms $\mB^{t}$ are defined entry-wise as 
\begin{align}
    [\mB^{t}]_{jk} = \frac{1}{n} \sum_{i=1}^n \E\!\left[ \frac{\partial}{\partial H_{ik}^t} [f_t(\bH^t)]_{ij} \right].
\end{align}
The AMP SE theorem for the spiked matrix model is as follows.
\begin{theorem} \label{th:amp_spiked_matrix_symmetric_theorem}
    Consider the spiked matrix \eqref{eq:spiked_matrix_symmetric_definition} model under assumptions \ref{as:mtp_model_amp_signal_assumptions_symmetric}-\ref{as:mtp_model_amp_initialization_assumption_symmetric}, and the AMP algorithm given in \eqref{eq:amp_spiked_matrix_symmetric_recursion}. Then, for $t \in \N, p \geq 1, L < \infty$ and any sequence of test functions $\{ \phi_n : \R^{n \times (t+1)d} \rightarrow \R \}_{n \in \N} \subset \PL_p(L)$, it holds
    \begin{align}
        \plim_{n \rightarrow \infty} \left| \phi_n\!\left( \bX, \bX^1, \dotsc, \bX^t \right) - \E\!\left[\phi_n\!\left( \bX, \bH^1, \dotsc, \bH^t \right)\right] \right| = 0.
    \end{align}
\end{theorem}

\begin{proof}
    We prove the theorem via a comparison with an ``oracle" algorithm that acts on the noise component $\bG \sim \GOE(n)$ of \eqref{eq:spiked_matrix_symmetric_definition}, to which Theorem~\ref{th:amp_matrix_valued_symmetric_gaussian_noise} applies. In particular, for $t \in \N$ define the (random) denoisers $\{g_t : \R^{n \times d} \rightarrow \R^{n \times d}\}_{n \in \N}$ based on the original denoiser sequence $\{f_t : \R^{n \times d} \rightarrow \R^{n \times d}\}_{n \in \N}$ as
    \begin{align}
        g_t(\mY) = f_t( \bX \mK^t + \mY), 
    \end{align}
    and the oracle AMP recursion matched to the iterations in \eqref{eq:amp_spiked_matrix_symmetric_recursion} is
    \begin{align} \label{eq:spiked_model_amp_oracle}
    \bX_O^t &= \frac{1}{\sqrt{n}}\bG \bM_O^{t-1} - \bM_O^{t-2}(\mB_O^{t-1})^\top ; \quad
    \bM_O^t = g_t(\bX_O^t),
    \end{align}
    with $\bM_O^0 = \bM^0$. It is immediate to see that by construction the SE parameters associated with this second recursion are exactly $\mK^t$ and $\mSigma^t$ for all $t \in \N$, and the same holds for the cross-covariance terms $\mSigma^{s,t}$. We write the limiting Gaussian distributions associated with $\bX_O^t$ in \eqref{eq:spiked_model_amp_oracle} as $\bZ^t \sim \Normal\!\left( 0, \mSigma^t \otimes \Id_n \right)$, $\E[(\bZ^s)^\top \bZ^t]/n = \mSigma^{s,t}$, and use the correction terms defined element-wise as
    \begin{align}
        [\mB_O^{t}]_{jk} = \frac{1}{n} \sum_{i=1}^n \E\!\left[ \frac{\partial}{\partial Z_{ik}^t} [g_t(\bZ^t)]_{ij} \right],
    \end{align}
    where the expectation is taken also with respect to $\bX$. Writing $\bH^t = \bX \mK^t + \bZ^t$, it holds by definition that $\mB^{t} = \mB_O^{t}$ for all $t \in \N$. Then, Theorem~\ref{th:amp_matrix_valued_symmetric_gaussian_noise} yields that, for any sequence $\{ \phi_n : \R^{n \times (t+1)d} \rightarrow \R \}_{n \in \N} \subset \PL_p(L)$, it holds
    \begin{align}
        \left| \phi_n\!\left( \bX, \bX_O^1, \dotsc, \bX_O^t \right) - \E\!\left[ \phi_n\!\left( \bX, \bZ^1,\dotsc, \bZ^t \right) \right] \right| \parrow 0,
    \end{align}
    which in turn implies
    \begin{align} \label{eq:spiked_model_shifted_recursion_convergence}
        \left| \phi_n\!\left( \bX, \bX \mK_1 + \bX_O^1, \dotsc, \bX \mK^t + \bX_O^t \right) - \E\!\left[ \phi_n\!\left( \bX, \bH^1, \dotsc, \bH^t \right) \right] \right| \parrow 0.
    \end{align}
    We are left to verify that as $n \rightarrow \infty$ the shifted recursion approximates well the original AMP iterates $\bX_t$ in \eqref{eq:amp_spiked_matrix_symmetric_recursion} for any sequence of uniformly pseudo-Lipschitz test functions, i.e.
    \begin{align} \label{eq:spiked_model_AMP_to_shift_discrepancy}
        \left| \phi_n\!\left( \bX, \bX^1, \dotsc, \bX^t \right) - \phi_n\!\left( \bX, \bX \mK_1 + \bX_O^1, \dotsc, \bX \mK^t + \bX_O^t \right) \right| \parrow 0,
    \end{align}
    for which it suffices to show that
    \begin{align} \label{eq:spiked_model_discrepancy_norm}
        \left\| \left( \bX, \bX^1, \dotsc, \bX^t\right) - \left( \bX, \bX \mK_1 + \bX_O^1, \dotsc, \bX \mK^t + \bX_O^t \right) \right\|_n \parrow 0
    \end{align}
    because of Lemma~\ref{lem:pseudo_lipschitz_convergence_of_argument_convergence_of_function}. We show the claim \eqref{eq:spiked_model_AMP_to_shift_discrepancy} holds by induction. For the base case $t=1$, we simply note that
    \begin{align}
        \bX^1 = \frac{1}{n} \bX \mLambda \bX^\top \bM^0 + \bX_O^1.
    \end{align}
    From \ref{as:mtp_model_amp_initialization_assumption_symmetric} we have that $(\mLambda \bX^\top \bM^0)/n \parrow \mK^1$ and therefore 
    \begin{align}
        \| \bX_1 - (\bX\mK^1 + \bX_O^1) \|_n &= \| \bX [(\mLambda \bX^\top \bM^0)/n - \mK^1] \|_n \\ 
        &\leq \| \bX \|_n \| (\mLambda \bX^\top \bM^0)/n - \mK^1 \|_{\mathrm{op}} \parrow 0,
    \end{align}
    since $\| \bX \|_n$ has a finite in-probability limit by \ref{as:mtp_model_amp_signal_assumptions_symmetric}. Let us now assume by way of induction that \eqref{eq:spiked_model_discrepancy_norm}, and hence \eqref{eq:spiked_model_AMP_to_shift_discrepancy}, holds up until iteration $t-1$. Then, at iteration $t$, we use the uniformly pseudo-Lipschitz property to write
    \begin{align}
        &\left| \phi_n\!\left( \bX, \bX^1, \dotsc, \bX^t \right) - \phi_n\!\left( \bX, \bX\mK^1 + \bX_O^1, \dotsc, \bX\mK^t + \bX_O^t \right) \right| \\
        &\quad \leq L \left\| \left( \bX^1 - (\bX \mK^1 + \bX_O^1), \dotsc, \bX^t - (\bX \mK^t + \bX_O^t) \right)  \right\|_{n} \nonumber \\
        &\qquad \times \left( 1 + \|  (\bX, \bX^1, \dotsc, \bX^t) \|_{n}^{p-1} + \| (\bX, \bX\mK^1 + \bX_O^1, \dotsc, \bX\mK^t + \bX^t) \|_{n}^{p-1} \right) \\
        &\quad \leq L \max\{1, 2^{(p-3)/2}\}\left( \| \left( \bX^1 - (\bX \mK^1 + \bX_O^1), \dotsc, \bX^{t-1} - (\bX \mK^t + \bX_O^{t-1}) \right) \|_{n} + \| \bX^t - (\bX \mK^t + \bX_O^t) \|_{n} \right) \nonumber \\
        &\qquad \times \left[ 1 + \| (\bX, \bX^1, \dotsc, \bX^{t-1}) \|_{n}^{p-1} + \| (\bX, \bX\mK^1 + \bX_O^1, \dotsc, \bX\mK^t + \bX^{t-1}) \|_{n}^{p-1} \right. \nonumber \\ 
        &\qquad \quad \left. + \|\bX^t\|_{n}^{p-1} + \|\bX\mK^t + \bX_O^t\|_{n}^{p-1} \right].
    \end{align}
    From the inductive hypothesis, we have
    \begin{align}
        \left\| \left( \bX^1 - (\bX \mK^1 + \bX_O^1), \dotsc, \bX^{t-1} - (\bX \mK^t + \bX_O^{t-1}) \right) \right\|_{n} \parrow 0,
    \end{align}
    and furthermore the inductive hypothesis on \eqref{eq:spiked_model_shifted_recursion_convergence}-\eqref{eq:spiked_model_discrepancy_norm} and Theorem~\ref{th:amp_matrix_valued_symmetric_gaussian_noise} applied to $\bX_O^t$ ensure that the quantities
    \begin{align}
        \| (\bX, \bX^1, \dotsc, \bX^{t-1}) \|_{n}^{p-1}; \quad  \|(\bX, \bX\mK^1 + \bX_O^1, \dotsc, \bX\mK^t + \bX^{t-1}) \|_{n}^{p-1}; \quad \|\bX\mK^t + \bX_O^t\|_{n}^{p-1}.
    \end{align}
    have well-defined and finite in-probability limits. Therefore, proving \eqref{eq:spiked_model_AMP_to_shift_discrepancy} reduces to showing that 
    \begin{align}
        \| \bX^t - (\bX \mK^t + \bX_O^t) \|_{n} \parrow 0.
    \end{align}
    By inspecting the AMP recursions that generate both $\bX^t$ and $\bX_O^t$, we have 
    \begin{align}
        \bX^t &= \frac{1}{n} \bX \mLambda \bX^\top f_{t-1}(\bX^{t-1}) + \frac{1}{\sqrt{n}}\bG f_{t-1}(\bX^{t-1}) - f_{t-2}(\bX^{t-2})(\mB^{t-1})^\top, \\
        \bX_O^t &= \frac{1}{\sqrt{n}}\bG g_{t-1}(\bX_O^{t-1}) - g_{t-2}(\bX_O^{t-2})(\mB_O^{t-1})^\top.
    \end{align}
    Letting $\Hat{\mK}_n^t \coloneqq \left(\mLambda \bX^\top f_{t-1}(\bX^{t-1})\right) / n$ and substituting $g_{t-1}$ and $g_{t-2}$ with the respective definitions, we obtain
    \begin{align}
         \left\| \bX^t - (\bX \mK^t + \bX_O^t) \right\|_{n} &= \left\| \bX ( \Hat{\mK}_n^t - \mK^t ) + n^{-1/2} \bG\left( f_{t-1}(\bX^{t-1}) - f_{t-1}( \bX\mK^{t-1} + \bX_O^{t-1}) \right) \right. \nonumber \\
         & \left. \qquad - \left( f_{t-2}(\bX^{t-2}) - f_{t-2}(\bX\mK^{t-2} + \bX_O^{t-2}) \right) (\mB^{t-1} )^\top \right\|_{n} \\
         &\leq \left\| \bX ( \Hat{\mK}_n^t - \mK^t ) \right\|_{n} + \left\| n^{-1/2} \bG\left( f_{t-1}(\bX^{t-1}) - f_{t-1}( \bX\mK^{t-1} + \mX_O^{t-1}) \right) \right\|_{n} \nonumber\\
         & \qquad + \left\| \left( f_{t-2}(\bX^{t-2}) - f_{t-2}(\bX\mK^{t-2} + \bX_O^{t-2}) \right) (\mB^{t-1} )^\top \right\|_{n}.
    \end{align}
    From the inductive hypothesis, we have that $\plim_{n \rightarrow \infty} \Hat{\mK}_n^t = \mK^t$ and $K^t$ is a finite-dimensional object, so
    \begin{align}
        \left\| \bX ( \Hat{\mK}_n^t - \mK^t ) \right\|_{n} \leq \| \bX \|_{n} \| \Hat{\mK}_n^t - \mK^t\|_{\mathrm{op}} \parrow 0
    \end{align}
    since $\| \mX \|_{n}$ has a finite in-probability limit by \ref{as:mtp_model_amp_signal_assumptions_symmetric}. For the second term, we use Theorem~\ref{th:largest_eigenvalue_goe} and the uniform Lipschitz assumption on the denoisers $f_{t-1}$:
    \begin{align}
        \left\| n^{-1/2} \bG \left( f_{t-1}(\bX^{t-1}) - f_{t-1}( \bX\mK^{t-1} + \bX_O^{t-1}) \right) \right\|_{n} &\leq L \|n^{-1/2} \bG\|_{\mathrm{op}} \| \bX^{t-1} - ( \bX\mK^{t-1} + \bX_O^{t-1}) \|_n \\
        &\leq L (B + o_p(1)) \| \bX^{t-1} - ( \bX\mK^{t-1} + \bX_O^{t-1}) \|_n \parrow 0,
    \end{align}
    for $B<\infty$, with convergence following from the inductive hypothesis.
    Finally, since $\mB^{t-1}$ is a finite-dimensional matrix with bounded entries (as an expected average of Jacobians of Lipschitz functions), the uniformly Lipschitz assumption on $f_{t-2}$ and the inductive hypothesis imply that
    \begin{align}
        \left\| \left( f_{t-2}(\bX^{t-2}) - f_{t-2}(\bX\mK^{t-2} + \bX_O^{t-2}) \right) (\mB^{t-1} )^\top \right\|_{n} \leq L \|\mB^{t-1}\|_\mathrm{op} \left\| \bX^{t-2} - \bX\mK^{t-2} + \bX_O^{t-1} \right\|_{n} \parrow 0.
    \end{align}
    Thus, we conclude that 
    \begin{align}
        \| \bX^t - (\bX \mK^t + \bX_O^t) \|_{n} \parrow 0
    \end{align}
    and the induction step is complete.
\end{proof}

\subsection{Proof of Theorem~\ref{th:amp_mtp_symmetric_main_result}} \label{sec:mtp_amp_symmetric_proof}
Finally, we provide a proof of Theorem~\ref{th:amp_mtp_symmetric_main_result}. We prove our main result by embedding the symmetric MTP model \eqref{eq:MTP_mv_symmetric_def} into a single-view, block-diagonal spiked matrix model of the same form as \eqref{eq:spiked_matrix_symmetric_definition}. We run our AMP algorithm \eqref{eq:amp_spiked_matrix_symmetric_recursion} on it with a special choice of denoiser constructed from the intended denoiser sequence for the MTP, which as usual we denote $\{f_t: \R^{n \times d} \rightarrow \R^{n \times d} \}_{t \in \N} \subset \Lip(L)$, for some $L< \infty$. We restate the model \eqref{eq:MTP_mv_symmetric_def} here for the reader's convenience. For $k \in [K]$, we observe
\begin{align}
    \bY_k = \frac{1}{n} \bX \mLambda_k \bX^\top + \frac{1}{\sqrt{n}} \bG_k, 
\end{align}
for $\bG_k \iid \GOE(n)$ and $\mLambda_k \in \Sym^d$. Furthermore, we assume the signal $\bX$, the initialization $\bM^0$ and the denoiser sequences $\{f_t\}_{t \in \N}$ all satisfy \ref{as:mtp_model_amp_signal_assumptions_symmetric}-\ref{as:mtp_model_amp_initialization_assumption_symmetric}. We also specify some rewighting matrices $\{ \mA_k^t \in \R^{d \times d} \}_{k \in [K], t \in \N}$.

The symmetric AMP iterations $\{ \bX^t \mid \bM^0, \mA_k^t, f_t \}_{t \in \N}$ are constructed for $t\in \N$ as
\begin{align} 
    \bX^{t} &= \sum_{k=1}^K \bY_k \bM^{t-1} (\mA_k^{t})^\top - \bM^{t-2} (\mB^{t-1})^\top;  \quad
    \bM^{t} = f_{t}(\bX^{t}),
\end{align}
and the correction terms are defined, together with the SE matrices, as in Section~\ref{sec:amp_mtp_main_result_symmetric}. 

We now begin with the embedding. From now on, we will write $N \coloneqq nK$ and $D \coloneqq dK$. Furthermore, we make the following definitions:
\begin{alignat}{2}
    \tX &\coloneqq \Id_K \otimes \mX && \quad \in \R^{N \times D}; \\
    \tLambda &\coloneqq \begin{bmatrix}
        \sqrt{K}\mLambda_1 & & & \\
        & \sqrt{K}\mLambda_2 & & \\
        & & \ddots & \\
        & & & \sqrt{K}\mLambda_K
    \end{bmatrix} && \quad \in \R^{D \times D},
\end{alignat}
and define the lifted observation model as
\begin{align}
    \tY = \frac{1}{N} \tX \tLambda \tX^\top + \frac{1}{\sqrt{N}}\tG \quad \in \R^{N \times N},
\end{align}
with $\tG \sim \GOE(N)$. For convenience, we establish come notation to keep track of the block structure in this lifted model. For $\tY \in \R^{N \times N}$ and $k \in [K]$, we define the $k$-th block $\tY_{(k)} \in \R^{n \times d}$ of $\tY$ as
\begin{align}
    \tY_{(k)} = (\tY_{ij}), \quad i \in \{(n(k-1) + 1, nk\}, \ j \in \{ n(k-1) + 1, nk \},
\end{align}
and similarly for matrices $\tX \in \R^{N \times D}$ we write $\tX_{(k)} \in \R^{n \times d}$ to denote the respective block-diagonal element of $\tX$:
\begin{align}
    \tX_{(k)} = (\tX_{ij}), \quad i \in \{(n(k-1) + 1, nk\}, \ j \in \{ d(k-1) + 1, dk \}.
\end{align}
As a remark, we have that $\sqrt{K}\tY_{(k)} = \bY_k$, and no information about $\bX$ is contained in off-diagonal blocks of $\tY$, so that the collection $\{\bY_k \in \R^{n \times n}\}_{k \in [K]}$ and $\tY$ are observationally equivalent. Let us define the initialization $\tM^0 \coloneqq \Id_K \otimes \bM^0$ and, for $t \in \N$, we  construct the denoisers $\{F_t : \R^{N \times D} \rightarrow \R^{N \times D}\}_{N \in \N}$ as
\begin{align}
    F_t(\tX) \coloneqq \Id_K \otimes f_t\!\left( \sum_{k=1}^K \tX_{(k)} (\mA_k^t)^\top \right),
\end{align}
and for convenience we write the linear transformation
\begin{align}
    \tX \mapsto \sum_{k=1}^K \tX_{(k)}(\mA_k^t)^\top
\end{align}
as $T_t(\tX)$ and denote $\bX^t = T_t(\tX^t)$. 

It is easily verified that if each $f_t \subset \Lip(L)$ then necessarily the corresponding $F_t \subset \Lip(L')$ as well for some $L' < \infty$, and we can describe an AMP algorithm acting onto $\tY$ as a valid single-view spiked matrix model as
\begin{align} \label{eq:spiked_model_embed_amp}
    \tX^t &= \tY \tM^{t-1} - \tM^{t-2}(\tB^{t-1})^\top; \quad
    \tM^t = F_t(\tX^t).
\end{align}

Theorem~\ref{th:amp_spiked_matrix_symmetric_theorem} then provides the SE characterization for the above algorithm. We have 
\begin{align}
    \tK^{t+1} &= \lim_{N \rightarrow \infty} \frac{1}{N} \E\!\left[ \tLambda\tX^\top F_{t}( \tH^t ) \right]; \\
    \tSigma^{t+1} &= \lim_{N \rightarrow \infty} \frac{1}{N} \E\!\left[ F_{t}( \tH^t )^\top F_{t}( \tH^t ) \right]; \\
    \tSigma^{s+1,t+1} &= \lim_{N \rightarrow \infty} \frac{1}{N} \E\!\left[ F_{s}( \tH^s )^\top F_{t}( \tH^t ) \right],
\end{align}
for $\tH^t \coloneqq \tX \tK^t + \tZ^t$, with $\tZ^t \sim \Normal(0, \tSigma^t \otimes \Id_N)$, with $\E[(\tZ^s)^\top \tZ^t]/n = \tSigma^{s,t}$. Again we identified $F_0 \equiv \tM^0$. 

We note that each $\tK^t, \tSigma^t, \tSigma^{s,t}$ inherit the Kronecker structure of the denoiser and signal, and are thus block-diagonal matrices. We observe that each block $\tZ^t_{(k)}$ is independent of other blocks in the same matrix $\tZ^t$. Furthermore, $\mSigma^{s,t}$ being block-diagonal for all $s<t$ implies that $\E[(\tZ^s_{(k)})^\top \tZ^s_{(k')}] / n = \tSigma_{(k)}^{s,t}$ if $k = k'$ and is zero otherwise, meaning that the rows of the $k$-th block in the limiting noise matrix for the iterates only correlates with entries in the same block at other iterations.

We now look at the structure in the correction terms $\tB^{t}$. Since both input and output of $F_{t}$ are block-diagonal, we can decompose $\tB^t$ into blocks as well, i.e. $\tB^t = \sum_{k \in [K]} \me_k \me_k^\top \otimes \tB_{(k)}^t$, where $\me_k$ is the $k$-th standard basis vector of $\R^K$ and $\tB_{(k)}^t$ is defined element-wise for $j,l \in [d]$ as
\begin{align}
    [\tB_{(k)}^t]_{jl} = \frac{1}{N} \sum_{i=1}^N \E\!\left[ \frac{\partial}{\partial [\tH_{(k)}^t]_{il}} [f_t(T_t(\tH^t))]_{ij} \right].
\end{align}
Using the chain rule and the defintion of $T_t$, we can rewrite
\begin{align}
    \frac{\partial}{\partial [\tH_{(k)}^t]_{il}} [f_t(T_t(\tH^t))]_{ij} &= \left\langle \frac{\partial [f_t(T_t(\tH^t))]_{ij}}{\partial [T_t(\tH^t)]_{i}}, \frac{\partial [T_t(\tH^t)]_{i}}{\partial [\tH_{(k)}^t]_{il} } \right\rangle \\
    & = \left\langle \frac{\partial [f_t(\bH^t)]_{ij}}{\partial [\bZ^t]_{i}}, [\mA_k^t]_{\bullet l} \right\rangle,
\end{align}
where we defined $\bZ^t \coloneqq T_t(\tZ^t), \bH^t \coloneqq T_t(\tH^t)$ and $[\mA_k]_{\bullet l}$ is the $l$-th column vector of $\bA_k^t$. Putting the above in matrix form and recalling the definition of the average expected divergence from \eqref{eq:amp_symmetric_expected_mean_divergence_definition}, we have
\begin{align}
    \tB^t = \frac{1}{K} \sum_{k=1}^K \left\{ \me_k\me_k^\top \otimes \mD^t \mA_k^t \right\},
\end{align}
and applying $T_{t+1}$ to the transpose of the diagonal blocks $(\tB^t_{(k)})^\top$ in $\tB^t$ recovers $(\mB^t)^\top$ of \eqref{eq:amp_mtp_correction_term_symmetric_definition} up to a scaling factor of $K^{-1}$.

With this, we are done relating the lifted recursion $\{ \tX^t \mid \tM^0, F_t\}_{t \in \N}$ to our intended recursion of Theorem~\ref{th:amp_mtp_symmetric_main_result}. Indeed, from the relation $\bX^t = T_t(\tX^t)$ we apply $T_t$ on both sides of \eqref{eq:spiked_model_embed_amp} and the AMP algorithm can then be expressed as
    \begin{align}
    \bX^t &= \frac{1}{\sqrt{K}}\sum_{k=1}^K \bY_k f_{t-1} (\bX^{t-1}) (\mA_k^t)^\top - \frac{1}{K} f_{t-2}(\bX^{t-2}) (\mB^{t-1})^\top, 
\end{align}
where we used the fact that $\sqrt{K}\tY_{(k)} = \bY_k$. Relabeling $\sqrt{K}\mA_k^t$ to $\mA_k^t$, the iterations $\{ \bX^t \mid \bM^0, \mA_k^t, f_t \}_{t \in \N}$ are then verified to be an equivalent representation of $\{ \tX^t \mid \tM^0, F_t\}_{t \in \N}$. Finally, we notice that for any $t \in \N, p \geq 1, L < \infty$ any sequence of test functions $\{ \phi : \R^{n \times (t+1)d} \rightarrow \R \}_{n \in \N} \subset \PL_p(L)$ acting onto $\{\bX^t \mid \bM^0, \mA_k^t, f_t\}_{t \in \N}$ can be expressed as an equivalent sequence of test functions $\{ \Phi_N : \R^{N \times (t+1)d} \rightarrow \R \}_{N \in \N} \subset \PL_p(L')$ for some $L' < \infty$ such that
\begin{align}
    \Phi_N(\tX, \tX^1, \dotsc, \tX^t) = \phi_n( \bX, T_1(\tX^1), \dotsc, T_t(\tX^t) ).
\end{align}
This proves all the desired forms of convergence for $\{\bX^t \mid \bM^0, \mA_k^t, f_t\}_{t \in \N}$ via Theorem~\ref{th:amp_spiked_matrix_symmetric_theorem} and shows that the SE quantities associated with $\{\bX^t \mid \bM^0, \mA_k^t, f_t\}_{t \in \N}$ in Theorem~\ref{th:amp_mtp_symmetric_main_result} are correct, thus proving the theorem.

\subsection{Proof of Corollaries~\ref{cor:amp_mtp_symmetric_wasserstein_convergence_result} and \ref{cor:amp_mtp_symmetric_empirical_correction_terms}} \label{sec:symmetric_mtp_corollaries_proof}

\begin{proof}[Proof of Corollary~\ref{cor:amp_mtp_symmetric_wasserstein_convergence_result}]
    Fix any $\epsilon > 0$. For $\phi: \R^{(t+1)d} \rightarrow \R$, we can equivalently write $\phi(X, H^1, \dotsc, H^t)$ as $\varphi(X, Z^1, \dotsc, Z^t)$, for $X \indep (Z^1,\dotsc,Z^t)$ a centered Gaussian vector whose covariance is given by SE. Let $\{\phi_n : \R^{n \times (t+1)d} \rightarrow \infty\}_{n \in \N}$ be given by $\phi_n(\bX) = n^{-1} \sum_{i \in [n]} \phi(\bX_i)$, which is uniformly pseudo-Lipschitz of order 2 by Lemma~\ref{lem:averages_pseudo_lipschitz_function_uniformly_pseudo_Lipschitz}. We have
    \begin{align}
        &\P\!\left[ \left| \phi_n(\bX, \bX^1, \dotsc, \bX^t) - \E\!\left[\phi(X, H^1, \dotsc, H^t)\right] \right| > \epsilon \right] \\
        &\quad \leq \P\!\left[ \left| \phi_n(\bX, \bX^1, \dotsc, \bX^t) - \E\!\left[\phi_n(\bX, \bH^1, \dotsc, \bH^t)\mid \bX \right] \right| \right. \nonumber \\
        &\qquad \left. + \left| \E\!\left[\phi_n(\bX, \bH^1, \dotsc, \bH^t)\mid \bX \right] - \E\!\left[\phi(X, H^1, \dotsc, H^t)\right] \right| > \epsilon \right] \\
        &\quad \leq \P\!\left[ \left| \phi_n(\bX, \bX^1, \dotsc, \bX^t) - \E\!\left[\phi_n(\bX, \bH^1, \dotsc, \bH^t)\mid \bX \right] \right| \geq \epsilon/2 \right] \nonumber \\
        &\qquad + \P\!\left[\left| \E\!\left[\phi_n(\bX, \bH^1, \dotsc, \bH^t)\mid \bX \right] - \E\!\left[\phi(X, H^1, \dotsc, H^t)\right] \right| > \epsilon/2 \right]. 
    \end{align}
    We begin with the second term in the summation. Conditionally on $\bX$, let $Z_i^{[t]} \in \R^{td}$ be a collection of i.i.d. Gaussian vectors independent of $\bX$ whose distribution is the same as $(Z^1,\dotsc,Z^t)$.
    \begin{align}
        \E\!\left[ \phi_n(\bX, \bH^1, \dotsc, \bH^t)\mid \bX \right] = \frac{1}{n} \sum_{i=1}^n \E\!\left[ \varphi(X_i, Z_i^{[t]}) \mid \bX \right].
    \end{align}
    The conditional expectations $\zeta(X_i) \coloneqq \E[ \varphi(X_i, Z_i^{[t]}) \mid \bX ]$ are pseudo-Lipschitz of order 2 by Lemma~\ref{lem:gaussian_expectation_pseudo_lipschitz_is_pseudo_lipschitz}. Then, from the assumption \eqref{eq:mtp_signal_symmetric_wasserstein_limits_assumption} we have that for the second term in the sum vanishes as $n \rightarrow \infty$. For the first term, let $C_n$ denote the event $| \|\bX\|_n - \E\|\bX\|_n | \leq \epsilon/4$, and write
    \begin{align}
        &\P\!\left[ \left| \phi_n(\bX, \bX^1, \dotsc, \bX^t) - \E\!\left[\phi_n(\bX, \bH^1, \dotsc, \bH^t)\mid \bX \right] \right| \geq \epsilon/2 \right] \nonumber \\
        & \quad = \E\!\left[\P\!\left[ \left| \phi_n(\bX, \bX^1, \dotsc, \bX^t) - \E\!\left[\phi_n(\bX, \bH^1, \dotsc, \bH^t)\mid \bX \right] \right| \geq \epsilon/2 \mid \bX \right] \right] \\
        & \leq \E\!\left[\P\!\left[ \left| \phi_n(\bX, \bX^1, \dotsc, \bX^t) - \E\!\left[\phi_n(\bX, \bH^1, \dotsc, \bH^t)\mid \bX \right]  \right| \ind_{C_n} \geq \epsilon/4 \mid \bX \right] \right] \nonumber \\
        &\qquad + \E\!\left[\P\!\left[ \left| \phi_n(\bX, \bX^1, \dotsc, \bX^t) - \E\!\left[\phi_n(\bX, \bH^1, \dotsc, \bH^t)\mid \bX \right]  \right| \ind_{C_n^c} \geq \epsilon/4 \mid \bX \right] \right] \\
        &\leq \E\!\left[ \frac{Q}{n\epsilon^2} \E\left[ \left(1 + \|\bZ^1,\dotsc, \bZ^t\|_n + \E\| \bX \|_n + \epsilon/4 \right)^2\right] \right] + \E\!\left[ \ind_{C_n^c} \right],
    \end{align}
    where the last step follows from Lemma~\ref{lem:gaussian_expectation_pseudo_lipschitz_is_pseudo_lipschitz}. Since by assumption $\E\!\left[ \ind_{C_n^c} \right] \rightarrow 0$ as $n \rightarrow \infty$ and 
    \begin{align}
        \E\left[ \left(1 + \|\bZ^1,\dotsc, \bZ^t\|_n + \E\| \bX \|_n + \epsilon/4 \right)^2\right] = O(1)
    \end{align} 
    by assumption \ref{as:mtp_model_amp_signal_assumptions_symmetric}, the whole term becomes negligible as $n \rightarrow \infty$. Putting everything together, we conclude that
    \begin{align}
        \left| \frac{1}{n}\sum_{i=1}^n \phi(X_i, X_i^1, \dotsc, X_i^t) - \E\!\left[\phi(X, H^1, \dotsc, H^t)\right] \right| \parrow 0.
    \end{align}
\end{proof}

\begin{proof}[Proof of Corollary~\ref{cor:amp_mtp_symmetric_empirical_correction_terms}]
    This corollary is proved via an induction argument. Furthermore, due to the arguments in the previous sections, we assume without loss of generality that we are running AMP on a single-view symmetric model
    \begin{align}
        \bY = \frac{1}{n} \bX \mLambda \bX^\top + \frac{1}{\sqrt{n}} \bG.
    \end{align}
    Let $\{ \bX^t \}_{t \in \N}$ and $\{\hat{\bX}^t\}_{t \in \N}$ denote the AMP iterations obtained with the exact correction terms $\mB^t$ and the consistent estimators $\hat{\mB}^t$, respectively. The two recursions are assumed to be initialized with the same matrix $\bM^0$ and use the same denoisers. From Lemma~\ref{lem:pseudo_lipschitz_convergence_of_argument_convergence_of_function}, it is sufficient to show that, for each $t \in \N$, it holds
    \begin{align}
        \plim_{n \rightarrow \infty} \| \bX^t - \hat{\bX}^t \|_n = 0.
    \end{align}
    The base case $t=1$ is obvious from the assumptions. By way of induction, let us assume the claim holds for all $s \leq t$. Then, for $t+1$, we have
    \begin{align}
        \| \bX^{t+1} - \hat{\bX}^{t+1} \| &\leq \left\| \frac{1}{n} \bX \mLambda \bX^\top (f_t({\bX}^t) - f_t(\hat{\bX}^t) ) \right\|_n + \left\| \frac{1}{\sqrt{n}} \bG (f_t({\bX}^t) - f_t(\hat{\bX}^t) ) \right\|_n \nonumber \\
        & \quad + \left\| f_t(\bX^{t-1})(\mB^t)^\top - f_t(\hat{\bX}^{t-1})(\hat{\mB}^t)^\top\right\|_n .    
    \end{align}
    We study the three summands on the right separately. From \ref{as:mtp_model_amp_signal_assumptions_symmetric}, the inductive hypothesis and Theorem~\ref{th:amp_spiked_matrix_symmetric_theorem}, we have
    \begin{align}
        \left\| \frac{1}{n} \bX \mLambda \bX^\top (f_t({\bX}^t) - f_t(\hat{\bX}^t) ) \right\|_n &\leq \|\bX\|_n \left\| \frac{1}{n} \mLambda \bX^\top (f_t({\bX}^t) - f_t(\hat{\bX}^t) ) |\right\|_{\op} \parrow 0.         
    \end{align}
    For the second summand, we use the probability bound on the operator norm of a GOE to write
    \begin{align}
        \left\| \frac{1}{\sqrt{n}} \bG (f_t({\bX}^t) - f_t(\hat{\bX}^t) ) \right\|_n &\leq (c + o_p(1)) \| f_t({\bX}^t) - f_t(\hat{\bX}^t) \|_n \\
        &\leq (c + o_p(1)) \| {\bX}^t - \hat{\bX}^t \| \parrow 0
    \end{align}
    from uniform Lipschitzness of $f_t$ and the inductive hypothesis. Finally, for the last term, we have
    \begin{align}
        \left\| f_t(\bX^{t-1})(\mB^t)^\top - f_t(\hat{\bX}^{t-1})(\hat{\mB}^t)^\top\right\|_n &\leq \|  (f_t({\bX}^t) - f_t(\hat{\bX}^t) )(\hat{\mB})^\top \|_n + \| f(\bX^{t-1}) ( \hat{\mB}^t - \mB^t )^\top \|_n \\
        &\leq \|  f_t({\bX}^t) - f_t(\hat{\bX}^t) \|_n \| \hat{\mB}\|_{\op} + \| f_t(\bX^t) \|_n \|\hat{\mB}^t - \mB^t\|_{\op.}
    \end{align}
    Since $\mB^t$ has bounded entries by definition and $\plim_{n \rightarrow \infty} \hat{\mB} = \mB^t$, it must be that $\| \hat{\mB}^t \|_{\op} = O_p(1)$, and since $\| f_t({\bX}^t) - f_t(\hat{\bX}^t) \|_n = o_p(1)$ by the inductive hypothesis the first term in the sum vanishes in probability. For the secon term, we note that, from Theorem~\ref{th:amp_spiked_matrix_symmetric_theorem} and the assumption on $\bX$, $\| f_t(\bX^t) \|_n = O_p(1)$ and since $\hat{\mB}^t$ is a consistent estimator of $\mB^t$ also the second term is $o_p(1)$. 
    This completes the induction and proves the corollary.
\end{proof}

\section{Proof of Theorem~\ref{th:amp_mtp_state_evolution_theorem_asymmetric}} \label{sec:amp_mtp_asymmetric_theorem_proof}
Theorem~\ref{th:amp_mtp_state_evolution_theorem_asymmetric} is proved by reducing the asymmetric MTP \eqref{eq:MTP_mv_def} to an instance of the symmetric model \eqref{eq:MTP_mv_symmetric_def}. In light of the embedding used in Section~\ref{sec:mtp_amp_symmetric_proof} to prove Theorem~\ref{th:amp_mtp_symmetric_main_result}, we focus on the case $K=1$ (i.e. a single-view spiked matrix model like the one considered in Section~\ref{sec:spiked_matrix_amp_proof}) without loss of generality.

Let us assume assumptions \ref{as:mtp_model_amp_signal_assumptions}-\ref{as:mtp_model_amp_initialization_assumption} hold and all quantities are defined as in Section~\ref{sec:amp_mtp_main_result_asymmetric}, and consider an observation model
\begin{align}
    \bY = \frac{1}{n} \bX \mGamma \bX + \frac{1}{\sqrt{n}} \bW,
\end{align}
with $\bW \in \R^{n \times n}$ a matrix of i.i.d. standard Gaussian entries. We symmetrize the model by writing ${\bX}_s = \bX \oplus \bX, \bM_s^0 = \bM^0 \oplus \bM^0$, where $\oplus$ denotes the matrix direct product, and
\begin{align}
    \mLambda = \begin{bmatrix}
        0 & \sqrt{2}\mGamma \\ \sqrt{2}\mGamma^\top & 0.
    \end{bmatrix}
\end{align}
With this, we have a symmetrized observation model $\bY_s \in \R^{2n \times 2n}$
\begin{align}
    \bY_s = \frac{1}{2n} \bX_s \mLambda \bX_s^\top + \frac{1}{\sqrt{2n}} \bG = \begin{bmatrix}
        \frac{1}{\sqrt{2n}} \bG_1 & \frac{1}{\sqrt{2}} \bY \\ \frac{1}{\sqrt{2n}} \bY^\top & \frac{1}{\sqrt{n}} \bG_2
    \end{bmatrix}
\end{align}
for $\bG_1, \bG_2 \iid \GOE(n)$ and $\bY$ jointly independent. Let us now define the reweighting matrices $\{ \mA_s^t \in \R^{2d \times 2d} \}_{t \in \N}$ and the denoiser sequences $\{ h_t : \R^{2n \times 2d} \rightarrow \R^{2n \times 2d} \}_{n \in \N}$ as
\begin{align}
    \mA_s^t \coloneqq \begin{bmatrix}
        0 & \sqrt{2} \mA^t \\ \sqrt{2} (\mA^t)^\top & 0
    \end{bmatrix}; \quad h_t(y) = \begin{bmatrix}
        f_t(y_{(1)} + y_{(2)}) & 0 \\ 0 & f_t(y_{(1)} + y_{(2)})
    \end{bmatrix},
\end{align}
where the notation $y_{(i)}$ is used to denote the $n \times d$ $i$-th block diagonal elements of $y$, for $i=1,2$. Clearly $h_t \subset \Lip(L)$ for some constant $L < \infty$ whenever $f_t$ is uniformly Lipschitz. With these definitions, we simply consider the symmetric algorithm ($t \in \N$) given by
\begin{align}
    \bX_s^t = \bY_s \bM_s^{t-1} - \bM_s^{t-2}(\mB_s^{t-1})^\top, \quad \bM_s^t = h_t(\bX_s^t)
\end{align}
and notice that the iterates $\bX_s^t$ are block-diagonal and satisfy the conditions for Theorem~\ref{th:amp_mtp_symmetric_main_result}. Furthermore, straightforward algebra reveals that the SE quantities for the iterations are block-diagonal and the block-diagonal consists of the elements of the desired AMP algorithm of Section~\ref{sec:amp_mtp_main_result_asymmetric}, repeates twice. Considering test functions acting exclusively onto the top-left diagonal block of the iterates completes the reduction and proves Theorem~\ref{th:amp_mtp_state_evolution_theorem_asymmetric}.  

\section{Proof of Proposition~\ref{prop:matrix_mmse_gradient_identity}} \label{sec:mmse_matrix_gradient_identity_proof}
We show here the calculation to obtain the MMSE gradient identity \eqref{eq:matrix_mmse_gradient_identity}. Let us first establish some notation. Throughout this section, for a matrix $\mX \in \R^{n \times d}$ we will use $\mX_i$ to denote the $i$-th column of the matrix (instead of the row as in the other sections of the paper). For a given SNR matrix $\mS \in \PD^d$, we denote by $\mT \in \PD^d$ it symmetric square root. For a random variable $\bX \sim P$ on $\R^{n \times d}$, where both $n$ and $d$ are fixed, we write $p(y; \mT) = \E[\varphi(y - \bX\mT)]$ for the density of $\bY_{\mT} \coloneqq \bX\mT + \bZ$, $\bZ \sim \Normal(0, \Id_d \otimes \Id_n) \indep P$. We use $\meta_i(y; \mT), i \in [d]$ to denote the Bayes rule for the $i$-th row of $\bX$:
\begin{align}
    \meta_i(y; \mT) = \E\!\left[ \bX_i \mid \bY_{\mT} = y \right] \in \R^{n}.
\end{align}
For $i\leq j \in [d]$, then, the $ij$-th entry of the unnormalized MMSE matrix $n M_n(\mS)$ can be parametrized by $\mT$ as
\begin{align}
    \E[\bX^\top \bX]_{ij} - \int \langle \meta_i(y; \mT), \meta_j(y; \mT)\rangle p(y; \mT) \ dy. 
\end{align}
We will be computing the gradient of the integral to the right with respect to the $kl$-th entry of $\mT$, for $k\leq l \in [d]$. Throughout, we will be exchanging the order of differentiation and integration. A rigorous justification for this can be obtained (for fixed $n$) via a straightforward multivariate extension of the arguments in \cite{guo2011}. To streamline notation, we will be using the notation $\partial_{kl}$ to denote a partial derivative with respect to the $kl$-th entry of $\mT$, and $\nabla_k$ to denote a gradient with respect to the $k$-th column of $y$. We present a few differential identities that will be useful in this section. 
\begin{lemma} \label{lem:overlap_differential_identities}
    The following differential identities hold:
    \begin{enumerate}[label=(\alph*)]
        \item \begin{align}\partial_{kl}\{p(y; \mT)\} = - \sum_{i=1}^n \left\{ \frac{\partial}{\partial y_{ik}} \left\{\meta_k(y; \mT)p(y; \mT)\right\} + \frac{\partial}{\partial y_{il}} \left\{ \meta_l(y; \mT)p(y; \mT) \right\} \right\} \end{align}
        \item \begin{align}
            \partial_{kl} \{\meta_i(y; \mT)\} = - \frac{\md_i(y; \mT_{kl})}{p(y;\mT)} - \frac{\meta_i(y;\mT) \partial_{kl} \{p(y; \mT)\} }{p(y; \mT)},
        \end{align} where
        \begin{align}
            \md_i(y; \mT_{kl}) = \begin{bmatrix}
                \sum_{j=1}^n \left\{ \frac{\partial}{\partial y_{jk}} \left\{ \E[\bX_{1i} \bX_{jk} \mid \bY_{\mT}= y] p(y; \mT)\right\} + \frac{\partial}{\partial y_{jl}} \left\{ \E[\bX_{1i} \bX_{jl} \mid \bY_{\mT}= y] p(y; \mT)\right\}  \right\} \\
                \vdots \\
                \sum_{j=1}^n \left\{ \frac{\partial}{\partial y_{jk}} \left\{ \E[\bX_{ni} \bX_{jk} \mid \bY_{\mT}= y] p(y; \mT)\right\} + \frac{\partial}{\partial y_{jl}} \left\{ \E[\bX_{ni} \bX_{jl} \mid \bY_{\mT}= y] p(y; \mT)\right\}  \right\}
            \end{bmatrix}
        \end{align}
        \item \begin{align}
            \nabla_k \{ \meta_i(y; \mT) \} = \sum_{j=1}^d \mT_{jk} \Cov(\bX_i, \bX_j \mid \bY_{\mT} = y)
        \end{align}
    \end{enumerate}
\end{lemma}
\begin{proof}
    All identities are simple but quite tedious matrix calculus.
\end{proof}

We are now ready to prove the gradient identity in Proposition~\ref{prop:matrix_mmse_gradient_identity}. Going forward, to ease notation we will suppress the explicit dependence of $\meta_i$ and $p$ on $y$ and $\mT$. Differentiating under the integral sign and by the product rule,
\begin{align} \label{eq:overlap_gradient_factorized}
    \partial_{kl} \left\{\int \langle \meta_i, \meta_j\rangle p \ dy\right\} &= \int \langle \meta_i, \meta_j\rangle \partial_{kl} \{ p\} \ dy + \int \langle \partial_{kl} \{\meta_i\}, \meta_j\rangle p \ dy + \int \langle \meta_i, \partial_{kl} \{\meta_j\} \rangle p \ dy.
\end{align}
Using Lemma~\ref{lem:overlap_differential_identities}(a) and via integration by parts,
\begin{align}
    \int \langle \meta_i, \meta_j\rangle \partial_{kl} \{ p\} \ dy
    &= -  \sum_{i=1}^n \left( \int \frac{\partial}{\partial y_{ik}} \left\{ \meta_kp\right\} \langle \meta_i, \meta_j\rangle \ dy + \int \frac{\partial}{\partial y_{il}} \left\{ \meta_lp\right\} \langle \meta_i, \meta_j\rangle \ dy \right)\\
    &= \int \left\langle \meta_k, \nabla_k \{\langle \meta_i, \meta_j \rangle\} \right\rangle p \ dy + \int \left\langle \meta_l, \nabla_l \{\langle \meta_i, \meta_j \rangle\} \right\rangle p \ dy. 
\end{align}
Next, from Lemma~\ref{lem:overlap_differential_identities}(c), we expand the gradients and obtain
\begin{align}
    &\int \left\langle \meta_k, \nabla_k \{\langle \meta_i, \meta_j \rangle\} \right\rangle p \ dy = \sum_{a=1}^d  \mT_{ak} \left\{\E\!\left[ \left\langle \Cov( \bX_i, \bX_a \mid \bY_{\mT}), \meta_j \meta_k^\top \right\rangle \right] + \E\!\left[ \left\langle \Cov( \bX_j, \bX_a \mid \bY_{\mT}), \meta_i \meta_k^\top \right\rangle \right] \right\}; \\
    &\int \left\langle \meta_l, \nabla_l \{\langle \meta_i, \meta_j \rangle\} \right\rangle p \ dy = \sum_{a=1}^d  \mT_{al} \left\{\E\!\left[ \left\langle \Cov( \bX_i, \bX_a \mid \bY_{\mT}), \meta_j \meta_l^\top \right\rangle \right] + \E\!\left[ \left\langle \Cov( \bX_j, \bX_a \mid \bY_{\mT}), \meta_i \meta_l^\top \right\rangle \right] \right\}.
\end{align}
We now move on to the other terms in \eqref{eq:overlap_gradient_factorized}. We just derive the expression for the second summand, as the third follows analogously. Using Lemma~\ref{lem:overlap_differential_identities}(b),
\begin{align}
    \int \langle \partial_{kl} \{\meta_i\}, \meta_j\rangle p \ dy = - \int \langle \md_i(y; \mT_{kl}), \meta_j \rangle \ dy - \int \langle \meta_i, \meta_j\rangle \partial_{kl} \{p\} \ dy.
\end{align}
We focus on the firs term on the right-hand side, as an expression for the second was already derived. Again, from integration by parts and Lemma~\ref{lem:overlap_differential_identities}(c),
\begin{align}
    - \int \langle \md_i(y; \mT_{kl}), \meta_j \rangle \ dy &= \int \left\langle \E\!\left[ \bX_i \bX_k^\top \right], \nabla_k \{ \meta_j  \}  \right\rangle \ dy + \int \left\langle \E\!\left[ \bX_i \bX_l^\top \right], \nabla_l \{ \meta_j  \}  \right\rangle \ dy \\
    &= \sum_{a=1}^d \left\{ \mT_{ak} \E\!\left[ \left\langle \E\!\left[ \bX_i \bX_k^\top \right],\Cov(\bX_j, \bX_a \mid \bY_{\mT})  \right\rangle \right]\right\} \nonumber \\
    &\quad + \sum_{a=1}^d \left\{ \mT_{al} \E\!\left[ \left\langle \E\!\left[ \bX_i \bX_l^\top \right],\Cov(\bX_j, \bX_a \mid \bY_{\mT})  \right\rangle \right] \right\}.
\end{align}
Putting everything together, we have
\begin{align}
    &\partial_{kl}\left\{ \int \langle \meta_i, \meta_j\rangle p \ dy \right\} = - \int \langle \md_i(y; \mT_{kl}), \meta_j \rangle \ dy - \int \langle \md_j(y; \mT_{kl}), \meta_j \rangle \ dy - \int \langle \meta_i, \meta_j\rangle \partial_{kl} \{ p\} \ dy \\
    &\quad = \sum_{a=1}^d \Big\{ \mT_{ak} \E\!\left[ \langle \Cov(\bX_i, \bX_k \mid \bY_{\mT}), \Cov(\bX_j, \bX_a \mid \bY_{\mT})  \rangle\right] + \mT_{ak} \E\!\left[ \langle \Cov(\bX_j, \bX_k \mid \bY_{\mT}), \Cov(\bX_i, \bX_a \mid \bY_{\mT}) \rangle\right] \Big\} \nonumber \\
    & \qquad + \sum_{a=1}^d \Big\{ \mT_{al} \E\!\left[ \langle \Cov(\bX_i, \bX_l \mid \bY_{\mT}), \Cov(\bX_j, \bX_a \mid \bY_{\mT})  \rangle\right] + \mT_{al} \E\!\left[ \langle \Cov(\bX_j, \bX_l \mid \bY_{\mT}), \Cov(\bX_i, \bX_a \mid \bY_{\mT}) \rangle\right] \Big\}.
\end{align}
Making the change of variable $\mT^2 = \mS$, normalizing by $1/n$ and arranging all partial derivatives in Kronecker form yields \eqref{eq:matrix_mmse_gradient_identity}, proving Proposition~\ref{prop:matrix_mmse_gradient_identity}.

\section{Useful Results}
In this section we collect some results that are used throughout the paper. For proofs of some of the propositions, we refer to the cited papers.

\begin{theorem}[\cite{Bai:1988aa}] \label{th:largest_eigenvalue_goe}
    Let $\bG \sim \GOE(n)$. Then, $\plim_{n \rightarrow \infty} \| \bG / \sqrt{n} \|_{\op} = \sqrt{2}$.  
\end{theorem}

\begin{theorem}[\cite{boucheron2013concentration}, Theorem 3.20] \label{th:gaussian_poincare_inequality}
    Let $Z \sim \Normal(0,\Id_n)$ be a standard Gaussian vector and $\psi: \R^n \rightarrow \R$ an absolutely continuous and weakly differentiable function with $\E[\|\nabla \psi(Z)\|^2]< \infty$. Then, there exists some $n$-independent constant $c$ such that $\Var(\psi(Z)) \leq c \E[\|\nabla \psi(Z)\|^2]$.
\end{theorem}

The following proposition collects some useful elementary properties of pseudo-Lipschitz functions. We omit proof as all statements are easily verified. 
\begin{prop} \label{prop:pseudo_lipschitz_function_properties}
    Let $(S, \| \cdot \|_S)$ be a normed space. Let $f: S \rightarrow \R \in \PL_p(L_f)$, $g: S \rightarrow \R \in \PL_q(L_g)$, $h: \R \rightarrow \R \in \PL_r(L_h)$, for some $p,q,r \geq 1$ and $L_f, L_g, L_h < \infty$. Then, all the following are true.
    \begin{itemize} 
        \item The function $f$ is locally Lipschitz, with Lipschitz constant $L_f(x)$ about a point $x$ upper bounded by $L_f(1 + 2\|x\|^{p-1})$. Thus, $f$ is weakly differentiable at and the magnitude of its total derivative at each point $x$ is upper bounded by $L_f(1 + 2\|x\|^{p-1})$.
        \item The function $q: S \rightarrow \R$ given by $p(x) = f(x)g(x)$ is pseudo-Lipschitz of order $p + q$ for some constant $L_p < \infty$.
        \item The function $c: S \rightarrow \R$ given by $c(x) = f(h(x))$ is pseudo-Lipschitz of order $pr$ for some constant $L_c < \infty$.
        \item The function $s: S \rightarrow \R$ given by $s(x) = f(x) + g(x)$ is pseudo-Lipschitz of order $\max\{ p, q \}$ with a constant $L_s$ upper bounded by $3(L_f + L_h) < \infty $.
    \end{itemize}
\end{prop}

The following are some results about uniformly pseudo-Lipschitz sequences of functions.
\begin{lemma} \label{lem:inner_products_uniformly_lipschitz_functions}
    Let $\{f_n : \R^{n} \rightarrow \R^{n}\}_{n \in \N} \subset \Lip(L)$ and $\{g_n : \R^{n} \rightarrow \R^{n}\}_{n \in \N} \subset \Lip(L')$ be two sequences of uniformly Lipschitz functions. Then, the functions $\{\psi_n : \R^{n} \rightarrow \R\}_{n \in \N}$ defined as $\psi_n (\mx) = \langle f_n(\mx), g_n(\mx) \rangle / n$ are uniformly pseudo-Lipschitz of order 2.
\end{lemma}

\begin{proof}
    Fix an $n \in \N$ and, for $i \in [n]$, let $f_i(\mx) = [f_n(\mx)]_i, g_i(\mx) = [g_n(\mx)]_i$.Then, from the uniformly Lipschitz assumption, each $f_i$ and $g_i$ is a Lipschitz function from $\R^{n}$ to $\R$, and the Lipschitz constants $L_i $ for $f_i$ and $L'_i$ for $g_i$ are upper bounded by $\sqrt{n}L$ and $\sqrt{n}L'$, respectively, and satisfy the constraints $\sum_{i\in [n]} L_i \leq c\sqrt{n}L $ and $\sum_{i\in [n]} L'_i \leq c'\sqrt{n}L'$ for some finite constants $c,c'$. Then, from Proposition~\ref{prop:pseudo_lipschitz_function_properties}, each product $f_i(\mx)g_i(\mx)$ is pseudo-Lipschitz of order 2, and so is their sum $\langle f_n(\mx), g_n(\mx) \rangle$, with a Lipschitz constant that is upper bounded by $nCLL'$, for some finite constant $C$. Dividing by $n$ removes the dimension-dependence from the Lipschitz constants in the sequence, and $\{\psi_n\}_{n \in \N}$ is uniformly pseudo Lipschitz of order 2 with some finite constant.
\end{proof}

\begin{lemma} \label{lem:averages_pseudo_lipschitz_function_uniformly_pseudo_Lipschitz}
    Let $f: \R^d \rightarrow \R \in \PL_2(L)$ be a pseudo-Lipschitz function of order 2. Then, the sequence of functions $\{\phi_n : (\R^d)^n \rightarrow \R\}_{n \in \N}$ defined as
    \begin{align}
        \phi_n(\mx_1, \dotsc, \mx_n) \coloneqq \frac{1}{n} \sum_{i=1}^n f(\mx_i)
    \end{align}
    is uniformly pseudo-Lipschitz of order 2 with a constant upper bounded by $\sqrt{3}L$.
\end{lemma}

\begin{proof}
    This is a simple consequence of Cauchy-Schwarz. For two vectors $\mX \coloneqq (\mx_1, \dotsc, \mx_n), \mY \coloneqq (\my_1, \dotsc \my_n)$, we have
    \begin{align}
        | \phi_n(\mX) - \phi_n(\mY) | &= \left| \frac{1}{n} \sum_{i=1}^n \left( f(\mx_i) - f(\my_i)\right) \right| \\
        &\leq \frac{L}{n} \sum_{i=1}^n \left( \| \mx_i - \my_i \| \left( 1 + \|\mx_i\| + \|\my_i\| \right) \right) \\
        &\leq L \left( \frac{1}{n} \sum_{i=1}^n \|\mx_i - \my_i\|^2 \right)^{1/2} \left( \frac{1}{n} \sum_{i=1}^n (1 +\|\mx_i\| + \|\my_i\|)^2 \right)^{1/2} \\
        &\leq L \| \mX - \mY \|_n \left( 3 (1 + \| \mX \|_n^2 + \| \mY \|_n^2 ) \right)^{1/2} \\
        & \leq \sqrt{3}L \| \mX - \mY \| (1 + \| \mX \|_n + \| \mY \|_n ).
    \end{align}
\end{proof}

\begin{lemma} \label{lem:pseudo_lipshitz_function_gaussian_variance_bound}
    Let $\{\psi_n : \R^{n \times d} \rightarrow \R \}_{n \in \N} \subset \PL_p(L)$ be a unifomly pseudo-Lipschitz sequence of order $p\geq 1$ wit constant $L$, and $\bZ \sim \Normal(0, \mSigma \otimes \Id_n)$ for some $\mSigma \in \PSD^d$. Then, for every $n \in \N$ we have the bound
    \begin{align}
        \Var(\psi_n(\bZ)) \leq \frac{cL\|\mSigma\|_{\mathrm{op}}^p}{n} \E\!\left[ \left( 1 + \|\bZ'\|_n^{p-1} \right)^2  \right],
    \end{align}
    for $\bZ' \sim \Normal(0, \Id_d \otimes \Id_n )$. Furthermore, 
    \begin{align}
        \plim_{n\rightarrow\infty} \left| \psi_n(\bZ) - \E[\psi_n(\bZ)] \right| = 0.
    \end{align}
\end{lemma}
\begin{proof}
    The first part is a direct consequence of the Gaussian Poincaré inequality in Theorem~\ref{th:gaussian_poincare_inequality} and the form of the weak derivative of $\psi_n$ given in Proposition~\ref{prop:pseudo_lipschitz_function_properties}. The second part is just Chebychev's inequality.
\end{proof}

\begin{lemma} \label{lem:pseudo_lipschitz_convergence_of_argument_convergence_of_function}
    Let $\{ \phi_n : \R^{n \times d} \rightarrow \R \}_{n \in \N} \subset \PL_p(L)$ for some $p \geq 1, L<\infty$, and $\bX, \bY \in \R^{n \times d}$ be two sequences of random matrices satisfying
    \begin{align}
        \| \bX - \bY \|_{n} \parrow 0; \quad \| \mX \|_{n} \parrow C,
    \end{align}
    for some finite constant $C$. Then, 
    \begin{align}
        \left| \phi_n(\bX) - \phi_n(\bY) \right| \parrow 0.
    \end{align}
\end{lemma}
\begin{proof}
    This is an elementary consequence of the definition of uniformly pseudo-Lipschitz functions. Simply write
    \begin{align}
        \left| \phi_n(\bX) - \phi_n(\bY) \right| &\leq L \| \bX - \bY \|_n \left( 1 + \|\bX\|_n^{p-1} + \| \bY \|_n^{p-1} \right) \\
        & \leq L \max\{1, 2^{(p-3)/2} \}\| \bX - \bY \|_n \left( 1 + 2\|\bX\|_n^{p-1} + \| \bX - \bY \|_n^{p-1} \right),
    \end{align}
    and the result follows from the continuous mapping theorem in light of the convergence assumptions on $\|\bX\|_n$ and $\|\bX -\bY\|_n$.
\end{proof}

\begin{lemma} \label{lem:gaussian_expectation_pseudo_lipschitz_is_pseudo_lipschitz}
    Let $\{ \phi_n : \R^{n \times d} \times \R^{n \times m} \rightarrow \R \}_{n \in \N} \subset \PL_p(L)$ be a uniformly pseudo-Lipschitz sequence of functions for some $p\geq 1$ and $L < \infty$. Let $\bZ \sim \Normal(0, \mSigma \otimes \Id_n)$ for some $\mSigma \in \PSD^m$. Then, the function sequence $\{ \zeta_n : \R^{n \times d} \rightarrow \R \}$ defined as $\zeta_n(\mX) = \E[\phi_n(\mX, \bZ)]$ is uniformly pseudo-Lipschitz of order $p$.
\end{lemma}

\begin{proof}
    For any $n \in \N$ and $\bX, \bY \in \R^{n \times d}$, we have
    \begin{align}
        \left| \zeta_n(\mX) - \zeta_n(\mY) \right| & \leq \E \left| \phi_n( \mX, \bZ ) - \phi_n( \mY, \bZ ) \right| \\
        &\leq L \E\!\left[ \| \bX - \bY \|_n\left( 1 + \| (\mX, \bZ) \|_n^{p-1} + \| (\mY, \bZ) \|_n^{p-1} \right) \right] \\
        &= L\| \bX - \bY \|_n \left( 1 + \E\| (\mX, \bZ) \|_n^{p-1} + \E\| (\mY, \bZ) \|_n^{p-1} \right) \\
        &\leq L\max\{1, 2^{(p-3)/2}\} \| \bX - \bY \|_n \left( 1 + 2\E\| \bZ \|_n^{p-1} + \E\| \mY\|_n^{p-1} + \E\| \mY\|_n^{p-1} \right),
    \end{align}
    and since the sequence of expectations $\E\|\bZ\|_n^{p-1}$ are uniformly bounded by some $C < \infty$ from the i.i.d. Gaussian the sequence $\{\zeta_n\}_{n \in \N}$ is uniformly pseduo-Lipschitz of order $p$ with constant upper bounded by $L \max\{1, 2^{(p-3)/2}\} \max\{1,C\}$.
\end{proof}

\begin{prop}
    Assume we have a sequence of random matrices $\mX \in \R^{n \times d}$ satisfying assumption \ref{as:mtp_model_amp_signal_assumptions_symmetric}, that is for any $p\geq 1$, $L < \infty$ and $\{\phi_n : \R^{n \times d} \rightarrow \R \}_{n \in \N}$ one has $\plim_{n \rightarrow \infty} | \phi_n(\bX) - \E[\phi_n(\bX)] | = 0$ and furthermore $\plim_{n \rightarrow \infty} \|\bX\|_n = \lim_{n\rightarrow\infty} \E \|\bX\|_n < \infty$. Then, for all $r > 0$ the limits $\lim_{n \rightarrow \infty} \E\|\bX\|_n^r$ are well-defined and finite.
\end{prop}

\begin{proof}
    From the continuous mapping theorem and the assumption $\plim_{n \rightarrow \infty} \|\bX\|_n = \lim_{n\rightarrow\infty} \E \|\bX\|_n$ it follows immediately that $\plim_{n \rightarrow \infty} \|\bX\|_n^r = \left( \lim_{n\rightarrow\infty} \E \|\bX\|_n \right)^r < \infty$. Then, since the map $\|\cdot\|_n^r$ is uniformly pseudo-Lipschitz of order $r$ by Proposition~\ref{prop:pseudo_lipschitz_function_properties}, we have that $\plim_{n \rightarrow \infty} | \|\bX\|_n^r - \E\|\bX\|_n^r | = 0$. Uniqueness of the limit then gives that $ \lim_{n \rightarrow \infty} \E\|\bX\|_n^r = \left( \lim_{n\rightarrow\infty} \E \|\bX\|_n \right)^r < \infty$.
\end{proof}

\end{document}